\documentclass[twoside]{article}

\usepackage[accepted]{aistats2023}
\usepackage{color}
\usepackage{amsfonts, amsmath, amssymb, amsthm}

\usepackage{hyperref}
\usepackage{cleveref}

\usepackage{graphicx}
\usepackage{subfig}
\usepackage{booktabs}

\allowdisplaybreaks

\usepackage[round]{natbib}

\let\hat\widehat
\let\tilde\widetilde
\DeclareMathOperator*{\E}{\mathbb{E}}
\let\Pr\relax
\DeclareMathOperator*{\Pr}{\mathbb{P}}

\newcommand{\eps}{\epsilon}

\newcommand{\inprod}[1]{\left\langle #1 \right\rangle}
\newcommand{\R}{\mathbb{R}}
\newcommand{\norm}[1]{\left\|#1 \right\|}

\newcommand{\tr}{\textnormal{Tr}}

\newtheorem{theorem}{Theorem}[section]
\newtheorem{corollary}[theorem]{Corollary}
\newtheorem{lemma}[theorem]{Lemma}

\newtheorem{proposition}[theorem]{Proposition}
\newtheorem{definition}[theorem]{Definition}

\newtheorem{remark}[theorem]{Remark}

\makeatletter
\newtheorem*{rep@theorem}{\rep@title}
\newcommand{\newreptheorem}[2]{%
\newenvironment{rep#1}[1]{%
 \def\rep@title{#2 \ref{##1}}%
 \begin{rep@theorem}}%
 {\end{rep@theorem}}}
\makeatother

\begin{document}

%

%

\twocolumn[

\aistatstitle{On the Neural Tangent Kernel Analysis of Randomly Pruned Neural Networks}

\aistatsauthor{ Hongru Yang, \,\, Zhangyang Wang}
\aistatsaddress{VITA Group, The University of Texas at Austin} ]

\begin{abstract}
Motivated by both theory and practice, we study how random pruning of the weights affects a neural network's neural tangent kernel (NTK). 
In particular, this work establishes an equivalence of the NTKs between a fully-connected neural network and its randomly pruned version. 
The equivalence is established under two cases. 
The first main result studies the infinite-width asymptotic. 
It is shown that given a pruning probability, for fully-connected neural networks with the weights randomly pruned at the initialization, as the width of each layer grows to infinity sequentially, the NTK of the pruned neural network converges to the limiting NTK of the original network with some extra scaling.
If the network weights are rescaled appropriately after pruning, this extra scaling can be removed. 
The second main result considers the finite-width case.
It is shown that to ensure the NTK's closeness to the limit, the dependence of width on the sparsity parameter is asymptotically linear, as the NTK's gap to its limit goes down to zero. 
Moreover, if the pruning probability is set to zero (i.e., no pruning), the bound on the required width matches the bound for fully-connected neural networks in previous works up to logarithmic factors. 
The proof of this result requires developing a novel analysis of a network structure which we called \textit{mask-induced pseudo-networks}.
Experiments are provided to evaluate our results.
\end{abstract}

\section{INTRODUCTION} 
Can a sparse neural network achieve competitive performance as a dense network?
The answer to this question can be traced back to the early work of \citep{lecun1990optimal} which showed that pruning a fully-trained neural network can preserve the original network's performance while reducing the inference cost.
This led to many further developments in  post-training pruning  such as \citep{han2015deep}. 

However, such gain seems hard to be transferred to the \textit{training} phase until the discovery of the lottery ticket hypothesis (LTH) \citep{frankle2018lottery}. 
The LTH states that there exists a sparse subnetwork inside a dense network at the initialization stage such that when trained in isolation, it can achieve almost matching performance with the original dense network. 
However, the method they used to find such a network is computationally expensive: they proposed iterative magnitude-based pruning (IMP) with rewinding which requires multiple rounds of pruning and re-training \citep{frankle2018lottery, frankle2019stabilizing,chen2020lottery}. 
Subsequent work has been making effort in finding good sparse subnetworks at initialization with little or no training \citep{lee2018snip, wang2019picking, tanaka2020pruning, frankle2020pruning, sreenivasan2022rare}. 
Nonetheless, these methods suffer a degenerate performance than IMP. Surprisingly, even random pruning, albeit the most naive approach, has been observed to be competitive for sparse training in practice \citep{su2020sanity, frankle2020pruning,liu2022the}. 

On the theory side, a recent line of work \citep{malach2020proving, pensia2020optimal, sreenivasan2022finding} proves that there exists a subnetwork in a larger network at the random initialization, that can match the performance of a smaller trained network without further training. 
However, finding such a subnetwork is computationally hard. 
Other than that, little theoretical understanding of the aforementioned practical pruning method is established. 
Now, since random pruning is the simplest (and cheapest) avenue towards sparsity, if we can understand how good a random pruned subnetwork could be, compared to the original unpruned network, 
then we can establish a ``lower bound"-type understanding on the effectiveness of neural network pruning, compared to other sophisticated pruning options. 

To understand the success of deep  networks theoretically, people have proved that running (stochastic) gradient descent on a sufficiently overparameterized deep neural network can rapidly drive the training error toward zero \citep{du2018gradient, allen2019convergence, du2019gradient, ji2019polylogarithmic, lee2019wide, zou2020gradient}, and further, under some conditions, those networks are able to generalize \citep{arora2019fine, cao2019generalization}.
All the aforementioned works either explicitly or implicitly establish that the neural network is close to its neural tangent kernel (NTK) \citep{jacot2018neural}, provided that the neural network is sufficiently overparameterized.
Further, if the network width grows to infinity, this matrix converges to some deterministic matrix under Gaussian initialization. 
In addition, it is shown that the convergence and generalization of the networks heavily depend on the condition number and the smallest eigenvalue of the NTK \citep{du2018gradient, du2019gradient, arora2019fine, cao2019generalization}. 

Motivated by the established theory on NTK and the recent empirical observation that random pruning becomes particularly effective if the original network is wide and deep \citep{liu2022the}, we study the effect of randomly pruning an overparameterized neural network in the NTK regime by asking the following question:
\begin{quote}
    \vspace{-0.8em}
    \textit{How does random pruning affect the wide neural network's tangent kernel?}
    \vspace{-0.8em}
\end{quote} 
If we can understand and bound the difference between the pruned network's NTK and its unpruned version, then we can hope for formalized results suggesting that the pruned neural network can achieve \textit{fast convergence} to zero training error and yield \textit{good generalization} after training. 
For practitioners, this perhaps surprising result is likely to bring random pruning back to the spotlight of model compression and efficient training, in our era when neural networks are practically scaled a lot wider and deeper, say those gigantic ``foundational models" \citep{bommasani2021opportunities}.

Interestingly, we show that random pruning only incur limited changes to the neural network's tangent kernel. 
We now summarize the main contributions of this work:
\begin{itemize}
    \item \textbf{Asymptotic limit.} 
    The first result shows that pruning does not change the NTK much, asymptotically. 
    More specifically, \Cref{thm: main_text_main_asymp} states that given a pruning probability, the NTK of the pruned network converges to the limiting NTK of the original network at the initialization with some extra scaling factors depending on the pruning probability, as the network width grows to infinity sequentially.
    As a simple corollary, this scaling can be removed by rescaling the weights after pruning. 
    Further, this sequential limit can be indeed approached by increasing the width of the network. 
    
    \item \textbf{Non-asymptotic bound.}
    The second main result studies how large the network width needs to be to ensure that the pruned network's NTK is close to its infinite-width limit. 
    \Cref{thm: main_text_main_non_asymp} shows an asymptotically linear dependence of the network width on the sparsity parameter, as the gap between the pruned network's NTK and its limit goes down to zero.  
    Further, if the pruning probability is set to zero, our width lower bound recovers the bound in \citep{arora2019exact} for fully-connected neural networks up to some logarithmic factors. 
    The proof of \Cref{thm: main_text_main_non_asymp} requires developing novel analysis of a network structure that is closely related to the pruned network which we called \textbf{mask-induced pseudo-networks}. 
    We give a detailed explanation in \Cref{section: backward}. 
    We further validate our theory experimentally in \Cref{sec: validate_theory}. 
\end{itemize}
Although our result is about the networks at the initialization, \cite{du2018gradient, arora2019exact, allen2019convergence} suggested that the network is still closely related to the NTK after training, provided that the network is sufficiently overparameterized. 
Therefore, by further applying the established analysis in the previous work, the equivalence can still hold after training. 

\subsection{Related Work}
\textbf{Sparse Neural Networks in Practice.}
Since the discovery of the Lottery Ticket Hypothesis \citep{frankle2018lottery}, many efforts have been made to develop methods to find good sparse networks with little overhead.
Those methodologies can be divided into two groups: static sparse training and dynamic sparse training \citep{liu2023ten}.

Static sparse training can be based on either {random pruning} and {non-random pruning}. 
As for random pruning, every layer can be uniformly pruned with the same pre-defined pruning ratio \citep{mariet2015diversity, he2017channel, gale2019state} or the pruning ratio can be varied for different layers such as Erd\"o-R\'enyi \citep{mocanu2018scalable} and Erd\"o-R\'enyi Kernel \citep{evci2020rigging}.
For non-random pruning, those methods usually prune network weights according to some proposed saliency criteria such as SNIP \citep{lee2018snip}, GraSP \cite{wang2019picking}, SynFlow \cite{tanaka2020pruning} and NTK-based score \cite{liu2020finding}.
On the other hand, dynamic sparse training \citep{mocanu2018scalable,liu2021sparse} explores the sparsity pattern in a prune-and-grow scheme according to some criteria \citep{mocanu2018scalable, mostafa2019parameter, dettmers2019sparse, evci2020rigging, ye2020good, jayakumar2020top, liu2021we}. 
Further, the sparsity pattern can be learned by using sparsity-inducing regularizer \citep{Yang2020DeepHoyer:}.
Other ways of reducing the computational cost include finding a good subnetwork and then fine-tuning \citep{sreenivasan2022rare}, and transferring lottery tickets \citep{morcos2019one, chen2021elastic}. 
To understand the transferability of lottery tickets, \cite{redman2021universality} studied IMP via the renormalization group theory in physics.
Based on this development, people in practice use sparsity to improve robustness \citep{chen2021sparsity, liu2022robust,ding2021audio} and data efficiency \citep{chen2021data, zhang2021efficient}. 

\textbf{Theoretical Study of The Lottery Ticket Hypothesis.}
On the theory side, there are works proving that a small dense network can indeed be approximated by pruning a larger network. \cite{malach2020proving} proved that a target network of width $d$ and depth $l$ can be indeed approximated by pruning a randomly initialized network that is of a polynomial factor (in $d,l$) wider and twice deeper even without further training. 
\cite{ramanujan2020s} empirically verified this stronger version of LTH.
Later, \cite{pensia2020optimal} improved the widening factor to a logarithmic bound, and \cite{sreenivasan2022finding} proves that with a polylogarithmic widening factor, such a result holds even if the network weights are binary.
Unsurprisingly, all of the above results are computationally hard to achieve. 
In addition, all these works are based on a functional approximation argument and don't consider how pruning affects the training process (and, subsequently, generalization). 

\textbf{Neural Tangent Kernels.}
Over the past few years, there is tremendous progress on understanding training overparameterized deep neural networks. 
A series of works \citep{du2018gradient,allen2019convergence, du2019gradient, ji2019polylogarithmic, lee2019wide, zou2020gradient} have established gradient descent convergence guarantee based on NTK \citep{jacot2018neural}. 
Further, under some conditions, these networks are able to generalize \citep{arora2019fine, cao2019generalization}.
\cite{yang2019scaling, arora2019exact} provided asymptotic and non-asymptotic proofs on the limiting NTK.
Further, algorithms for computing the tangent kernels are developed for various architectures \citep{lee2019wide, arora2019exact, han2022fast}. 
Other related works include studying how depth affects the diagonal of NTK \citep{hanin2019finite} and the smallest eigenvalue of NTK under certain data distribution assumption \citep{nguyen2021tight}.
Overall, the neural tangent kernel provides valuable, yet oversimplified, explanation on the neural network's success \citep{chizat2019lazy}. 

One work in a similar spirit to ours is \citep{liao2022on} which studies the convergence of training an over-parameterized one-hidden-layer neural network with sparse activation by gradient descent. 
Although both works consider random pruning (or masking), our work is different from theirs in a sense that the sparsity in our setting is from pruning the weights instead of neurons whereas their sparsity is obtained from masking neurons at the each step of gradient descent. 
Further, we consider neural networks of arbitrary depth and their work is focusing on the one-hidden-layer neural networks. 
Note that for the problem considered in this work, pruning (masking) neurons will be trivial since it merely incur changes to the network width. 


\section{PRELIMINARIES}
\textbf{Notations.}
We use lowercase letters to denote scalars and boldface letters and symbols (e.g. $\mathbf{x}$) to denote vectors and matrices.
Element-wise product is denoted by $\odot$ and $\otimes$ denotes the Kronecker product. 
$\Pi_{\mathbf{x}}$ denotes the orthogonal projection onto the vector space generated by $\mathbf{x}$ and $\Pi_{\mathbf{A}}$ denote the orthogonal projection onto the column space of $\mathbf{A}$.
We use $\textnormal{diag}(\mathbf{x})$ to denote a diagonal matrix where its diagonals are elements from the vector $\mathbf{x}$.
Further, $\widetilde{O}, \widetilde{\Theta}, \widetilde{\Omega}$ are used to suppress logarithmic factors in $O, \Theta, \Omega$.

\subsection{Problem Formulation}
Here we want to study the training dynamics of a sparse sub-network in an ultra-wide neural network. 
For simplicity, we first apply our analysis on fully-connected neural networks. 
We denote by $f(\mathbf{x}) = f(\boldsymbol{\theta}, \mathbf{x})$ the output of the full network, $\tilde{f}(\mathbf{x}) = f(\boldsymbol{\theta} \odot \mathbf{m}, \mathbf{x}) \in \R$ the output of a sparse sub-network obtained by random pruning where $\boldsymbol{\theta} \in \R^N$ denotes the network parameters, $\mathbf{m} \in \R^N$ is the sparse mask and $\mathbf{x} \in \R^d$ is the input. 
We distinguish the output of each layer of the original full networks from the sparse sub-networks by adding tilde to the symbols. 
For simplicity, we assume the network outputs a scalar \footnote{Without loss of generality, our analysis can be extended to the vector-output case. }. 
We assume that the sparse mask is obtained from sampling each individual weight i.i.d. from a Bernoulli distribution with probability $\alpha$. 
Formally, let $\mathbf{x} \in \R^d$ be the input, and denote $\tilde{\mathbf{g}}^{(0)}(\mathbf{x}) = \mathbf{x}$ and $d_0 = d$. 
An $L$-hidden-layer fully connected network can be defined recursively as: 
\begin{align*}
    \tilde{\mathbf{f}}^{(h)}(\mathbf{x}) &= \left(\mathbf{W}^{(h)} \odot \mathbf{m}^{(h)}\right) \tilde{\mathbf{g}}^{(h-1)}(\mathbf{x}) \in \R^{d_h},\\
    \tilde{\mathbf{g}}^{(h)}(\mathbf{x}) &= \sqrt{\frac{c_\sigma}{d_h}} \sigma\left( \tilde{\mathbf{f}}^{(h)}(\mathbf{x}) \right) \in \R^{d_h}, \quad h = 1,2,\ldots, L,
\end{align*}
where $\mathbf{W}^{(h)} \in \R^{d_h \times d_{h-1}}$ is the weight matrix in the $h$-th layer, $\mathbf{m}^{(h)} \in \R^{d_h \times d_{h-1}}$ is the sparse mask for the $h$-th layer $\sigma: \R \rightarrow \R$ is a coordinate-wise activation function which we only consider ReLU activation in this work and $c_\sigma = \left( \E_{z \sim \mathcal{N}(0,1)} \left[ \sigma(z)^2 \right] \right)^{-1}$ is used to normalize the output of the activation.
For ReLU, a simple calculation shows $c_\sigma = 2$.
Let $\mathbf{m} = (\mathbf{m}^{(1)}, \ldots, \mathbf{m}^{(L+1)})$ and $\boldsymbol{\theta} = (\mathbf{W}^{(1)}, \ldots, \mathbf{W}^{(L+1)})$ represents the masks and weights in the network, respectively. 
All the weights $\mathbf{W}^{(h)}_{ij}$ are initialized i.i.d. from $\mathcal{N}(0,1)$ and the masks $\mathbf{m}^{(h)}_{ij}$ are sampled i.i.d. from $\textnormal{Bernoulli}(\alpha)$.

The NTK of the pruned network is given by
\begin{align}\label{eq: ntk}
    \tilde{\boldsymbol{\Theta}}(\mathbf{x,x'}) &= \inprod{\frac{\partial \tilde{f}(\mathbf{x})}{\partial \boldsymbol{\theta}}, \frac{\partial \tilde{f}(\mathbf{x'})}{\partial \boldsymbol{\theta}}} = \sum_{h=1}^{L+1} \inprod{\frac{\partial \tilde{f}(\mathbf{x})}{\partial \mathbf{W}^{(h)}}, \frac{\partial \tilde{f}(\mathbf{x}')}{\partial \mathbf{W}^{(h)}}}.
\end{align}
We now compute the gradient of the pruned network. 
Note that since the weights being pruned are staying at zero always during the training process, the gradient of the pruned network is simply the masked gradient of the unpruned network. 
Thus, its gradient is given by
\begin{align}\label{eq: grad_pruned_network}
    \frac{\partial \tilde{f}(\mathbf{x})}{\partial \mathbf{W}^{(h)}} &= \left(\tilde{\mathbf{b}}^{(h)}(\mathbf{x}) \left( \tilde{\mathbf{g}}^{(h-1)}(\mathbf{x}) \right)^\top \right) \odot \mathbf{m}^{(h)},
\end{align}
where $\tilde{\mathbf{b}}^{(h)}$ is given by 
\begin{align}\label{eq: def_b}
    & \tilde{\mathbf{b}}^{(L+1)}(\mathbf{x}) = 1 \in \R, \nonumber \\
    & \tilde{\mathbf{b}}^{(h)}(\mathbf{x}) = \nonumber \\
    & \sqrt{\frac{c_\sigma}{d_h}} \tilde{\mathbf{D}}^{(h)}(\mathbf{x}) (\mathbf{W}^{(h+1)} \odot \mathbf{m}^{(h+1)})^\top \tilde{\mathbf{b}}^{(h+1)}(\mathbf{x}) \in \R^{d_h}, 
\end{align}
and 
\begin{align*}
    & \tilde{\mathbf{D}}^{(h)}(\mathbf{x}) = \text{diag}\left( \dot{\sigma} \left( \tilde{\mathbf{f}}^{(h)}(\mathbf{x}) \right) \right) \in \R^{d_h \times d_h}, \quad h = 1, \ldots, L.
\end{align*}
Further, in order to give the infinite-width limit of the NTK for the fully-connected neural networks we need to define the following quantities:
for $h \in [L]$, define
\[
\Sigma^{(0)}(\mathbf{x,x}') = \mathbf{x}^\top \mathbf{x}',
\]
\[
\boldsymbol{\Lambda}^{(h)}(\mathbf{x,x}') =
\begin{bmatrix}
\Sigma^{(h-1)}(\mathbf{x, x}) & \Sigma^{(h-1)}(\mathbf{x, x}') \\
\Sigma^{(h-1)}(\mathbf{x', x}) & \Sigma^{(h-1)}(\mathbf{x', x'})
\end{bmatrix}
 \in \R^{2 \times 2},
\]
\[
\Sigma^{(h)}(\mathbf{x,x}') = c_\sigma \E_{(u,v) \sim \mathcal{N}(\mathbf{0}, \boldsymbol{\Lambda}^{(h)})} [\sigma(u) \sigma(v)],
\]
and
\begin{align*}
    \dot{\Sigma}^{(h)} (\mathbf{x,x'}) = c_\sigma \E_{(u,v) \sim \mathcal{N}(\mathbf{0}, \boldsymbol{\Lambda}^{(h)})} [\dot{\sigma}(u) \dot{\sigma}(v)],
\end{align*}
where $\dot{\sigma}$ denotes the derivative of ReLU: $\dot{\sigma}(x) = \mathbb{I}(x > 0)$.
We define similar quantities of $\Sigma^{(h)}$ for randomly pruned neural networks in \Cref{section: asymp}.
It can be shown that
\begin{align*}
    \boldsymbol{\Theta}_\infty(\mathbf{x,x'}) &= \lim_{d_1, d_2, \ldots, d_L \rightarrow \infty} \sum_{h=1}^{L+1} \inprod{\frac{\partial f(\boldsymbol{\theta}, \mathbf{x})}{\partial \mathbf{W}^{(h)}}, \frac{\partial f(\boldsymbol{\theta}, \mathbf{x}')}{\partial \mathbf{W}^{(h)}}} \\
    &= \sum_{h=1}^{L+1} \left( \Sigma^{(h-1)}(\mathbf{x,x'}) \prod_{h'=h}^{L+1} \dot{\Sigma}^{(h')} (\mathbf{x,x'}) \right).
\end{align*}

\section{MAIN RESULTS}
In this section, we present the main results of our work. 
We show that given the pruning probability, the NTK of the pruned network is closely related to the limiting NTK of the unpruned network, if the network is sufficiently wide.

\textbf{Asymptotic Limit.}
We first present the asymptotic limit of the pruned network as width grows to infinity. 
\begin{theorem}[The limiting NTK of randomly pruned networks]\label{thm: main_text_main_asymp}
Consider an $L$-hidden-layer fully-connected ReLU neural network. 
Suppose the network weights are initialized from an i.i.d. standard Gaussian distribution and the weights except the input layer are pruned independently with probability $1 - \alpha$ at the initialization. 
Assume the backpropagation is computed by sampling an independent copy of weights. 
Then, as the width of each layer goes to infinity sequentially,
\begin{align*}
    \lim_{d_1, d_2, \ldots, d_L \rightarrow \infty} \tilde{\boldsymbol{\Theta}}(\mathbf{x,x'}) = \alpha^{L} \boldsymbol{\Theta}_\infty(\mathbf{x,x'}),
\end{align*}
where $\tilde{\boldsymbol{\Theta}}$ denotes the NTK of the pruned network and $\boldsymbol{\Theta}_\infty$ denotes the limiting NTK of the unpruned network. 
\end{theorem}
The theorem suggests that given a pruning probability, asymptotically as the network width grows to infinity, the NTK of the randomly pruned network will converges to the limiting NTK of the full network up to some scaling depending on the pruning probability. 
Although we assume an independent copy of weights for the backward propagation, we will remove this assumption in \Cref{thm: main_text_main_non_asymp}.
\begin{remark}
From \citep{arora2019exact}, if the training dataset of size $n$ is given by $(\mathbf{X,y})$, the function induced by the NTK $\boldsymbol{\Theta}(\mathbf{X,X}) \in \R^{n \times n}$ is given as
\begin{align*}
    f_{\textnormal{ntk}}(\mathbf{x}) = \boldsymbol{\Theta}(\mathbf{x},\mathbf{X})^\top \boldsymbol{\Theta}(\mathbf{X},\mathbf{X})^{-1} \mathbf{y},
\end{align*}
where $\boldsymbol{\Theta}(\mathbf{x,X}) \in \R^n$.
Thus, any scaling factor in front of the NTK is cancelled and the actual function induced by the NTK is the \emph{same}.
\end{remark}
On the other hand, this scaling factor can be removed simply by rescaling the weights according to the pruning probability which is given in the following corollary.
\begin{corollary}\label{corollary: main_text_rescale}
Consider the same setting as in \Cref{thm: main_text_main_asymp} except now we rescale the mask by $1/\sqrt{\alpha}$. Then, the neural tangent kernel after rescaling $\widetilde{\boldsymbol{\Theta}}_\alpha$ satisfies
\begin{align*}
    \lim_{d_1, d_2, \ldots, d_L \rightarrow \infty} \tilde{\boldsymbol{\Theta}}_\alpha (\mathbf{x,x'}) = \boldsymbol{\Theta}_\infty(\mathbf{x,x'}).
\end{align*}
\end{corollary}
\begin{proof}
Let $\tilde{f}_\alpha$ be the network after rescaling and $\mathbf{m}_\alpha$ denote the rescaled mask, i.e., $\mathbf{m}_\alpha = \mathbf{m} \cdot (1/\sqrt{\alpha})$.
Based the definition of $\tilde{\mathbf{b}}^{(h)}(\mathbf{x})$ in \Cref{eq: def_b}, we define $\tilde{\mathbf{b}}_\alpha^{(L+1)} = 1$ and for $h = 1, 2, \ldots, L$,
\begin{align*}
    \tilde{\mathbf{b}}_{\alpha}^{(h)}(\mathbf{x}) := \sqrt{\frac{c_\sigma}{d_h}} \tilde{\mathbf{D}}^{(h)}(\mathbf{x}) (\mathbf{W}^{(h+1)} \odot \mathbf{m}^{(h+1)}_\alpha)^\top \tilde{\mathbf{b}}^{(h+1)}_\alpha(\mathbf{x}).
\end{align*}
Based on this definition we have $\tilde{\mathbf{b}}_\alpha^{(h)}(\mathbf{x}) = (1/\sqrt{\alpha})^{L+1-h} \tilde{\mathbf{b}}^{(h)}(\mathbf{x})$.
Similarly, define the rescaled activation output: $\tilde{\mathbf{g}}_\alpha^{(1)} = \sqrt{\frac{c_\sigma}{d_h}} \sigma(\mathbf{W}^{(h)} \mathbf{x} )$ and for $h = 2, \ldots, L$,
\begin{align*}
    \tilde{\mathbf{g}}^{(h)}_\alpha(\mathbf{x}) &= \sqrt{\frac{c_\sigma}{d_h}} \sigma\left( \left(\mathbf{W}^{(h)} \odot \mathbf{m}^{(h)}\right) \tilde{\mathbf{g}}^{(h-1)}_\alpha(\mathbf{x}) \right) \in \R^{d_h}.
\end{align*}
Since ReLU is positively homogeneous, i.e., $\sigma(cx) = c \cdot \sigma(x)$ for $c>0$, we have $\tilde{\mathbf{g}}_\alpha^{(h)}(\mathbf{x}) = (1/\sqrt{\alpha})^{h-1} \tilde{\mathbf{g}}^{(h)}$. 
Thus, by \Cref{eq: grad_pruned_network}, for all $h \in [L+1]$ we have
\begin{align*}
    \frac{\partial \tilde{f}_\alpha(\mathbf{x})}{\partial \mathbf{W}^{(h)}} = \left( \frac{1}{\sqrt{\alpha}} \right)^{L} \frac{\partial \tilde{f}(\mathbf{x})}{\partial \mathbf{W}^{(h)}}.
\end{align*}
Plugging this in \Cref{eq: ntk} finishes our proof. 
\end{proof}
\textbf{From Asymptotic to Non-asymptotic.}
Since \Cref{thm: main_text_main_asymp} considers sequential limits which assumes all the previous layers are already at the limit distribution when we analyze with a given layer. 
However, a typical drawback of such analysis is that the limit of expectation (as the previous layer's width grows to infinity) is not necessarily the same as the expectation of limit (the previous layer's width is exactly infinite). 
Thus, we need to justify that the network is indeed able to approach the limit by increasing width.
In mathematical language, this is the same as justifying the exchange of limit for $\E\sigma(\cdot)$ and $\E\dot{\sigma}(\cdot)$. 
Fortunately, ReLU (and its derivative) are nice enough and we can justify this by leveraging the tools in measure-theoretic probability. 
\begin{lemma}\label{lemma: main_text_exchange_limit}
Conditioned on $\mathbf{g}^{(h-1)}(\mathbf{x}), \mathbf{g}^{(h-1)}(\mathbf{x}')$.
Consider a fixed $i \in [d_{h+1}]$. 
Let 
\[
X_{d_h} = 
\begin{bmatrix}
\sqrt{\frac{c_\sigma}{d_h}} \sum_{j=1}^{d_h} \mathbf{W}^{(h+1)}_{ij} \mathbf{m}^{(h+1)}_{ij} \sigma(\tilde{\mathbf{f}}_j^{(h)}(\mathbf{x})) \\
\sqrt{\frac{c_\sigma}{d_h}} \sum_{j=1}^{d_h} \mathbf{W}^{(h+1)}_{ij} \mathbf{m}^{(h+1)}_{ij} \sigma(\tilde{\mathbf{f}}_j^{(h)}(\mathbf{x}'))
\end{bmatrix}
\in \R^2,
\]
and let $g: \R^2 \rightarrow \R$ to be $g(x,y) \in \{ \sigma(x) \sigma(y), \dot{\sigma}(x) \dot{\sigma}(y)\}$.
Then, 
\begin{align*}
    \lim_{d_h \rightarrow \infty} \E[g(X_{d_h})] = \E[g(\lim_{d_h \rightarrow \infty} X_{d_h})].
\end{align*}
\end{lemma}
The proof can be found in \Cref{app: asym_to_nonasym} in the Appendix.

\textbf{Non-Asymptotic Bound. }
Building upon the asymptotic result in \Cref{thm: main_text_main_asymp}, given the pruning probability, we study how wide the neural network needs to be in order for its NTK to be close to the limiting NTK. 
\begin{theorem}[Non-asymptotic Bound of Randomly Pruned Network's NTK, Simplified Version of \Cref{thm: main_non_asymp}]\label{thm: main_text_main_non_asymp}
Consider an $L$-hidden-layer fully-connected ReLU neural network with the $h$-th layer of width $d_h$. 
Suppose $d_1 = d_2 = \ldots = d_L = d$. 
Let the weights be initialized i.i.d. by standard Gaussian distribution. 
Suppose all the weights except the input layer are pruned independently with probability $1 - \alpha$ at the initialization and rescaled by $1/\sqrt{\alpha}$ after pruning.
For $\delta \in (0,1)$ and sufficiently small $\eps > 0$, if
\begin{align}\label{eq: required_width}
    d &\geq \tilde{\Omega}\left(\max\left( \frac{1}{\alpha} \frac{L^6}{\eps^4}, \frac{1}{\alpha^2} \frac{L^2}{\eps^2} \right) \right),
\end{align} 
then for any inputs $\mathbf{x,x'} \in \R^{d_0}$ such that $\norm{\mathbf{x}}_2 \leq 1,\ \norm{\mathbf{x}'}_2 \leq 1$, with probability at least $1 - \delta$ over the randomness in the initialization and pruning, we have
\begin{align*}
    \left| \tilde{\boldsymbol{\Theta}}(\mathbf{x,x'}) - \boldsymbol{\Theta}_\infty (\mathbf{x,x'}) \right| \leq (L+1) \eps.
\end{align*}
\end{theorem}

Note that the two terms in \Cref{eq: required_width} has different dependence on $1/\alpha$: only $1/\alpha$ is needed for the forward propagation and $1/\alpha^2$ is needed for the backward pass, which we show in \Cref{section: nonasymp}. 
If we let $\eps \rightarrow 0$, the first term in \Cref{eq: required_width} will dominate and the required width $d$ only needs to scale linearly with $1/\alpha$ in this asymptotic case. 
We validate our theory by comparing the Monte Carlo estimate of NTK value to the limiting NTK value in \Cref{sec: validate_theory}. 
\begin{remark}
By setting the probability of pruning a given weight to be zero, our result matches the bound for fully-connected neural networks in \citep{arora2019exact} up to logarithmic factors. 
\end{remark}

\section{THE ASYMPTOTIC LIMIT}\label{section: asymp}
In this section, we show how to derive the asymptotic limit of the NTK of the pruned networks, which gives a proof outline of \Cref{thm: main_text_main_asymp}. 
We give an outline of our analysis in this section and we defer the complete proof to \Cref{app: asym} in Appendix.

We first introduce two quantities for randomly pruned neural networks analogous to the fully-connected networks. 
\begin{definition}\label{def: forward_backward_limit}
Define
\begin{align*}
    \widetilde{\Sigma}^{(h)}(\mathbf{x,x'}) &:= \lim_{d_1, \ldots, d_h \rightarrow \infty} {\inprod{\tilde{\mathbf{g}}^{(h)} (\mathbf{x}), \tilde{\mathbf{g}}^{(h)} (\mathbf{x}')}}, 
\end{align*}
where the limit is taken sequentially from $d_1$ to $d_h$. 
\end{definition}
As a simple consequence of the law of large numbers, $\tilde{\Sigma}^{(h)}$ is well-defined. 
Based on \Cref{eq: ntk}, we compute 
\begin{align*}
    &\inprod{\frac{\partial \tilde{f}(\mathbf{x})}{\partial \mathbf{W}^{(h)}}, \frac{\partial \tilde{f}( \mathbf{x}')}{\partial \mathbf{W}^{(h)}}} &= \left(\tilde{\mathbf{b}}^{(h)} (\mathbf{x}) \right)^\top \mathbf{G}^{(h-1)} \tilde{\mathbf{b}}^{(h)} (\mathbf{x}') ,
\end{align*}
where $\mathbf{G}^{(h-1)}$ is a diagonal matrix and $\mathbf{G}^{(h-1)}_{ii} = \inprod{\tilde{\mathbf{g}}^{(h-1)} (\mathbf{x})\odot \mathbf{m}^{(h)}_i, \tilde{\mathbf{g}}^{(h-1)} (\mathbf{x}')\odot \mathbf{m}^{(h)}_i}$.
Notice that under the sequential limit, as $d_{h-1} \rightarrow \infty$, $\mathbf{G}_{ii}^{(h-1)} \rightarrow \alpha \tilde{\Sigma}^{(h-1)}(\mathbf{x,x'})$.
Thus, the NTK depends on analyzing both the forward propagation and the backward propagation of the pruned neural network.
We show the results in the following two simple lemmas. 
\begin{lemma}\label{lemma: main_text_asymp_forward}
Suppose a fully-connected neural network uses ReLU as its activation and $d_1, d_2, \ldots, d_{L} \rightarrow \infty$ sequentially, then
\begin{align*}
    \tilde{\Sigma}^{(h)}(\mathbf{x,x'}) = \alpha^{h-1} \Sigma^{(h)}(\mathbf{x,x'}),
\end{align*}
for $h = 1,2,\ldots, L$.
\end{lemma}

\begin{lemma}\label{lemma: main_text_asymp_backward}
Assume we use a fresh sample of weights in the backward pass, then
\begin{align*}
    & \lim_{d_1, \ldots, d_L \rightarrow \infty} \inprod{\tilde{\mathbf{b}}^{(h)}(\mathbf{x}), \tilde{\mathbf{b}}^{(h)}(\mathbf{x'})} \\
    &= \alpha^{L+1-h} \prod_{h'=h}^L \dot{\Sigma}^{(h')}(\mathbf{x,x'}).
\end{align*}
\end{lemma}
The proof of \Cref{lemma: main_text_asymp_backward} assumes that we use an independent Gaussian copy in the backward propagation which can be removed in the next section. 
Combining the two lemmas provided above, we can prove \Cref{thm: main_text_main_asymp}.

Note that pruning the input layer creates additional difficulties since the input dimension is fixed. 
The NTK of the full network depends on $\Sigma^{(0)}(\mathbf{x,x'}) = \mathbf{x}^\top \mathbf{x'}$. 
If we prune the input layer then $\tilde{\Sigma}^{(0)}(\mathbf{x,x'}) = (\mathbf{m} \odot \mathbf{x})^\top (\mathbf{m} \odot \mathbf{x'})$ which is random. 
In this case, it seems hard to relate $\tilde{\Sigma}^{(1)}(\mathbf{x,x'})$ to $\Sigma^{(1)}(\mathbf{x,x'})$ in this asymptotic regime.

\section{THE NON-ASYMPTOTIC BOUND}\label{section: nonasymp}
In this section, we give a proof outline of \Cref{thm: main_text_main_non_asymp}.
Since in this section we are only talking about the pruned network, there is no longer ambiguity in distinguishing pruned and unpruned networks.
For notation ease, we remove the tilde above all the symbols of the quantities in the pruned network. 
In addition, we use $\mathbf{m}_i^{(h)}$ to denote the $i$-th row of $\mathbf{m}^{(h)}$ and similar for $\mathbf{w}_i^{(h)}$. 
From a high level, the proof consists of analyzing the forward propagation and the backward propagation. 
We give a complete treatment in \Cref{app: nonasymp} in the Appendix.

\subsection{Analyzing the Forward Propagation}
We first present our result on the forward propagation. 
\begin{theorem}[Simplified Version of \Cref{thm: concentration_g}]\label{thm: main_text_concentration_g}
Consider the same setting as in \Cref{thm: main_text_main_non_asymp}. 
There exist constants $c$ such that if $\eps \leq \min(c, \frac{1}{L})$ and 
\begin{align*}
     d \geq \tilde{\Omega}\left(\frac{1}{\alpha} \frac{L^2 }{\eps^2} \right),
\end{align*}
then with probability $1 - \delta$ over the randomness in the initialization of all the weights and masks, for all $h \in [L],\  i \in [d_{h+1}], \  (\mathbf{x}^{(1)}, \mathbf{x}^{(2)}) \in \{(\mathbf{x,x}), (\mathbf{x,x'}), (\mathbf{x',x'})\}$,
\begin{align*}
    & \Big| \left( \mathbf{g}^{(h)}(\mathbf{x}^{(1)}) \odot \mathbf{m}^{(h+1)}_i \right)^\top \left( \mathbf{g}^{(h)}(\mathbf{x}^{(2)}) \odot \mathbf{m}^{(h+1)}_i \right) \\
    & \quad - {\Sigma}^{(h)}(\mathbf{x}^{(1)}, \mathbf{x}^{(2)}) \Big| \leq \eps .
\end{align*}
\end{theorem}
Our result provides the required width to ensure the activation of each layer is close to its limit. 
The dependence on $1/\alpha$ is precisely due to the presence of random masks and notice that each mask is a sub-Gaussian random variable with variance proxy $1/\alpha$. 

\subsection{Analyzing the Backward Propagation}\label{section: backward}
In this section, we show that $\inprod{\mathbf{b}^{(h)}(\mathbf{x}^{(1)}), \mathbf{b}^{(h)}(\mathbf{x}^{(2)})} \approx \prod_{h' = h}^L \dot{\Sigma}^{(h')}(\mathbf{x}^{(1)},\mathbf{x}^{(2)})$ under the assumption that the event in \Cref{thm: main_text_concentration_g} occurs.
This is where we formally justify the fresh Gaussian copy trick.
We consider a fixed pair $(\mathbf{x}^{(1)}, \mathbf{x}^{(2)})$ and suppress the dependence on inputs when there is no confusion.
We do this by induction: assume $\mathbf{b}^{(h+1)}(\mathbf{x}^{(1)})^\top \mathbf{b}^{(h+1)}(\mathbf{x}^{(2)}) \approx \prod_{h'=h+1}^L \dot{\Sigma}(\mathbf{x}^{(1)}, \mathbf{x}^{(2)})$.
Define $\mathbf{G}^{(h)}_i := [(\mathbf{g}^{(h)}(\mathbf{x}) \odot \mathbf{m}^{(h+1)}_i),\ (\mathbf{g}^{(h)}(\mathbf{x}') \odot \mathbf{m}^{(h+1)}_i)]$ and $\mathbf{F}_i^{(h+1)}:= (\mathbf{W}^{(h+1)} \odot \mathbf{m}^{(h+1)}) \mathbf{G}^{(h)}_i$.
Notice that the dependence of $\mathbf{b}^{(h+1)}$ on $\mathbf{W}^{(h+1)}$ is by $\mathbf{F}_i^{(h+1)}$. 
If $\mathbf{W}^{(h+1)}$ is independent to $\mathbf{b}^{(h+1)}$ (which it isn't), then 
\begin{align}\label{eq: expectation_inprod_b}
    &\mathcal{E} := \E_{\mathbf{W}^{(h+1)}}\left[ (\mathbf{b}^{(h)}(\mathbf{x}^{(1)}))^\top \mathbf{b}^{(h)}(\mathbf{x}^{(2)}) \right] \\
    &= \frac{2}{d_h} \sum_i \mathbf{b}_i^{(h+1)} (\mathbf{x}^{(1)}) \mathbf{b}_i^{(h+1)} (\mathbf{x}^{(2)}) \tr( \mathbf{M}_i^{(h+1)} \mathbf{D} \mathbf{M}_i^{(h+1)}). \nonumber
\end{align}
It is easy to show that $ \tr( \mathbf{M}_i^{(h+1)} \mathbf{D} \mathbf{M}_i^{(h+1)}) \approx \dot{\Sigma}$ and $(\mathbf{b}^{(h)}(\mathbf{x}^{(1)}))^\top \mathbf{b}^{(h)}(\mathbf{x}^{(2)})$ is close to its expectation. 
Then by induction hypothesis we are done. 
Now we show that $\mathbf{W}^{(h+1)}$ is nearly independent to $\mathbf{b}^{(h+1)}$.
Recall a \textbf{special property of the standard Gaussian}: given $\mathbf{w} \sim \mathcal{N}(\mathbf{0, I})$ and two fixed vectors $\mathbf{x}, \mathbf{y}$, if $\mathbf{x^\top y} = 0$, then $\mathbf{w}^\top \mathbf{x}$ and $\mathbf{w}^\top \mathbf{y}$ are independent.
Thus, conditioned on $\mathbf{b}^{(h+1)}$, $\mathbf{G}_i^{(h)}, \mathbf{F}_i^{(h+1)}, \mathbf{m}^{(h+1)}$, we have $\mathbf{w}_i^{(h+1)} \Pi_{\mathbf{G}_i}^\bot \stackrel{\mathcal{D}}{=} \tilde{\mathbf{w}}_i^{(h+1)} \Pi_{\mathbf{G}_i}^\bot$ where $\tilde{\mathbf{w}}_i^{(h+1)}$ is an i.i.d. copy of ${\mathbf{w}}_i^{(h+1)}$.
Let
\begin{align*}
& \mathbf{b}^{(h)}_\bot := 
\left( \mathbf{b}^{(h+1)} \right)^\top
\begin{bmatrix}
((\tilde{\mathbf{w}}_1^{(h+1)})^\top \Pi_{\mathbf{G}_1}^\bot ) \odot \mathbf{m}_1^{(h+1)} \\
\vdots \\
((\tilde{\mathbf{w}}_{d_{h+1}}^{(h+1)})^\top \Pi_{\mathbf{G}_{d_{h+1}}}^\bot ) \odot \mathbf{m}^{(h+1)}_{d_{h+1}} \\
\end{bmatrix} \mathbf{D}, \\
& \mathbf{b}^{(h)}_\parallel := 
\left( \mathbf{b}^{(h+1)} \right)^\top
\begin{bmatrix}
((\mathbf{w}_1^{(h+1)})^\top \Pi_{\mathbf{G}_1})  \odot \mathbf{m}_1^{(h+1)} \\
\vdots \\
((\mathbf{w}_{d_{h+1}}^{(h+1)})^\top \Pi_{\mathbf{G}_{d_{h+1}}} ) \odot \mathbf{m}^{(h+1)}_{d_{h+1}}
\end{bmatrix} \mathbf{D}.
\end{align*}
Notice that $\mathbf{b}^{(h)} = \mathbf{b}^{(h)}_\bot + \mathbf{b}^{(h)}_\parallel$.
Next, we are going to show that the main contribution of $\inprod{\mathbf{b}^{(h)}(\mathbf{x}), \mathbf{b}^{(h)}(\mathbf{x'})}$ is from $\mathbf{b}^{(h)}_{\bot}$ and $(\mathbf{b}^{(h)}_{\bot})^\top \mathbf{b}^{(h)}_{\bot} \approx \mathcal{E}$ whereas the contribution from the dependent part $\mathbf{b}^{(h)}_\parallel$ is small.
We show these two results in \Cref{prop: main_text_independent} and \Cref{prop: main_text_dependent}.
\begin{proposition}[Informal Version of \Cref{prop: independent}]\label{prop: main_text_independent}
Under some appropriate conditions, with probability at least $1 - \delta_2/2$ over the randomness in $\mathbf{W}^{(h+1)}$, for any $(\mathbf{x}^{(1)}, \mathbf{x}^{(2)}) \in \{ (\mathbf{x,x}), (\mathbf{x,x'}), (\mathbf{x',x'}) \}$, we have
\begin{align*}
    & \bigg| \frac{2}{d_h} \left(\mathbf{b}^{(h)}_\bot (\mathbf{x}^{(1)}) \right)^\top \mathbf{b}^{(h)}_\bot (\mathbf{x}^{(2)}) - \mathcal{E} \bigg| \leq O\left( \sqrt{\frac{\log \frac{1}{\delta_2}}{\alpha d_h}} \right).
\end{align*}
\end{proposition}

\begin{proposition}[Informal Version of \Cref{prop: dependent}]\label{prop: main_text_dependent}
Under some appropriate conditions, if $d \geq \tilde{\Omega}(\frac{1}{\alpha} \frac{L^2}{\eps^2})$, with probability $1 - \delta_2/2$ over the randomness in the initialization of $\mathbf{W}^{(h+1)}, \mathbf{m}^{(h+1)}, \ldots, \mathbf{W}^{(L+1)}, \mathbf{m}^{(L+1)}$, 
\begin{align*}
    \sqrt{\frac{1}{d_h}}\norm{\mathbf{b}^{(h)}_\parallel}_2 
    &\leq O\left(\sqrt{\frac{1}{\alpha^2 d_h} \log \frac{1}{\delta_2}} \right). 
\end{align*}
\end{proposition}
The proof of \Cref{prop: main_text_independent} requires some intricate calculation and then applying Gaussian chaos concentration bound which is left in \Cref{sec: independent} in Appendix. 
We now give a detailed description of the proof of \Cref{prop: main_text_dependent}. 
For the ease of presentation, we omit the dependence on layer and inputs when there is no confusion. 
First of all, we can decompose $\Pi_{\mathbf{G}_i} = \Pi_{\mathbf{g}(\mathbf{x}) \odot \mathbf{m}_i} + \Pi_{\mathbf{G}_i/\mathbf{g}(\mathbf{x}) \odot \mathbf{m}_i}$ where $\mathbf{G}_i/\mathbf{g}(\mathbf{x}) \odot \mathbf{m}_i$ denotes the subspace of $\mathbf{G}_i$ orthogonal to $\mathbf{g}(\mathbf{x}) \odot \mathbf{m}_i$.
Bounding the second part is simple by utilizing the special property of the standard Gaussian. 
We now focus on bounding the first part. 
Writing $\mathbf{g}_i$ short for $\mathbf{g} \odot \mathbf{m}_i$,
\begin{align}\label{eq: main_text_dependent}
    &\left( \mathbf{b}^{(h+1)} \right)^\top
    \begin{bmatrix}
    ((\mathbf{w}_1^{(h+1)})^\top \Pi_{\mathbf{g}_1})  \odot \mathbf{m}_1^{(h+1)} \\
    \vdots \\
    ((\mathbf{w}_{d_{h+1}}^{(h+1)})^\top \Pi_{\mathbf{g}_{d_{h+1}}} ) \odot \mathbf{m}^{(h+1)}_{d_{h+1}}
    \end{bmatrix} \nonumber \\
    &= \frac{1}{\sqrt{\alpha}} \sum_{i} \mathbf{b}^{(h+1)}_i ( {\mathbf{w}}^{(h+1)}_i)^\top \frac{\mathbf{g}_i \mathbf{g}_i^\top}{\norm{\mathbf{g}_i}^2_2}.
\end{align}
The presence of the mask introduces further difficulties in the analysis.
In particular, without the pruning masks, the above vector nicely simplifies to 
\begin{align}\label{eq: without_pruning}
    (\mathbf{b}^{(h+1)})^\top \mathbf{W}^{(h+1)} \mathbf{g}^{(h)} \frac{\mathbf{g}^{(h)}(\mathbf{x})}{\norm{\mathbf{g}^{(h)}(\mathbf{x})}_2} = f(\mathbf{x}) \frac{\mathbf{g}^{(h)}(\mathbf{x})}{\norm{\mathbf{g}^{(h)}(\mathbf{x})}_2}.
\end{align}
We would like the above relation to also hold for pruned network. 
However, this is not true since each $\mathbf{g}_i$ is different. 
A closer examination of the expression in \Cref{eq: main_text_dependent} tells us that like the relation in \Cref{eq: without_pruning}, the $i$-th coordinate of this vector can be written as the product of $\mathbf{g}^{(h)}_i(\mathbf{x})$ and some structure similar to the pruned network, which we call \textit{mask-induced pseudo-networks}.

\subsubsection{Mask-Induced Pseudo-Network}
\begin{definition}[Pseudo-network induced by mask]\label{def: main_text_pseudo_network}
Define the pseudo-network induced by the $h$-th layer $j$-th column of sparse masks $ \mathbf{m}^{(h)}$ for all $h \in \{2, \ldots, L\}$, $j \in [d_{h-1}]$ and $h' \in \{h+1, h+2, \ldots, L\}$ to be
\begin{align*}
    \mathbf{g}^{(h,j,h)} &= \sqrt{\frac{c_\sigma}{d_{h}}} \mathbf{D}^{(h)} \textnormal{diag}_i \left(\frac{\mathbf{m}_{ij}^{(h)} \sqrt{\alpha} }{\norm{\mathbf{g}^{(h-1)} \odot \mathbf{m}_i^{(h)}}_2^2} \right) \mathbf{f}^{(h)}, \\
    \mathbf{f}^{(h,j,h')} &= \left(\mathbf{W}^{(h')} \odot \mathbf{m}^{(h')} \right) \mathbf{g}^{(h, j, h'-1)}, \\
    \mathbf{g}^{(h, j, h')} &= \sqrt{\frac{c_\sigma}{d_{h'}}} \mathbf{D}^{(h')}(\mathbf{x}) \mathbf{f}^{(h,j,h')}.
\end{align*}
The output of this pseudo-network is $f^{(h,j,L+1)}$.
\end{definition}
Using this definition, we can write 
\begin{align*}
    &\left( \frac{1}{\sqrt{\alpha}} \sum_{i} \mathbf{b}^{(h+1)}_i ( {\mathbf{w}}^{(h+1)}_i)^\top \frac{\mathbf{g}_i^{(h)} \mathbf{g}^{(h)\top}_i}{\norm{\mathbf{g}_i^{(h)}}^2_2} \right)_j \\
    &= \frac{1}{{\alpha}} \mathbf{g}_j^{(h)} {f}^{(h+1, j, L+1)}.
\end{align*}
Now our goal is to show that $|{f}^{(h+1, j, L+1)}| = \tilde{O}(1)$ for all $h, j$. 
This requires us to analyze the forward propagation of this pseudo-network. 
Specifically, we need to show that the norm of $\mathbf{g}^{(h,j,h')}$ is $O(1)$ for all $h < h' \leq L$.
However, whether a neuron turns on depends on the input it receives in the pruned network instead of the pseudo-network. 
Nonetheless, we show that this doesn't matter when we consider the \textit{norm} of the activation in the pseudo-network, since it has the \textit{same} distribution as the activation in the pruned network. 
\begin{proposition}\label{prop: main_text_gaussian_indicator}
For any given nonzero vectors $\mathbf{x,y}$, the distribution of $(\mathbf{w}^\top \mathbf{x})^2 \mathbb{I}(\mathbf{w}^\top \mathbf{y} > 0)$ is the same as $(\mathbf{w}^\top \mathbf{x})^2 \mathbb{I}(\mathbf{w}^\top \mathbf{x} > 0)$ where $\mathbf{w} \sim \mathcal{N}(\mathbf{0, I})$.
\end{proposition}
This proposition says we can bound the norm of the activation by ignoring which neurons turn on. 
Thus, utilizing this result, we can analyze the forward propagation of the pseudo-network just as analyzing the pruned network and show that we indeed have $|{f}^{(h+1, j, L+1)}| = \tilde{O}(1)$ for all $h, j$. 
This completes the proof outline of \Cref{prop: main_text_dependent}.


\section{EXPERIMENTS}
This section presents our empirical results. 
Our results contain two parts: first we validate our theory; then, we evaluate our theory on real world dataset.

\subsection{Validating Our Theory}\label{sec: validate_theory}
\begin{figure}[ht]
    \centering
    \includegraphics[width=\columnwidth]{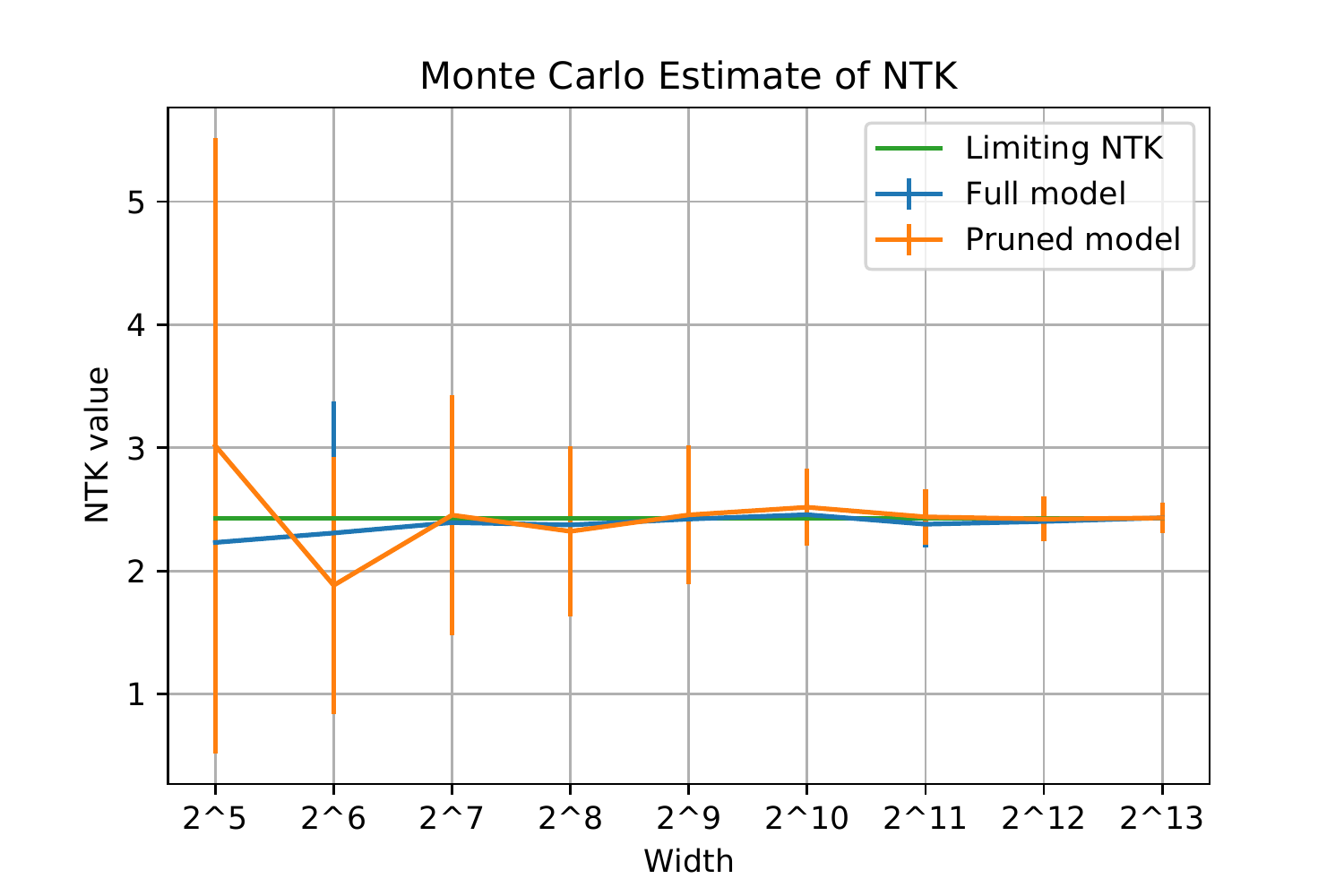}
    \caption{Figure (a) validates \Cref{corollary: main_text_rescale} which shows the empirical NTK value generated by the full model and pruned model with varying width compared with theoretical NTK limit.
    The limiting NTK value is computed by a known closed-form formula in \citep{arora2019exact}.}
    \label{figure: width_ntk}
\end{figure}

\begin{figure}[ht]
    \vskip -0.2in
    \centering
    \includegraphics[width=\columnwidth]{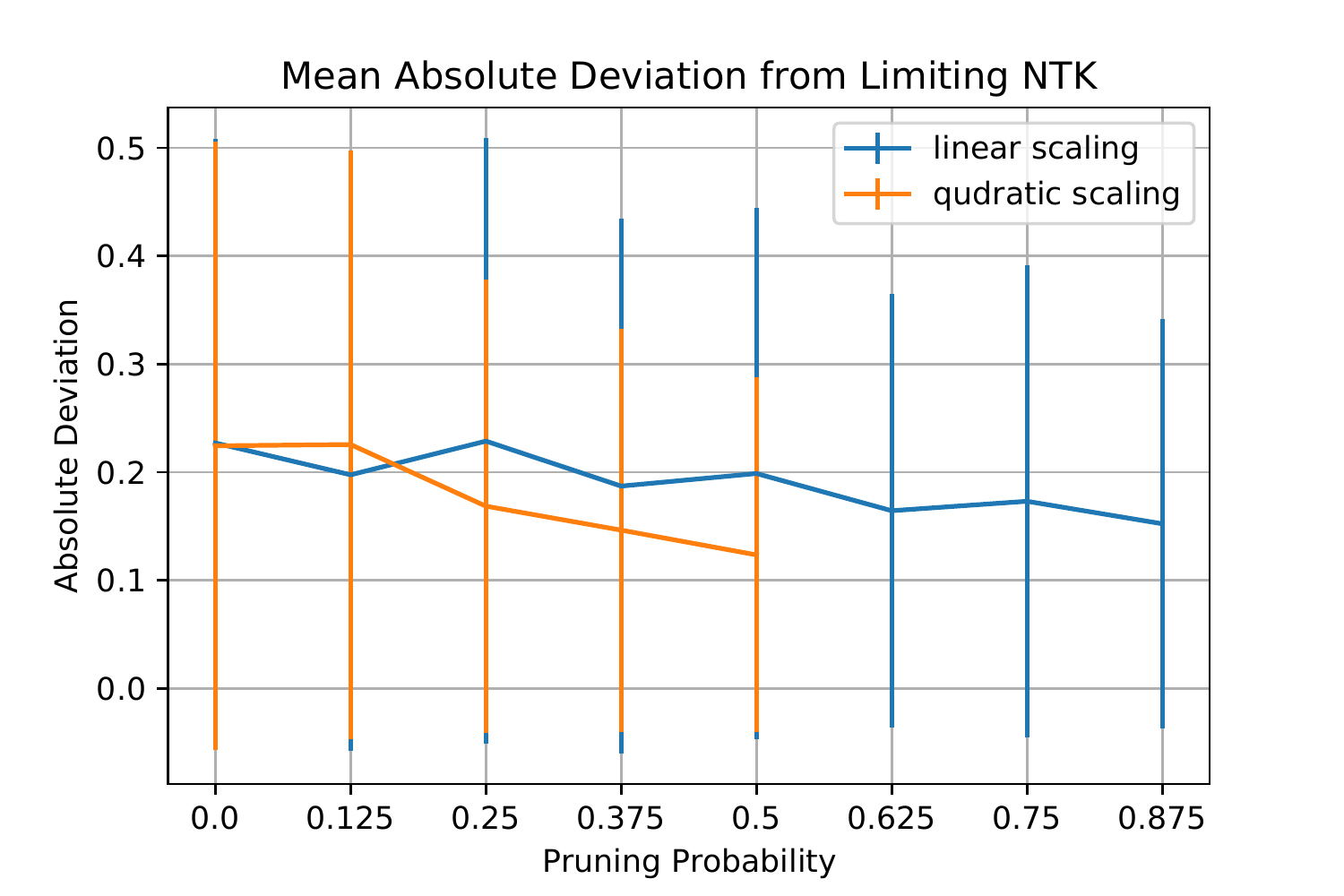}
    \caption{The results of the mean absolute deviation of the empirical NTK value from the limiting NTK. 
    At each pruning probability, the width of the network is scaled quadratically and linearly with respect to $1 / \alpha$.}
    \label{figure: pruning_fraction_ntk}
\end{figure}

\textbf{Validation of \Cref{corollary: main_text_rescale} (and, thus, \Cref{thm: main_text_main_asymp})}: We show that the empirical NTK value computed from the pruned network converges to the theoretical NTK limit as the width increases.
We use fully-connected neural networks with 3-hidden layers of the same width as our model.
We rescale the weights by $1 / \sqrt{\alpha}$ after pruning. 
We first randomly generate two data points $\mathbf{x,y}$ and then randomly initialize the networks with Gaussian distribution. 
We fix the pruning probability to be $1/2$ and vary the width from $32$ to $8192$. 
For each trial, we create 64 samples of the empirical NTK values generated by the unpruned and pruned networks, and plot their mean. 
\Cref{figure: width_ntk} shows that, as the width increase, our empirical estimates from both unpruned and pruned model converge to the limiting NTK value.

\textbf{Validation of \Cref{thm: main_text_main_non_asymp}}:
\Cref{thm: main_text_main_non_asymp} suggests that $d_h$ needs to scale asymptotically linearly with respect to $1 / \alpha$ to maintain the gap between the empirical NTK and limiting NTK. 
To evaluate our \Cref{thm: main_text_main_non_asymp}, we start with a full model of width $1024$ and then prune the model with various probability $1- \alpha$ while scaling the width quadratically and linearly with $1/\alpha$.
Since quadratically scaling width is expensive, we stop at $0.5$ pruning probability.
We generate 100 samples for each pruning probability and take their mean absolute deviation from the theoretically computed NTK value. 
The result is shown in \Cref{figure: pruning_fraction_ntk}.
In both cases, the gap to the limiting NTK is non-increasing.

\subsection{On the Real World Data}
\begin{figure}
    \vskip -0.2in
    \centering
    \includegraphics[width=\columnwidth]{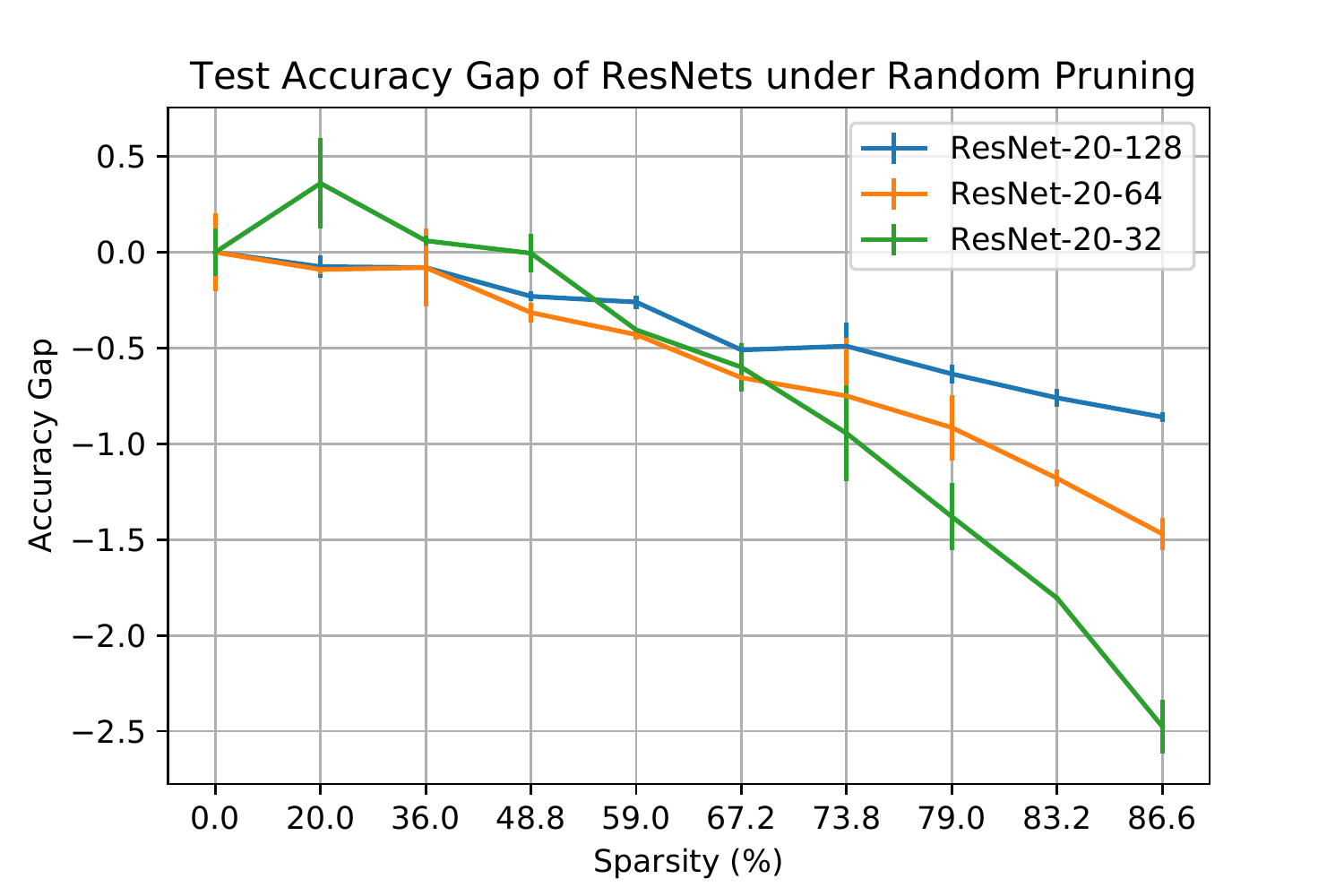} 
    \caption{Performance of random pruning without rescaling on ResNet-20 of different widths. 
    Sparsity on the x-axis means the fraction of weights remaining in IMP and pruning probability in random pruning.}
    \label{figure: resnet_random}
\end{figure}
In this section, we further evaluate our theory on real-world data. 
Our theory suggests that if the network is wide enough, the pruned networks should retain much of the performance of the full networks.  
We note that here we prune all layers of the neural networks.
We adopt the implementation from \cite{chen2021elastic}.

We extensively test our theory across different neural network architectures and datasets.
For pruning methods, in addition to random pruning (with and without rescaling weights after pruning), we also include Iterative Magnitude-based Pruning (IMP) in our experiments. 
We train fully-connected neural networks on MNIST dataset \cite{deng2012mnist} and, VGGs and ResNets \cite{he2016deep} on CIFAR-10 \cite{krizhevsky2009learning}, and vary the width of these architectures. 
We generate each data point in the plot by averaging over 2 independent runs.
We defer the detailed experiment setup in \Cref{sec: exp_setup} in Appendix.

\textbf{Results.}
In \Cref{figure: resnet_random}, for random pruning without rescaling, the testing performance gap narrows as the network width is getting larger.  
For ResNet-20-128, at sparsity $86.6\%$, the performance of random pruning and the full model is within $1\%$ on CIFAR-10. 
Similar results have been observed for other pruning methods and other architectures and datasets.  
Further experiment results are shown in \Cref{sec: further_exp} in Appendix.

\section{DISCUSSION AND FUTURE WORK}
In this paper, we establish an equivalence between the NTK of a randomly pruned neural network and the limiting NTK of the unpruned network under both asymptotic and finite-width cases. 
For the finite width case, we establish an asymptotically linear dependence of network width on the sparsity parameter $1/\alpha$. 
One open problem is whether $1/\alpha^2$ dependence is indeed necessary for the backward propagation so that the width dependence on $1/\alpha$ can be improved to exactly linear instead of asymptotically linear. 
We leave further investigation on this open problem. 

One limitation of our current analysis is that it only applies to random pruning and assumes that the pruning distribution is completely independent from the weight initialization. Therefore, our analysis is not valid for magnitude-based pruning or gradient-based pruning, as the weights being pruned have internal correlations with the magnitude of the weights.
Another limitation is that the NTK analysis inherently restricts the neural network's ability to perform feature learning. We believe that the advantages of pruning, such as improving network generalization, can be demonstrated in a feature learning setting. This direction is left for future research and exploration as well.

\subsection*{Acknowledgements}
H. Yang and Z. Wang thank the anonymous reviewers for the helpful feedback and comments. Z. Wang is supported by NSF Scale-MoDL (award number: 2133861).

\bibliography{ref}
\bibliographystyle{apalike}

\onecolumn
\aistatstitle{Supplementary Materials}
\tableofcontents
\newpage

\section{ASYMPTOTIC ANALYSIS (Proof of \texorpdfstring{\Cref{thm: main_text_main_asymp}}{Asymptotic Limit})}\label{app: asym}
This section is devoted to prove the asymptotic limit of the pruned networks' NTK.
Recall that we use tilde over a symbol to denote the quantity in the \emph{pruned} network and the corresponding symbol without tilde denotes the quantity in the unpruned network. 
\begin{theorem}[The limiting NTK of randomly pruned networks, Restatement of \Cref{thm: main_text_main_asymp}]\label{thm: main_asymp}
Consider an $L$-hidden-layer fully-connected ReLU neural network. 
Suppose the network weights are initialized from an i.i.d. standard Gaussian distribution and the weights except the input layer are pruned independently with probability $1 - \alpha$ at the initialization. 
Assume the backpropagation is computed by sampling a independent copy of weights. 
Then, as the width of each layer goes to infinity sequentially,
\begin{align*}
    \lim_{d_1, d_2, \ldots, d_L \rightarrow \infty} \tilde{\boldsymbol{\Theta}}(\mathbf{x,x'}) = \alpha^{L} \boldsymbol{\Theta}_\infty(\mathbf{x,x'}),
\end{align*}
where $\tilde{\boldsymbol{\Theta}}$ denotes the NTK of the pruned network and $\boldsymbol{\Theta}_\infty$ denotes the limiting NTK of the unpruned network. 
\end{theorem}
For the pruned neural networks, its gradient is given by
\begin{align*}
    \frac{\partial \tilde{f}(\mathbf{x})}{\partial \mathbf{W}^{(h)} } &= \left(\tilde{\mathbf{b}}^{(h)}(\mathbf{x}) \left( \tilde{\mathbf{g}}^{(h-1)}(\mathbf{x}) \right)^\top \right) \odot \mathbf{m}^{(h)}, \quad h = 2,\ldots, L+1
\end{align*}
where 
\begin{gather}
    \tilde{\mathbf{b}}^{(h)}(\mathbf{x}) = 
    \begin{cases}
    1 \in \R, & h = L+1\\
    \sqrt{\frac{c_\sigma}{d_h}} \tilde{\mathbf{D}}^{(h)}(\mathbf{x}) (\mathbf{W}^{(h+1)} \odot \mathbf{m}^{(h+1)})^\top \tilde{\mathbf{b}}^{(h+1)}(\mathbf{x}) \in \R^{d_h}, &h = 1, \ldots, L,
    \end{cases}
\end{gather}
and 
\begin{align}\label{eq: definition: D}
    \tilde{\mathbf{D}}^{(h)}(\mathbf{x}) = \text{diag}\left( \dot{\sigma} \left( \tilde{\mathbf{f}}^{(h)}(\mathbf{x}) \right) \right) \in \R^{d_h \times d_h}, \quad h = 1, \ldots, L.
\end{align}
Note that since the weights being pruned are staying at zero always during the training process, the gradient of the pruned network is simply the masked gradient of the unpruned network. 

Now, we have
\begin{align*}
    \inprod{\frac{\partial \tilde{f}(\mathbf{x})}{\partial \mathbf{W}^{(h)} }, \frac{\partial \tilde{f}( \mathbf{x}')}{\partial \mathbf{W}^{(h)}}} 
    &= \inprod{\left(\tilde{\mathbf{b}}^{(h)}(\mathbf{x}) \left( \tilde{\mathbf{g}}^{(h-1)}(\mathbf{x}) \right)^\top \right) \odot \mathbf{m}^{(h)} , \left(\tilde{\mathbf{b}}^{(h)}(\mathbf{x}') \left( \tilde{\mathbf{g}}^{(h-1)}(\mathbf{x}') \right)^\top \right) \odot \mathbf{m}^{(h)}}.
\end{align*}
Now we write
\[
\left(\tilde{\mathbf{b}}^{(h)}(\mathbf{x}) \left( \tilde{\mathbf{g}}^{(h-1)}(\mathbf{x}) \right)^\top \right) \odot \mathbf{m}^{(h)} = 
\begin{bmatrix}
\tilde{\mathbf{b}}^{(h)}_1(\mathbf{x}) \tilde{\mathbf{g}}^{(h-1)} (\mathbf{x}) \odot \mathbf{m}^{(h)}_1 \\
\tilde{\mathbf{b}}^{(h)}_2(\mathbf{x}) \tilde{\mathbf{g}}^{(h-1)} (\mathbf{x})\odot \mathbf{m}^{(h)}_2 \\
\vdots \\
\tilde{\mathbf{b}}^{(h)}_{d_h}(\mathbf{x}) \tilde{\mathbf{g}}^{(h-1)}(\mathbf{x}) \odot \mathbf{m}^{(h)}_{d_h} \\
\end{bmatrix}.
\]
Thus, 
\begin{align}\label{eq: inprod_ntk_elem}
    \inprod{\frac{\partial \tilde{f}(\mathbf{x})}{\partial \mathbf{W}^{(h)} }, \frac{\partial \tilde{f}( \mathbf{x}')}{\partial \mathbf{W}^{(h)}}} &= \inprod{\left(\tilde{\mathbf{b}}^{(h)}(\mathbf{x}) \left( \tilde{\mathbf{g}}^{(h-1)}(\mathbf{x}) \right)^\top \right) \odot \mathbf{m}^{(h)} , \left(\tilde{\mathbf{b}}^{(h)}(\mathbf{x}') \left( \tilde{\mathbf{g}}^{(h-1)}(\mathbf{x}') \right)^\top \right) \odot \mathbf{m}^{(h)}} \nonumber \\
    &= \sum_{i=1}^{d_h} \tilde{\mathbf{b}}^{(h)}_i (\mathbf{x}) \tilde{\mathbf{b}}^{(h)}_i (\mathbf{x}') \inprod{\tilde{\mathbf{g}}^{(h-1)} (\mathbf{x})\odot \mathbf{m}^{(h)}_i, \tilde{\mathbf{g}}^{(h-1)} (\mathbf{x}')\odot \mathbf{m}^{(h)}_i} \nonumber \\
    &= \left(\tilde{\mathbf{b}}^{(h)} (\mathbf{x}) \right)^\top \mathbf{G}^{(h-1)} \tilde{\mathbf{b}}^{(h)} (\mathbf{x}'),
\end{align}
where we define $\mathbf{G}^{(h-1)}$ as a diagonal matrix and $\mathbf{G}^{(h-1)}_{ii} = \inprod{\tilde{\mathbf{g}}^{(h-1)} (\mathbf{x})\odot \mathbf{m}^{(h)}_i, \tilde{\mathbf{g}}^{(h-1)} (\mathbf{x}')\odot \mathbf{m}^{(h)}_i}$.
Observe that 
\begin{align*}
    \lim_{d_{h-1} \rightarrow \infty} \inprod{\tilde{\mathbf{g}}^{(h-1)} (\mathbf{x}) \odot \mathbf{m}^{(h)}_i, \tilde{\mathbf{g}}^{(h-1)} (\mathbf{x}')\odot \mathbf{m}^{(h)}_i} &= \lim_{d_{h-1} \rightarrow \infty} \frac{c_\sigma}{d_{h-1}} \sum_{j=1}^{d_{h-1}} \sigma \left( \tilde{\mathbf{f}}^{(h-1)}_j (\mathbf{x}) \right) \sigma \left( \tilde{\mathbf{f}}^{(h-1)}_j (\mathbf{x'}) \right) \left( \mathbf{m}_{ij}^{(h)} \right)^2 \\
    &= \E \left[ c_\sigma \sigma \left( \tilde{\mathbf{f}}^{(h-1)}_j (\mathbf{x}) \right) \sigma \left( \tilde{\mathbf{f}}^{(h-1)}_j (\mathbf{x'}) \right) \left( \mathbf{m}_{ij}^{(h)} \right)^2 \right] \\
    &= \E \left[ c_\sigma \sigma \left( \tilde{\mathbf{f}}^{(h-1)}_j (\mathbf{x}) \right) \sigma \left( \tilde{\mathbf{f}}^{(h-1)}_j (\mathbf{x'}) \right) \right] \E \left[ \left( \mathbf{m}_{ij}^{(h)} \right)^2 \right] \\
    &= \alpha \E \left[ c_\sigma \sigma \left( \tilde{\mathbf{f}}^{(h-1)}_j (\mathbf{x}) \right) \sigma \left( \tilde{\mathbf{f}}^{(h-1)}_j (\mathbf{x'}) \right) \right].
\end{align*}
This requires us to analyze $\tilde{\mathbf{f}}^{(h)}(\mathbf{x})$ for $h \in [L]$. 

We now analyze the forward dynamics of the pruned neural network:
\begin{align*}
    [\tilde{\mathbf{f}}^{(h+1)} (\mathbf{x})]_i &= \sum_{j=1}^{d_h} [\mathbf{W}^{(h+1)} \odot \mathbf{m}^{(h+1)}]_{ij} [\tilde{\mathbf{g}}^{(h)}(\mathbf{x})]_j \\
    &= \sqrt{\frac{c_\sigma}{d_h}} \sum_{j=1}^{d_h} [\mathbf{W}^{(h+1)} \odot \mathbf{m}^{(h+1)}]_{ij} \sigma\left( \left[ \tilde{\mathbf{f}}^{(h)}(\mathbf{x}) \right]_j \right) .
\end{align*} 
Conditioned on $\mathbf{g}^{(h-1)}(\mathbf{x}), \mathbf{g}^{(h-1)}(\mathbf{x}')$, we have $\tilde{\mathbf{f}}_j^{(h)}(\mathbf{x}), \tilde{\mathbf{f}}_j^{(h)}(\mathbf{x}')$ are i.i.d. random variables for all $j \in [n]$. 
However, for $h \in \{1, \ldots, L\}$, as $d_h \rightarrow \infty$, by the central limit theorem, $[\tilde{\mathbf{f}}^{(h+1)}(\mathbf{x})]_i$ converges to a Gaussian random variable.
{This is certainly not true for the output in the first layer because the input dimension can't go to infinity. 
Thus, we make assumption that the pruning only starts from the second layer.}

Now by i.i.d assumption of the mask and weights, we can compute the covariance of pre-activation as 
\begin{align}\label{eq: correlation}
\begin{split}
    \E_{\mathbf{W}^{(h+1)}} \left[ \left. \left[ \tilde{\mathbf{f}}^{(h+1)}(\mathbf{x}) \right]_i \left[ \tilde{\mathbf{f}}^{(h+1)}(\mathbf{x}') \right]_i \right| \tilde{\mathbf{f}}^{(h)}, \mathbf{m}^{(h+1)} \right] &= \inprod{\tilde{\mathbf{g}}^{(h)}(\mathbf{x}) \odot \mathbf{m}^{(h+1)}_i, \tilde{\mathbf{g}}^{(h)}(\mathbf{x}') \odot \mathbf{m}^{(h+1)}_i} \\
    &= \frac{c_\sigma}{d_h} \sum_{j=1}^{d_h} \sigma\left( \left[ \tilde{\mathbf{f}}^{(h)} (\mathbf{x}) \right]_j \right) \sigma\left( \left[ \tilde{\mathbf{f}}^{(h)} (\mathbf{x}') \right]_j \right) \left( \mathbf{m}_{ij}^{(h+1)} \right)^2 \\
    &\xrightarrow[]{d_h \rightarrow \infty} \alpha c_\sigma \E\left[ \sigma\left( \left[ \tilde{\mathbf{f}}^{(h)} (\mathbf{x}) \right]_j \right) \sigma\left( \left[ \tilde{\mathbf{f}}^{(h)} (\mathbf{x}') \right]_j \right) \right],
\end{split}
\end{align}
by the law of large number.

Recall \Cref{def: forward_backward_limit}, we define
\begin{align*}
    \tilde{\Sigma}^{(h)}(\mathbf{x,x'}) := \lim_{d_1, \ldots, d_h \rightarrow \infty} \inprod{\tilde{\mathbf{g}}^{(h)}(\mathbf{x}), \tilde{\mathbf{g}}^{(h)}(\mathbf{x}')}  = \lim_{d_1, \ldots, d_h \rightarrow \infty} \frac{c_\sigma}{d_h} \sum_{j=1}^{d_h} \sigma\left( \left[ \tilde{\mathbf{f}}^{(h)} (\mathbf{x}) \right]_j \right) \sigma\left( \left[ \tilde{\mathbf{f}}^{(h)} (\mathbf{x}') \right]_j \right).
\end{align*}
where the limit is taking sequentially from $d_1$ to $d_h$. 
We further define
\[
\Tilde{\boldsymbol{\Lambda}}^{(1)} =
\begin{bmatrix}
\Tilde{\Sigma}^{(0)}(\mathbf{x,x}) & \Tilde{\Sigma}^{(0)}(\mathbf{x,x'}) \\
\Tilde{\Sigma}^{(0)}(\mathbf{x',x}) & \Tilde{\Sigma}^{(0)}(\mathbf{x,x}) 
\end{bmatrix} ,
\]
\[
\Tilde{\boldsymbol{\Lambda}}^{(h)} = \alpha
\begin{bmatrix}
\Tilde{\Sigma}^{(h-1)}(\mathbf{x,x}) & \Tilde{\Sigma}^{(h-1)}(\mathbf{x,x'}) \\
\Tilde{\Sigma}^{(h-1)}(\mathbf{x',x}) & \Tilde{\Sigma}^{(h-1)}(\mathbf{x,x}) 
\end{bmatrix},
\]

\begin{lemma}[Restatement of \Cref{lemma: main_text_asymp_forward}]\label{lemma: asymp_forward}
Suppose the neural network uses ReLU as its activation and $d_1, d_2, \ldots, d_{L} \rightarrow \infty$ sequentially, then
\begin{align*}
    \Tilde{\Sigma}^{(h)}(\mathbf{x,x}') &= c_\sigma \E_{(u,v) \sim \mathcal{N}(\mathbf{0}, \tilde{\boldsymbol{\Lambda}}^{(h)})} [\sigma(u) \sigma(v)], \\
    \tilde{\Sigma}^{(h)}(\mathbf{x,x'}) &= \alpha^{h-1} \Sigma^{(h)}(\mathbf{x,x'}),
\end{align*}
for $h = 1,2,\ldots, L$.
\end{lemma}
\begin{proof}
We prove by induction. 
First, notice that $\tilde{\Sigma}^{(0)}(\mathbf{x,x'}) = \Sigma^{(0)}(\mathbf{x,x'})$.
When $h = 1$, there is noting to prove. 
Now, assume the induction hypothesis holds for all $h$ such that $h \leq t$ where $t \geq 1$ and we want to show that $\tilde{\Sigma}^{(t+1)}(\mathbf{x,x'}) = \alpha^{t} \Sigma^{(t+1)}(\mathbf{x,x'})$. 
Notice that Equation \ref{eq: correlation} is true for $h \in \{1, \ldots, L\}$. 
Therefore, as $d_t \rightarrow \infty$
\begin{align*}
    \tilde{\Sigma}^{(t+1)}(\mathbf{x,x'}) = c_\sigma \E_{\left[ \tilde{\mathbf{f}}^{(t+1)}(\mathbf{x}) \right]_1, \left[ \tilde{\mathbf{f}}^{(t+1)}(\mathbf{x}') \right]_1}\left[ \sigma\left( \left[ \tilde{\mathbf{f}}^{(t+1)}(\mathbf{x}) \right]_1 \right) \sigma\left( \left[ \tilde{\mathbf{f}}^{(t+1)}(\mathbf{x}') \right]_1 \right) \right].
\end{align*}
Assume all the previous layers are already at the limit, for $t = 1, \ldots, L$,
\begin{align*}
    \left( \left[ \tilde{\mathbf{f}}^{(t+1)}(\mathbf{x}) \right]_1, \left[ \tilde{\mathbf{f}}^{(t+1)}(\mathbf{x}') \right]_1 \right) \sim \mathcal{N}\left(\mathbf{0}, \alpha
    \begin{bmatrix}
    \Tilde{\Sigma}^{(t)}(\mathbf{x,x}) & \Tilde{\Sigma}^{(t)}(\mathbf{x,x'}) \\
    \Tilde{\Sigma}^{(t)}(\mathbf{x',x}) & \Tilde{\Sigma}^{(t)}(\mathbf{x,x}) 
    \end{bmatrix}
    \right)
    = \mathcal{N}(\mathbf{0}, \tilde{\boldsymbol{\Lambda}}^{(t+1)}).
\end{align*}
This proves the first equality. 

By induction hypothesis on $\tilde{\Sigma}^{(t)}(\mathbf{x,x'})$, we have $\tilde{\boldsymbol{\Lambda}}^{(t+1)} = \alpha \cdot \alpha^{t-1} \boldsymbol{\Lambda}^{(t+1)}$. Hence
\begin{align*}
    \tilde{\Sigma}^{(t+1)}(\mathbf{x,x'}) &= c_\sigma \E_{(u,v) \sim \mathcal{N}(\mathbf{0}, \alpha^{t} \boldsymbol{\Lambda}^{(t+1)})} [\sigma(u) \sigma(v)] \\ 
    &= c_\sigma \E_{(u',v') \sim \mathcal{N}(\mathbf{0}, \boldsymbol{\Lambda}^{(t+1)})} [\sigma(\alpha^{\frac{t}{2}} u') \sigma(\alpha^{\frac{t}{2}} v')] \\
    &= \alpha^{t} c_\sigma \E_{(u',v') \sim \mathcal{N}(\mathbf{0}, \boldsymbol{\Lambda}^{(t+1)})} [\sigma(u') \sigma(v')] \\
    &= \alpha^{t} \Sigma^{(t+1)}(\mathbf{x,x'}),
\end{align*}
where the second last inequality is from our assumption that the activation is ReLU.
\end{proof}
This lemma implies that 
\begin{align}\label{eq: inprod_g}
    \tilde{\Sigma}^{(h)}(\mathbf{x,x'}) = \lim_{d_1, \ldots, d_{h} \rightarrow \infty} \inprod{\tilde{\mathbf{g}}^{(h)}(\mathbf{x}), \tilde{\mathbf{g}}^{(h)}(\mathbf{x'})} = \alpha^{h-1} \Sigma^{(h)}(\mathbf{x,x'}).
\end{align}
Thus, combining \Cref{eq: inprod_ntk_elem} and \Cref{eq: inprod_g} we have
\begin{align}\label{eqn: converge_G}
    &\inprod{\left(\tilde{\mathbf{b}}^{(h)}(\mathbf{x}) \left( \tilde{\mathbf{g}}^{(h-1)}(\mathbf{x}) \right)^\top \right) \odot \mathbf{m}^{(h)} , \left(\tilde{\mathbf{b}}^{(h)}(\mathbf{x}') \left( \tilde{\mathbf{g}}^{(h-1)}(\mathbf{x}') \right)^\top \right) \odot \mathbf{m}^{(h)}} \nonumber \\
    &= \left(\tilde{\mathbf{b}}^{(h)} (\mathbf{x}) \right)^\top \mathbf{G}^{(h-1)} \tilde{\mathbf{b}}^{(h)} (\mathbf{x}') \nonumber \\
    &\xrightarrow[]{d_1, \ldots, d_{h-1} \rightarrow \infty} \alpha^{h-1} {\Sigma}^{(h-1)}(\mathbf{x,x'}) \lim_{d_1, \ldots, d_{h-1} \rightarrow \infty} \left(\tilde{\mathbf{b}}^{(h)} (\mathbf{x}) \right)^\top \tilde{\mathbf{b}}^{(h)} (\mathbf{x}').
\end{align}

\begin{lemma}[Restatement of \Cref{lemma: main_text_asymp_backward}]\label{lemma: asymp_backward}
Assume we use a fresh sample of weights in the backward pass, then
\begin{align}\label{eq: inprod_b}
    \lim_{d_1, \ldots, d_L \rightarrow \infty} \inprod{\tilde{\mathbf{b}}^{(h)}(\mathbf{x}), \tilde{\mathbf{b}}^{(h)}(\mathbf{x'})} = \alpha^{L+1-h} \prod_{h'=h}^L \dot{\Sigma}^{(h')}(\mathbf{x,x'}).
\end{align}
\end{lemma}
\begin{proof}
For the factor $\inprod{\tilde{\mathbf{b}}^{(h)} (\mathbf{x}), \tilde{\mathbf{b}}^{(h)}(\mathbf{x'})}$, we expand using the definition of $\tilde{\mathbf{b}}^{(h)} (\mathbf{x})$
\begin{align*}
    & \inprod{\tilde{\mathbf{b}}^{(h)}(\mathbf{x}), \tilde{\mathbf{b}}^{(h)}(\mathbf{x'})} \\
    &= \inprod{\sqrt{\frac{c_\sigma}{d_h}} \tilde{\mathbf{D}}^{(h)}(\mathbf{x}) \left( \mathbf{W}^{(h+1)} \odot \mathbf{m}^{(h+1)} \right)^\top \tilde{\mathbf{b}}^{(h+1)}(\mathbf{x}), \sqrt{\frac{c_\sigma}{d_h}} \tilde{\mathbf{D}}^{(h)}(\mathbf{x}') \left( \mathbf{W}^{(h+1)} \odot \mathbf{m}^{(h+1)} \right)^\top \tilde{\mathbf{b}}^{(h+1)}(\mathbf{x}')}
\end{align*}
First we analyze $\tilde{\mathbf{D}}^{(h)}(\mathbf{x})$. 
Since we use ReLU as the activation function, $\dot{\sigma}(x) = \mathbb{I}(x > 0)$ and in particular, $\dot{\sigma}(c x) = \mathbb{I}(c x > 0) = \mathbb{I}(x > 0) = \dot{\sigma}(x)$ for any positive constant $c$.
By Lemma \ref{lemma: asymp_forward}, we show that under sequential limit, $\tilde{\mathbf{f}}^{(h)}(\mathbf{x}) $ has the same distribution as $ \alpha^{h-1} \mathbf{f}^{(h)}(\mathbf{x})$ which implies $\tilde{\mathbf{D}}^{(h)}(\mathbf{x})$ has the same distribution as $\mathbf{D}^{(h)}(\mathbf{x})$.

Observe that $\mathbf{W}^{(h+1)} \odot \mathbf{m}^{(h+1)}$ and $\tilde{\mathbf{b}}^{(h+1)}(\mathbf{x})$ are dependent.
Now we apply the independent copy trick which is rigorously justified for ReLU network with Gaussian weights by replacing $\mathbf{W}^{(h+1)}$ with a fresh new sample $\widetilde{\mathbf{W}}^{(h+1)}$.
\begin{align}\label{eqn: convergence_b}
    & \inprod{\tilde{\mathbf{b}}^{(h)}(\mathbf{x}), \tilde{\mathbf{b}}^{(h)}(\mathbf{x'})} \nonumber\\
    &= \inprod{\sqrt{\frac{c_\sigma}{d_h}} \tilde{\mathbf{D}}^{(h)}(\mathbf{x}) \left( \mathbf{W}^{(h+1)} \odot \mathbf{m}^{(h+1)} \right)^\top \tilde{\mathbf{b}}^{(h+1)}(\mathbf{x}), \sqrt{\frac{c_\sigma}{d_h}} \tilde{\mathbf{D}}^{(h)}(\mathbf{x}') \left( \mathbf{W}^{(h+1)} \odot \mathbf{m}^{(h+1)} \right)^\top \tilde{\mathbf{b}}^{(h+1)}(\mathbf{x}')} \nonumber \\
    &\approx \inprod{\sqrt{\frac{c_\sigma}{d_h}} \tilde{\mathbf{D}}^{(h)}(\mathbf{x}) \left( \widetilde{\mathbf{W}}^{(h+1)} \odot {\mathbf{m}}^{(h+1)} \right)^\top \tilde{\mathbf{b}}^{(h+1)}(\mathbf{x}), \sqrt{\frac{c_\sigma}{d_h}} \tilde{\mathbf{D}}^{(h)}(\mathbf{x}') \left( \widetilde{\mathbf{W}}^{(h+1)} \odot {\mathbf{m}}^{(h+1)} \right)^\top \tilde{\mathbf{b}}^{(h+1)}(\mathbf{x}')} \nonumber \\
    &\xrightarrow[]{d_1, \ldots, d_{h} \rightarrow \infty} \alpha\frac{c_\sigma}{d_h} \tr\left( \tilde{\mathbf{D}}^{(h)}(\mathbf{x}) \tilde{\mathbf{D}}^{(h)}(\mathbf{x}') \right) \lim_{d_1, \ldots, d_{h} \rightarrow \infty} \inprod{\tilde{\mathbf{b}}^{(h+1)}(\mathbf{x}), \tilde{\mathbf{b}}^{(h+1)}(\mathbf{x'}) } \nonumber \\
    &\xrightarrow[]{d_1, \ldots, d_{h} \rightarrow \infty} \alpha \dot{\Sigma}^{(h)}(\mathbf{x,x'}) \lim_{d_1, \ldots, d_{h} \rightarrow \infty} \inprod{\tilde{\mathbf{b}}^{(h+1)}(\mathbf{x}), \tilde{\mathbf{b}}^{(h+1)}(\mathbf{x'}) } .
\end{align}
where we justify the limit as the following: first let $\mathbf{D}$ short for $\tilde{\mathbf{D}}^{(h)}(\mathbf{x}) \tilde{\mathbf{D}}^{(h)}(\mathbf{x'})$
\begin{align*}
    & \left( \frac{c_\sigma}{d_h} (\widetilde{\mathbf{W}}^{(h+1)} \odot \mathbf{m}^{(h+1)}) \mathbf{D} (\widetilde{\mathbf{W}}^{(h+1)} \odot \mathbf{m}^{(h+1)})^\top \right)_{ij} = \frac{c_\sigma}{d_h} \sum_k \mathbf{D}_{kk} \widetilde{\mathbf{W}}^{(h+1)}_{ik} \mathbf{m}^{(h+1)}_{ik} \widetilde{\mathbf{W}}^{(h+1)}_{jk} \mathbf{m}^{(h+1)}_{jk},
\end{align*}
which converges to a diagonal matrix as $d_h \rightarrow \infty$.
Thus, the inner product is given by
\begin{align*} 
    & \frac{c_\sigma}{d_h} \sum_{i,j} \tilde{\mathbf{b}}_i^{(h+1)} (\mathbf{x}) \tilde{\mathbf{b}}_j^{(h+1)} (\mathbf{x}') \sum_{k} \mathbf{D}_{kk} \widetilde{\mathbf{W}}_{ik}^{(h+1)} \mathbf{m}_{ik}^{(h+1)} \widetilde{\mathbf{W}}_{jk}^{(h+1)} \mathbf{m}_{jk}^{(h+1)} \nonumber \\
    &= \frac{c_\sigma}{d_h} \sum_{i,j} \tilde{\mathbf{b}}_i^{(h+1)} (\mathbf{x}) \tilde{\mathbf{b}}_j^{(h+1)}(\mathbf{x}') (\tilde{\mathbf{w}}_{i}^{(h+1)} \odot \mathbf{m}_{i}^{(h+1)})^\top \mathbf{D} (\tilde{\mathbf{w}}_{j}^{(h+1)} \odot \mathbf{m}_{j}^{(h+1)}) \nonumber \\
    &= \frac{c_\sigma}{d_h} \sum_{i,j} \tilde{\mathbf{b}}_i^{(h+1)}(\mathbf{x}) \tilde{\mathbf{b}}_j^{(h+1)}(\mathbf{x}') \left(\tilde{\mathbf{w}}_{i}^{(h+1)} \right)^\top \mathbf{M}_{i}^{(h+1)} \mathbf{D} \mathbf{M}_{j}^{(h+1)} \tilde{\mathbf{w}}_{j}^{(h+1)} \nonumber \\
    &\xrightarrow[]{d_1, \ldots, d_{h} \rightarrow \infty} \frac{c_\sigma}{d_h} \sum_i \tilde{\mathbf{b}}_i^{(h+1)}(\mathbf{x}) \tilde{\mathbf{b}}_i^{(h+1)}(\mathbf{x}') \tr(\mathbf{M}_{i}^{(h+1)} \mathbf{D} \mathbf{M}_{i}^{(h+1)}) \nonumber \\
    &\xrightarrow[]{d_1, \ldots, d_{h} \rightarrow \infty} \alpha \dot{\Sigma}^{(h)}(\mathbf{x,x'}) \lim_{d_1, \ldots, d_h \rightarrow \infty} \inprod{\tilde{\mathbf{b}}^{(h+1)}(\mathbf{x}), \tilde{\mathbf{b}}^{(h+1)}(\mathbf{x'}) } ,
\end{align*}
where $\mathbf{M}_i = \textnormal{diag}(\mathbf{m}_i)$ and $\tilde{\mathbf{w}}_i$ is the i-th row of $\widetilde{\mathbf{W}}$ and $\lim_{d_h \rightarrow \infty} \frac{c_\sigma}{d_h} \tr(\mathbf{M}_i^{(h+1)} \mathbf{D} \mathbf{M}_i^{(h+1)}) = \alpha \dot{\Sigma}^{(h)}(\mathbf{x,x'})$.
Now, we can unroll the formula of $\inprod{\tilde{\mathbf{b}}^{(h)}(\mathbf{x}), \tilde{\mathbf{b}}^{(h)}(\mathbf{x'})}$ in Equation \eqref{eqn: convergence_b}, we have
\begin{align*}
    \lim_{d_1, \ldots, d_L \rightarrow \infty} \inprod{\tilde{\mathbf{b}}^{(h)}(\mathbf{x}), \tilde{\mathbf{b}}^{(h)}(\mathbf{x'})} = \alpha^{L+1-h} \prod_{h'=h}^L \dot{\Sigma}^{(h')}(\mathbf{x,x'}).
\end{align*}
\end{proof}
\begin{proof}[Proof of \Cref{thm: main_asymp}]
Combining the result in Equation \eqref{eqn: converge_G} and Equation \eqref{eq: inprod_b}, we have
\begin{align*}
    \inprod{\frac{\partial \tilde{f}(\mathbf{x})}{\partial \mathbf{W}^{(h)}}, \frac{\partial \tilde{f}(\mathbf{x}')}{\partial \mathbf{W}^{(h)} }} &= \inprod{\left(\tilde{\mathbf{b}}^{(h)}(\mathbf{x}) \left( \tilde{\mathbf{g}}^{(h-1)}(\mathbf{x}) \right)^\top \right) \odot \mathbf{m}^{(h)} , \left(\tilde{\mathbf{b}}^{(h)}(\mathbf{x}') \left( \tilde{\mathbf{g}}^{(h-1)}(\mathbf{x}') \right)^\top \right) \odot \mathbf{m}^{(h)}} \\
    &\xrightarrow[]{d_{1}, \ldots, d_L \rightarrow \infty} \alpha^{L} \Sigma^{(h-1)}(\mathbf{x,x'}) \prod_{h'=h}^{L+1} \dot{\Sigma}^{(h')}(\mathbf{x,x'}) .
\end{align*}
We conclude
\begin{align}
    \tilde{\boldsymbol{\Theta}}_\infty(\mathbf{x,x'}) := \lim_{d_1, d_2, \ldots, d_L \rightarrow \infty}  \tilde{\boldsymbol{\Theta}}(\mathbf{x,x'}) = \alpha^{L} \boldsymbol{\Theta}_\infty(\mathbf{x,x'}),
\end{align}
which proves \Cref{thm: main_asymp}. 
\end{proof}

\subsection{Proof of \texorpdfstring{\Cref{lemma: main_text_exchange_limit}}{Exchange Limit}: Going from Asymptotic Regime to Non-Asymptotic Regime}\label{app: asym_to_nonasym}
Before we give proof for our non-asymptotic result, we note that our asymptotic result is obtained from taking sequential limits of all the hidden layers which is a somewhat a limited notion of limits since we assume all the layer before is already at the limit when we deal with a given layer. 
Non-asymptotic analysis, on the other hand, consider using a large but finite amount of samples to get close to (but not exactly at) the limit. 
Thus, we need to justify that the networks are indeed able to approach by increasing width.
In mathematical language, this is the same as justifying taking the limit outside of $\E\sigma(\cdot), \E\dot{\sigma}(\cdot)$. 

We invoke several results from measure-theoretic probability theory.
\begin{definition}[Uniformly integrable]\label{def: uniformly_integrable}
A sequence of random variables $\{X_n\}$ is called \emph{uniformly integrable} if
\begin{align*}
    \lim_{a \rightarrow \infty} \sup_n \E[|X_n| \mathbb{I}(|X_n| \geq a)] \rightarrow 0.
\end{align*}
\end{definition}
\begin{lemma}[Theorem 3, Chapter 7.10 in \citep{grimmett2020probability}]\label{lemma: uniform_integrable_to_convergence_L1}
Suppose that $\{X_n\}$ is a sequence of random variables satisfying $X_n \rightarrow X$ in probability. 
The following statements are equivalent:
\begin{enumerate}
    \item The family $\{X_n\}$ is uniformly integrable. 
    \item $\E|X_n| < \infty$ for all $n$ and $\E|X_n| \rightarrow \E|X| < \infty$. 
\end{enumerate}
\end{lemma}

\begin{theorem}[Skorokhod's Representation Theorem, \citep{billingsley1999convergence}]\label{thm: skorokhod}
Let $\{\mu_n\}$ be a sequence of probability measure defined on a metric space $S$ such that $\mu_n$ converges weakly to some probability measure $\mu_\infty$ on $S$ as $n \rightarrow \infty$.
Suppose that the support of $\mu_\infty$ is separable. 
Then there exists $S$-valued random variables $X_n$ defined on a common probability space $(\Omega, \mathcal{F}, \mathbb{P})$ such that the law of $X_n$ is is $\mu_n$ for all $n$ (including $n = \infty$) and such that $(X_n)_{n \in \mathbb{N}}$ converges to $X_\infty$, $\mathbb{P}$-almost surely.
\end{theorem}

\begin{theorem}[Continuous Mapping Theorem \citep{mann1943stochastic}]
Let $\{X_n\}, X$ be random variables defined on a metric space $S$. 
Suppose a function $g: S \rightarrow S'$ (where $S'$ is another metric space) has the set of discontinuities of measure zero. Then
\begin{align*}
    X_n \xrightarrow[]{\mathcal{D}} X \quad \Rightarrow \quad g(X_n) \xrightarrow[]{\mathcal{D}} g(X),
\end{align*}
where $\xrightarrow[]{\mathcal{D}}$ represents convergence in distribution. 
\end{theorem}

\begin{lemma}[Restatement of \Cref{lemma: main_text_exchange_limit}]\label{lemma: exchange_limit}
Conditioned on $\mathbf{g}^{(h-1)}(\mathbf{x}), \mathbf{g}^{(h-1)}(\mathbf{x}')$.
Fix $i \in [d_{h+1}]$.
Let 
\[
X_n = 
\begin{bmatrix}
\sqrt{\frac{c_\sigma}{n}} \sum_{j=1}^n \mathbf{W}^{(h+1)}_{ij} \mathbf{m}^{(h+1)}_{ij} \sigma(\tilde{\mathbf{f}}_j^{(h)}(\mathbf{x})) \\
\sqrt{\frac{c_\sigma}{n}} \sum_{j=1}^n \mathbf{W}^{(h+1)}_{ij} \mathbf{m}^{(h+1)}_{ij} \sigma(\tilde{\mathbf{f}}_j^{(h)}(\mathbf{x}'))
\end{bmatrix}
\in \R^2,
\]
and define let $g: \R^2 \rightarrow \R$ to be $g(x,y) \in \{ \sigma(x) \sigma(y), \dot{\sigma}(x) \dot{\sigma}(y)\}$.
Then, 
\begin{align*}
    \lim_{n \rightarrow \infty} \E[g(X_n)] = \E[g(\lim_{n \rightarrow \infty} X_n)].
\end{align*}
\end{lemma}
\begin{proof}
First of all, conditioned on $\mathbf{g}^{(h-1)}(\mathbf{x}), \mathbf{g}^{(h-1)}(\mathbf{x}')$, we have $\tilde{\mathbf{f}}_j^{(h)}(\mathbf{x}), \tilde{\mathbf{f}}_j^{(h)}(\mathbf{x}')$ are i.i.d. random variables for all $j \in [n]$. 

We first prove the exchange of limit for $g(x,y) = \sigma(x) \sigma(y)$ since this function is continuous. 
By the Central Limit Theorem, $X_n \xrightarrow[]{\mathcal{D}} X_\infty \sim \mathcal{N}(\mathbf{0}, \tilde{\boldsymbol{\Lambda}}^{(h+1)})$.
By the Continuous Mapping Theorem, $g(X_n) \xrightarrow[]{\mathcal{D}} g(X_\infty)$.
Then by the Skorokhod's Representation Theorem in \Cref{thm: skorokhod}, there exists another sequence $\{X_n'\}$ and $X_\infty'$ such that $g(X_n) \stackrel{\mathcal{D}}{=} X_n'$ and $g(X_\infty) \stackrel{\mathcal{D}}{=} X_\infty'$ and $X_n' \xrightarrow{a.s.} X_\infty'$. 
Now we use the fact that the sequence $\{X_n'\}$ is uniformly integrable (see \Cref{def: uniformly_integrable}). 
By \Cref{lemma: uniform_integrable_to_convergence_L1}, this implies convergence in $L^1$ (and notice that $g(x,y)$ only outputs non-negative values)
\begin{align*}
    \lim_{n \rightarrow \infty} \E[X_n'] = \E[X_\infty'].
\end{align*}
Since 
\begin{align*}
    \E[g(X_n)] &= \E[X_n'], \\
    \E[g(X_\infty)] &= \E[X_\infty'],
\end{align*}
we have 
\begin{align*}
    \lim_{n \rightarrow \infty} \E[g(X_n)] = \E[g(X_\infty)] =  \E[g(\mathcal{N}(\mathbf{0}, \tilde{\boldsymbol{\Lambda}}^{(h+1)}))].
\end{align*}

Now we prove the result for $g(x,y) = \dot{\sigma}(x) \dot{\sigma(y)} = \mathbb{I}(x \geq 0, y \geq 0)$. 
Again, apply Skorokhod's Representation Theorem, there exists a sequence of random variables $\{X_n''\}$ and another random variable $X_\infty''$ such that $X_n \stackrel{\mathcal{D}}{=} X_n''$ and $X_\infty \stackrel{\mathcal{D}}{=} X_\infty''$ and $X_n'' \xrightarrow[]{a.s.} X_\infty''$. 
Since convergence almost surely implies convergence in probability, we have
\begin{align*}
    \lim_{n \rightarrow \infty} \E[g(X_n'')] = \E[g(X_\infty'')],
\end{align*}
which implies
\begin{align*}
    \lim_{n \rightarrow \infty} \E[g(X_n)] = \E[g(X_\infty)].
\end{align*}
\end{proof}

\section{NON-ASYMPTOTIC ANALYSIS (Proof of \texorpdfstring{\Cref{thm: main_text_main_non_asymp}}{Nonasymptotic Bound})}\label{app: nonasymp}
\subsection{Probability}\label{app: probabilities}
\begin{theorem}[Multiplicative Chernoff Bound]\label{thm: multiplicative_chernoff}
If $X_1, X_2, \ldots, X_m$ are i.i.d. Bernoulli random variables with probability $p$, then
\begin{align*}
    \Pr\left[ \left| \sum_{i=1}^m X_i - pm \right| \geq \eps pm \right] \leq 2 \exp(-\min(\eps^2, \eps) pm).
\end{align*}
\end{theorem}

\begin{theorem}\label{thm: concentration_subgaussian_bernoulli}
Assume $X_1, \ldots, X_m$ are i.i.d. Sub-Gaussian random variables with variance proxy $\sigma^2$ and $Y_1, \ldots, Y_m$ are i.i.d. Bernoulli random variables with probability $p$.
For $\eps \in (0, 1/2),\ t > 0$, 
\begin{align*}
    \Pr\left[ \left| \frac{1}{pm} \sum_{i=1}^m X_i Y_i - \E[X] \right| \geq \eps(|\E[X]| + t) + t \right] \leq 2 \exp(-(1-\eps)pm t^2 / (2\sigma^2)) + 2 \exp(-\min(\eps^2, \eps) pm).
\end{align*}
\end{theorem}
\begin{proof}
Let $\hat{p} = \frac{\sum_{i=1}^m Y_i}{m}$. 
By the concentration of Sub-Gaussian random variable with variance proxy $\sigma^2$, we have
\begin{align*}
    \Pr\left[ \left| \frac{1}{\hat{p}m} \sum_{i=1}^m X_i Y_i - \E[X] \right| \geq t \right] \leq 2 \exp(-\hat{p}m t^2 / (2\sigma^2)) + 2 \exp(-\min(\eps^2, \eps) pm).
\end{align*}
By \Cref{thm: multiplicative_chernoff}, we have with probability at least $1 - 2 \exp(-\min(\eps^2, \eps) pm)$,
$\hat{p} = (1 \pm \eps) p$.
Thus, with probability at least $1 - 2 \exp(-\hat{p}m t^2 / (2\sigma^2)) - 2 \exp(-\min(\eps^2, \eps) pm)$,
\begin{align*}
    \frac{1}{pm} \sum_{i=1}^m X_i Y_i = \frac{\hat{p}}{p} \frac{1}{\hat{p}m} \sum_{i=1}^m X_i Y_i = (1 \pm \eps)(\E[X] \pm t).
\end{align*}
\end{proof}

\begin{theorem}\label{thm: concentration_subgamma_bernoulli}
Assume $X_1, \ldots, X_m$ are i.i.d. Sub-Gamma random variables with parameters $(\sigma^2, c)$ and $Y_1, \ldots, Y_m$ are i.i.d. Bernoulli random variables with probability $p$.
For $\eps \in (0, 1/2),\ t > 0$, 
\begin{align*}
    \Pr\left[ \left| \frac{1}{pm} \sum_{i=1}^m X_i Y_i - \E[X] \right| \geq \eps(|\E[X]| + t) + t \right] \leq 2 \exp(-(1-\eps)pm \min(t^2 / (2\sigma^2), t/c)) + 2 \exp(-\min(\eps^2, \eps) pm).
\end{align*}
\end{theorem}
\begin{proof}
By the concentration of Sub-Gamma random variables, we have
\begin{align*}
    \Pr\left[ \left| \frac{1}{\hat{p}m} \sum_{i=1}^m X_i Y_i - \E[X] \right| \geq t \right] \leq 2 \exp(-\hat{p}m \min(t^2 / (2\sigma^2), t/c)).
\end{align*}
The rest of proof follows from the proof of \Cref{thm: concentration_subgaussian_bernoulli}. 
\end{proof}

\begin{lemma}[Gaussian Chaos of Order 2 \citep{boucheron2013concentration}]\label{lemma: gaussian_chaos}
Let $\boldsymbol{\xi} \sim \mathcal{N}(\mathbf{0}, \mathbf{I}_n)$ be an $n$-dimensional unit Gaussian random vector, $\mathbf{A} \in \R^{n \times n}$ be a symmetric matrix, then for any $t > 0$, 
\begin{align*}
    \Pr\left[ |\boldsymbol{\xi}^\top \mathbf{A} \boldsymbol{\xi} - \E[\boldsymbol{\xi}^\top \mathbf{A} \boldsymbol{\xi}] | > 2 \norm{\mathbf{A}}_F \sqrt{t} + 2 \norm{\mathbf{A}}_2 t \right] \leq 2 \exp(-t).
\end{align*}
Equivalently,
\begin{align*}
    \Pr \left[ |\boldsymbol{\xi}^\top \mathbf{A} \boldsymbol{\xi} - \E[\boldsymbol{\xi}^\top \mathbf{A} \boldsymbol{\xi}] | > t \right] \leq 2 \exp \left( - \frac{t^2}{4 \norm{\mathbf{A}}_F^2 + \norm{\mathbf{A}}_2 t} \right).
\end{align*}
\end{lemma}

\begin{lemma}[Example 2.30 in \citep{wainwright2019high}]\label{lemma: gaussian_complexity_concentration}
Let $\mathbf{w} \sim \mathcal{N}(\mathbf{0, I}_d)$ and $\mathcal{A}$ be a set in $\R^d$. 
Then $\sup_{\mathbf{a} \in \mathcal{A}} \inprod{\mathbf{a, w}}$ is a sub-Gaussian random variable with variance proxy $\sup_{\mathbf{a} \in \mathcal{A}} \norm{\mathbf{a}}_2^2$. 
\end{lemma}

\subsection{Other Auxiliary Results}
\begin{lemma}[Lemma E.2 in \citep{arora2019exact}]
For events $\mathcal{A, B}$, define the event $\mathcal{A} \Rightarrow \mathcal{B}$ as $\neg \mathcal{A} \vee \mathcal{B}$. Then $\Pr[\mathcal{A} \Rightarrow \mathcal{B}] \geq \Pr[\mathcal{B} | \mathcal{A}]$.
\end{lemma}

\begin{lemma}[Lemma E.3 in \citep{arora2019exact}]\label{lemma: fresh_gaussian}
Let $\mathbf{w} \sim \mathcal{N}(\mathbf{0}, \mathbf{I}_d)$, $\mathbf{G} \in \R^{d \times k}$ be some fixed matrix, and random vector $\mathbf{F = w^\top G}$, then conditioned on the value of $\mathbf{F}$, $\mathbf{w}$ remains Gaussian in the null space of the column space of $\mathbf{G}$, i.e.,
\begin{align*}
    \Pi_{\mathbf{G}}^\bot \mathbf{w} \stackrel{\mathcal{D}}{=}_{\mathbf{F = w^\top G}} \Pi_{\mathbf{G}}^\bot \tilde{\mathbf{w}}.
\end{align*}
where $\tilde{\mathbf{w}} \sim \mathcal{N}(\mathbf{0,I}_d)$ is a fresh i.i.d. copy of $\mathbf{w}$.
\end{lemma}

\subsection{Proof of the Main Result}
Now we prove our main result. 
Notice that we rescale the mask $\mathbf{m}_{ij}^{(h)} \sim \sqrt{\frac{1}{\alpha}} \textnormal{Bernoulli}(\alpha)$ so that $\E (\mathbf{m}_{ij}^{(h)})^2 = 1$.
From a high level, our proof follows the proof outline of our asymptotic result. 

\begin{theorem}[Non-Asymptotic Bound, Full Version of \Cref{thm: main_text_main_non_asymp}]\label{thm: main_non_asymp}
Consider an $L$-hidden-layer fully-connected ReLU neural network with all the weights initialized with i.i.d. standard Gaussian distribution. 
Suppose all the weights except the input layer are pruned with probability $1 - \alpha$ at the initialization and after pruning we rescale the weights by $1/\sqrt{\alpha}$.
For $\delta \in (0,1)$ and sufficiently small $\eps >0 $, if
\begin{align*}
    d_h \geq \Omega\left(\max( \frac{1}{\alpha} \frac{L^6}{\eps^4} \log \frac{L d_{h+1} }{\delta} , \frac{1}{\alpha^2} \frac{L^2}{\eps^2} \log \frac{L d_{h+1} \sum_{h'=1}^{L-1} d_{h'}}{\delta}, \frac{1}{\alpha} \frac{L^4}{\eps^2} \log \frac{2L d_{h+1} \sum_{h'=1}^{h-1} d_h' }{\delta_3}) \right),\ \forall h \in [L].
\end{align*} 
Then for any inputs $\mathbf{x,x'} \in \R^{d_0}$ such that $\norm{\mathbf{x}}_2 \leq 1,\ \norm{\mathbf{x}'}_2 \leq 1$, with probability at least $1 - \delta$ we have
\begin{align*}
    \left| \inprod{\frac{\partial f(\boldsymbol{\theta}, \mathbf{x})}{\partial \boldsymbol{\theta}}, \frac{\partial f(\boldsymbol{\theta}, \mathbf{x'})}{\partial \boldsymbol{\theta}}} - \boldsymbol{\Theta}^{(L)}(\mathbf{x,x'}) \right| \leq (L+1) \eps.
\end{align*}
\end{theorem}

Our analysis conditions on the following event occur. 
\begin{lemma}\label{lemma: num_remained_weights}
For $\eps \in (0, 1/2), \delta \in (0,1)$, if $ d_h \geq \Omega(\frac{1}{\alpha\eps^2} \cdot \log (\frac{2d_{h+1} L}{\delta}))$, then
\begin{align*}
    \Pr\left[\forall i \in [d_{h+1}],\ h \in [L]: \left| \sum_{j=1}^{d_h} \mathbb{I}(\mathbf{m}^{(h)}_{ij} \neq 0) - \alpha d_h \right| \geq \eps \alpha d_h \right] \leq \delta 
\end{align*}
\end{lemma}
\begin{proof}
The proof is by applying \Cref{thm: multiplicative_chernoff} and then take a union bound over $i \in [d_{h+1}], h \in [L]$. 
\end{proof}

Let $\mathbf{m}^{(h)}_i$ denote the $i$-th row of the mask in $h$-th layer.  
We first define the following events:
\begin{itemize}
    \item $\mathcal{A}^h_i (\mathbf{x,x'}, \eps_1):= \left\{ \left|\left(\mathbf{g}^{(h)}(\mathbf{x}) \odot \mathbf{m}^{(h+1)}_i \right)^\top \left( \mathbf{g}^{(h)}(\mathbf{x'}) \odot \mathbf{m}_i^{(h+1)}\right) - \boldsymbol{\Sigma}^{(h)} (\mathbf{x,x'}) \right| \leq \eps_1 \right\}$.
    \item $\mathcal{A}^h(\mathbf{x,x'}, \eps_1) = \bigcap_{i=1}^{d_{h+1}} \mathcal{A}_i^h(\mathbf{x,x'}, \eps_1) \bigcap \left\{ \left|\left(\mathbf{g}^{(h)}(\mathbf{x}) \right)^\top \mathbf{g}^{(h)}(\mathbf{x'}) - \boldsymbol{\Sigma}^{(h)} (\mathbf{x,x'}) \right| \leq \eps_1 \right\}$.
    \item $\overline{\mathcal{A}}^h( \eps_1) = \mathcal{A}^h(\mathbf{x,x}, \eps_1) \cap \mathcal{A}^h(\mathbf{x,x'}, \eps_1) \cap \mathcal{A}^h (\mathbf{x',x'}, \eps_1)$.
    \item $\overline{\mathcal{A}}(\eps_1) = \bigcap_{h = 0}^L \overline{\mathcal{A}}^h(\mathbf{x,x'}, \eps_1)$.
    \item $\mathcal{B}^h(\mathbf{x,x'}, \eps_2) = \left\{ \left| \inprod{\mathbf{b}^{(h)}(\mathbf{x}), \mathbf{b}^{(h)}(\mathbf{x}')} - \prod_{h=h}^L \dot{\Sigma}^{(h)}(\mathbf{x,x'}) \right| < \eps_2 \right\}$.
    \item $\overline{\mathcal{B}}^h( \eps_2) = \mathcal{B}^h(\mathbf{x,x}, \eps_2) \cap \mathcal{B}^h(\mathbf{x,x'}, \eps_2) \cap \mathcal{B}^h(\mathbf{x',x'}, \eps_2)$.
    \item $\overline{\mathcal{B}}(\eps_2) = \bigcap_{h=1}^{L+1} \overline{\mathcal{B}}^h (\mathbf{x,x'}, \eps_2)$.
    \item $\overline{\mathcal{C}}(\eps_3)$: a event defined in \Cref{def: pseudo_network_event}.
    \item $\mathcal{D}_i^h(\mathbf{x,x'}, \eps_4) = \left\{\left| 2 \frac{\tr(\mathbf{M}_i^{(h+1)} \mathbf{D}^{(h)}(\mathbf{x,x'}) \mathbf{M}_i^{(h+1)}) }{d_h} - \dot{\Sigma}^{(h)}(\mathbf{x,x'}) \right| < \eps_4 \right\}$ where $\mathbf{M}_i^{(h+1)} = \textnormal{diag}(\mathbf{m}_i^{(h+1)})$.
    \item $\mathcal{D}^h(\mathbf{x,x'},\eps_4) = \bigcap_{i=1}^{d_{h+1}} \mathcal{D}_i^h(\mathbf{x,x'}, \eps_4)$.
    \item $\overline{\mathcal{D}}^h( \eps_4) = \mathcal{D}^h(\mathbf{x,x}, \eps_4) \cap \mathcal{D}^h(\mathbf{x,x'}, \eps_4) \cap \mathcal{D}^h(\mathbf{x',x'}, \eps_4)$.
    \item $\overline{\mathcal{D}}(\eps_4) = \bigcap_{h=1}^{L+1} \overline{\mathcal{D}}^h(\eps_4)$.
\end{itemize}
\begin{proof}[Proof of \Cref{thm: main_non_asymp}]
Recall that
\begin{align*}
    \inprod{\frac{\partial \tilde{f}(\mathbf{x})}{\partial \mathbf{W}^{(h)}}, \frac{\partial \tilde{f}( \mathbf{x}')}{\partial \mathbf{W}^{(h)}}} &= \left(\tilde{\mathbf{b}}^{(h)} (\mathbf{x}) \right)^\top \mathbf{G}^{(h-1)} \tilde{\mathbf{b}}^{(h)} (\mathbf{x}') ,
\end{align*}
where $\mathbf{G}^{(h-1)}$ is a diagonal matrix and $\mathbf{G}^{(h-1)}_{ii} = \inprod{\tilde{\mathbf{g}}^{(h-1)} (\mathbf{x})\odot \mathbf{m}^{(h)}_i, \tilde{\mathbf{g}}^{(h-1)} (\mathbf{x}')\odot \mathbf{m}^{(h)}_i}$ and
\begin{align*}
    \lim_{d_1, d_2, \ldots, d_L \rightarrow \infty} \inprod{\frac{\partial \tilde{f}(\mathbf{x})}{\partial \mathbf{W}^{(h)}}, \frac{\partial \tilde{f}( \mathbf{x}')}{\partial \mathbf{W}^{(h)}}} = {\Sigma}^{(h-1)}(\mathbf{x}^{(1)}, \mathbf{x}^{(2)}) \prod_{h' = h}^L \dot{\Sigma}^{(h')} (\mathbf{x}^{(1)}, \mathbf{x}^{(2)}).
\end{align*}
The rest of proof of our main result is based on letting \Cref{thm: main_theorem_intermediate} hold for $\eps'$ and then take $\eps:= \eps'/L$. 
\end{proof}
\begin{theorem}\label{thm: main_theorem_intermediate}
Consider the same setting as in \Cref{thm: main_non_asymp}. 
If 
\begin{align*}
    d_h \geq \Omega\left(\max( \frac{1}{\alpha} \frac{L^2}{\eps^4} \log \frac{L d_{h+1} }{\delta} , \frac{1}{\alpha^2} \frac{1}{\eps^2} \log \frac{L d_{h+1} \sum_{h'=1}^{L-1} d_{h'}}{\delta}, \frac{1}{\alpha} \frac{L^2}{\eps^2} \log \frac{2L d_{h+1} \sum_{h'=1}^{h-1} d_h' }{\delta_3}) \right),\ \forall h \in [L],
\end{align*} 
and $\eps \leq \frac{c}{L}$ for some constant $c$, then for any fixed $\mathbf{x,x'} \in \R^{d_0},\ \norm{\mathbf{x}}_2,\ \norm{\mathbf{x}'}_2 \leq 1$, we have with probability $1 - \delta$, $\forall 0 \leq h \leq L,\ \forall (\mathbf{x}^{(1)}, \mathbf{x}^{(2)}) \in \{(\mathbf{x}, \mathbf{x}), (\mathbf{x}, \mathbf{x'}), (\mathbf{x'}, \mathbf{x'})\}$, 
\begin{align*}
    \left| \left( \mathbf{g}^{(h)}(\mathbf{x}^{(1)}) \odot \mathbf{m}^{(h+1)}_i \right)^\top \left( \mathbf{g}^{(h)}(\mathbf{x}^{(2)}) \odot \mathbf{m}^{(h+1)}_i \right) - {\Sigma}^{(h)}(\mathbf{x}^{(1)}, \mathbf{x}^{(2)}) \right| \leq \eps^2/2, \quad \forall i \in [d_{h+1}],
\end{align*}
and
\begin{align*}
    \left| \inprod{\mathbf{b}^{(h)}(\mathbf{x}^{(1)}), \mathbf{b}^{(h)}(\mathbf{x}^{(2)})} - \prod_{h' = h}^L \dot{\Sigma}^{(h')} (\mathbf{x}^{(1)}, \mathbf{x}^{(2)}) \right|< 3L\eps.
\end{align*}
In other words, 
\begin{align*}
    \Pr\left[ \overline{\mathcal{A}}\left( \frac{\eps^2}{2} \right) \bigcap \overline{\mathcal{B}}(3L\eps) \right] \geq 1 - \delta.
\end{align*}
\end{theorem}
The first part of the result of \Cref{thm: main_theorem_intermediate} is proved by the following theorem. 
\begin{theorem}[Full Version of \Cref{thm: main_text_concentration_g}]\label{thm: concentration_g}
Consider the same setting as in \Cref{thm: main_non_asymp}. 
There exist constants $c$ such that if $ d_h \geq \Omega(\frac{1}{\alpha} \frac{L^2 }{\eps^2} \log \frac{18 d_{h+1} L}{\delta}),\ \forall h \in \{1,2, \ldots, L\}$ and $\eps \leq \min(c, \frac{1}{L})$ then for any fixed $\mathbf{x,x'} \in \R^{d_0}, \ \norm{\mathbf{x}}_2, \norm{\mathbf{x'}}_2 \leq 1$, we have with probability $1 - \delta$, $\forall 0 \leq h \leq L,\ \forall i \in [d_{h+1}], \  \forall (\mathbf{x}^{(1)}, \mathbf{x}^{(2)}) \in \{(\mathbf{x,x}), (\mathbf{x,x'}), (\mathbf{x',x'})\}$,
\begin{align*}
    & \left| \left( \mathbf{g}^{(h)}(\mathbf{x}^{(1)}) \odot \mathbf{m}^{(h+1)}_i \right)^\top \left( \mathbf{g}^{(h)}(\mathbf{x}^{(2)}) \odot \mathbf{m}^{(h+1)}_i \right) - {\Sigma}^{(h)}(\mathbf{x}^{(1)}, \mathbf{x}^{(2)}) \right| \leq \eps \\
    & \left| \left( \mathbf{g}^{(h)}(\mathbf{x}^{(1)}) \right)^\top  \mathbf{g}^{(h)}(\mathbf{x}^{(2)}) - {\Sigma}^{(h)}(\mathbf{x}^{(1)}, \mathbf{x}^{(2)}) \right| \leq \eps.
\end{align*}
In other words, if $d_h \geq \Omega( \frac{1}{\alpha} \frac{L^2}{\eps_1^2}  \log \frac{18 d_{h+1} L}{\delta_1}),\ \forall h \in \{1,2, \ldots, L\}$ and $\eps_1 \leq \min(c_2, \frac{1}{L})$ then 
\begin{align*}
    \Pr\left[ \overline{\mathcal{A}}(\eps_1) \right] \geq 1 - \delta_1
\end{align*}
\end{theorem}
The proof of \Cref{thm: concentration_g} can be found in \Cref{sec: proof_forward_dynamics}. 

\begin{lemma}\label{lemma: D}
If $d_h \geq \Omega(\frac{1}{\alpha}\frac{1}{\eps_4^2} \log \frac{12L d_{h+1}}{\delta_4})$ for all $h \in [L]$, then
\begin{align*}
    \Pr\left[ \overline{\mathcal{A}}(\eps_1^2/2) \Rightarrow \overline{\mathcal{D}} \left( \eps_1 + \eps_4 \right) \right] \geq 1 - \delta_4.
\end{align*}
\end{lemma}
The proof of \Cref{lemma: D} on a single pair can be found in \Cref{sec: proof_D} and then take a union bound over pairs $(\mathbf{x,x}), (\mathbf{x,x'}) , (\mathbf{x',x'})$. 

\begin{lemma}\label{lemma: event_C}
If $d_h \geq \Omega(\frac{1}{\alpha} \frac{L^2}{\eps^2} \log \frac{2L d_{h+1} \sum_{h'=1}^{h-1} d_h' }{\delta_3}) = \tilde{\Omega}(\frac{1}{\alpha} \frac{L^2}{\eps^2} )$ for all $h \in [L]$, then
\begin{align*}
    \Pr\left[ \overline{\mathcal{A}}( \eps_1) \Rightarrow \overline{\mathcal{C}}\left( 2 \sqrt{\log \frac{4 \sum_{h'=1}^{L-1} d_{h'}}{\delta_3}} \right) \right] \geq 1 - \delta_3
\end{align*}
\end{lemma}
The proof of \Cref{lemma: event_C} can be found in \Cref{sec: proof_event_C}. 

\begin{lemma}\label{lemma: event_B}
Let $\eps_3 = 2 \sqrt{\log \frac{4 \sum_{h'=1}^{L-1} d_{h'}}{\delta_3}}$.
If $d_h \geq \frac{8}{\alpha} \log \frac{6}{\delta_2}$,
with probability $1 - \delta_2$, the event $\overline{\mathcal{C}}(\eps_3)$ holds and, there exists constant $C,C'$ such that for any $\eps_2, \eps_4 \in [0,1]$, we have
\begin{align*}
    &\Pr\left[ \overline{\mathcal{A}}^L(\eps_1^2/2) \bigcap \overline{\mathcal{B}}^{h+1}(\eps_2) \bigcap \overline{\mathcal{C}}(\eps_3) \bigcap \overline{\mathcal{D}}^h(\eps_4) \Rightarrow \overline{\mathcal{B}}^h\left(\eps_2 + 2\eps_4 + \frac{48\sqrt{2}}{\sqrt{d_h}} +  48 \sqrt{\frac{2}{\alpha} \frac{\log \frac{8}{\delta_2}}{d_h}} + \frac{96}{{\alpha}} \frac{\sqrt{2\log \frac{4 \sum_{h'=1}^{L-1} d_{h'}}{\delta_3}}}{\sqrt{d_h}} \right) \right] \\
    &\geq 1 - \delta_2/2.
\end{align*}
\end{lemma}
The proof of \Cref{lemma: event_B} can be found in \Cref{sec: proof_event_B}. 

\begin{proof}[Proof of Theorem \ref{thm: main_theorem_intermediate}]
We prove by induction on \Cref{lemma: event_B}. 
We first let the event in \Cref{lemma: num_remained_weights} holds with $\eps$ and probability $1 - \delta/5$. 
In the statement of \Cref{thm: concentration_g}, we set $\delta_1 = \delta/5,\ \eps_1 = \frac{\eps^2}{8}$, if $d_h \geq \Omega( \frac{1}{\alpha} \frac{L^2}{\eps^4} \log \frac{d_{h+1} L}{\delta}) = \tilde{\Omega}(\frac{1}{\alpha} \frac{L^2}{\eps^4}),\ \forall h \in \{1,2, \ldots, L\}$, we have
\begin{align*}
    \Pr[\overline{\mathcal{A}}(\eps^2/8)] \geq 1 - \delta / 5.
\end{align*}
In the statement of \Cref{lemma: D}, we set $\delta_4 = \delta/5$ and $\eps_1 = \eps/2,\ \eps_4 = \eps/4$. 
If $d_h \geq \Omega(\frac{1}{\alpha}\frac{1}{\eps^2} \log \frac{L d_{h+1}}{\delta}) = \tilde{\Omega}(\frac{1}{\alpha}\frac{1}{\eps^2})$ for all $h \in [L]$, then
\begin{align*}
    \Pr\left[ \overline{\mathcal{A}}(\eps^2/8) \Rightarrow \overline{\mathcal{D}} \left( \eps \right) \right] \geq 1 - \delta/5.
\end{align*}
In the statement of Lemma \ref{lemma: event_C}, setting $\delta_3 = \delta / 5$, if $d_h \geq \Omega(\frac{1}{\alpha} \frac{L^2}{\eps^2} \log \frac{2L d_{h+1} \sum_{h'=1}^{h-1} d_h' }{\delta_3}) = \tilde{\Omega}(\frac{1}{\alpha} \frac{L^2}{\eps^2})$ for all $h \in [L]$, then
\begin{align*}
    \Pr\left[ \overline{\mathcal{A}}( \eps^2/8) \Rightarrow \overline{\mathcal{C}}\left( 2 \sqrt{\log \frac{20 \sum_{h'=1}^{L-1} d_{h'}}{\delta}} \right) \right] \geq 1 - \delta/5.
\end{align*}
Take a union bound we have
\begin{align*}
    \Pr\left[ \overline{\mathcal{A}}(\eps^2/2) \bigcap \overline{\mathcal{C}}\left( 2 \sqrt{\log \frac{20 \sum_{h'=1}^{L-1} d_{h'}}{\delta}} \right) \bigcap \overline{\mathcal{D}}(\eps) \right] \geq 1 - \frac{3\delta}{5}.
\end{align*}
Now we begin the induction. 
First of all, $\Pr\left[ \overline{\mathcal{B}}^{L+1}(0) \right] = 1$ by definition.
For $1 \leq h \leq L$, in the statement of Lemma \ref{lemma: event_B}, set $\eps_2 = 3(L+1 - h), \eps_3 = 2 \sqrt{\log \frac{20 \sum_{h'=1}^{L-1} d_{h'}}{\delta}}, \delta_2 = \frac{\delta}{4L}$.
If $d_h \geq \Omega(\frac{1}{\alpha^2} \frac{1}{\eps^2} \log \frac{L \sum_{h'=1}^{L-1} d_{h'}}{\delta}) = \tilde{\Omega}(\frac{1}{\alpha^2} \frac{1}{\eps^2})$, we have $\frac{48\sqrt{2}}{\sqrt{d_h}} +  48 \sqrt{\frac{2}{\alpha} \frac{\log \frac{8}{\delta_2}}{d_h}} + \frac{96}{{\alpha}} \frac{\sqrt{2\log \frac{4 \sum_{h'=1}^{L-1} d_{h'}}{\delta_3}}}{\sqrt{d_h}} < \eps/2$. Thus we have
\begin{align*}
    & \Pr\left[\overline{\mathcal{B}}^{(h+1)}((3L - 3h)\eps) \bigcap \overline{\mathcal{C}}\left( \eps_3 \right) \bigcap \overline{\mathcal{D}}(\eps) \Rightarrow \overline{\mathcal{B}}^h\left( (3L + 2 - 3h)\eps + \frac{48\sqrt{2}}{\sqrt{d_h}} +  48 \sqrt{\frac{2}{\alpha} \frac{\log \frac{8}{\delta_2}}{d_h}} + \frac{96}{{\alpha}} \frac{\sqrt{2\log \frac{4 \sum_{h'=1}^{L-1} d_{h'}}{\delta_3}}}{\sqrt{d_h}} \right) \right] \\
    & \geq \Pr\left[\overline{\mathcal{B}}^{(h+1)}((3L - 3h)\eps) \bigcap \overline{\mathcal{C}}\left( \eps_3 \right) \bigcap \overline{\mathcal{D}}(\eps) \Rightarrow \overline{\mathcal{B}}^h((3L + 3 - 3h)\eps) \right] \\
    &\geq 1 - \frac{\delta}{5L}
\end{align*}
Applying union bound for every $h \in [L]$, we have
\begin{align*}
    & \Pr\left[ \overline{\mathcal{A}}^L(\eps^2 / 8) \bigcap \overline{\mathcal{B}}(3L\eps) \bigcap \overline{\mathcal{C}}(\eps_3) \bigcap \overline{\mathcal{D}}(\eps) \right] \\
    & \geq \Pr\left[ \overline{\mathcal{A}}^L(\eps^2 / 8) \bigcap_{h=1}^L \overline{\mathcal{B}}^h(3(L + 1 - h)\eps) \bigcap \overline{\mathcal{C}}(\eps_3) \bigcap \overline{\mathcal{D}}(\eps) \right] \\
    &\geq 1 - \Pr\left[\neg \left( \overline{\mathcal{A}}(\eps^2/8) \bigcap \overline{\mathcal{C}}(\eps_3) \bigcap \overline{\mathcal{D}}^h(\eps) \right)  \right] \\
    & \quad - \sum_{h=1}^L \Pr\left[ \neg \left( \overline{\mathcal{B}}^{(h+1)}((3L - 3h)\eps) \bigcap \overline{\mathcal{C}}\left( \eps_3 \right) \bigcap \overline{\mathcal{D}}(\eps) \Rightarrow \overline{\mathcal{B}}^h((3L + 3 - 3h)\eps) \right) \right] \\
    & \geq 1 - \delta
\end{align*}
\end{proof}

\subsection{Proof of \texorpdfstring{\Cref{thm: concentration_g}}{Concentration of Activation}: Forward Propagation}\label{sec: proof_forward_dynamics}
In this section, we prove $\overline{\mathcal{A}}(\mathbf{x,x'}, \eps)$ holds which is shown in \Cref{thm: concentration_g} below.
The main goal is to obtain bounds on $\left| \left(\mathbf{m}^{(h+1)} \odot \mathbf{{g}}^{(h)}(\mathbf{x}) \right)^\top (\mathbf{m}^{(h+1)} \odot \mathbf{{g}}^{(h)}(\mathbf{x}')) - \Sigma^{(h)}(\mathbf{x,x'}) \right|$.
We first introduce a result from previous work. 
\begin{lemma}[Lemma 13 in \citep{daniely2016toward}]\label{lemma: daniely}
Define the function
\begin{align*}
    \overline{\sigma}(\boldsymbol{\Sigma}) = c_\sigma \E_{(X,Y) \sim \mathcal{N}(\mathbf{0}, \boldsymbol{\Sigma})} \sigma(X) \sigma(Y),
\end{align*}
and the set
\begin{align*}
    \mathcal{M}_+^\gamma := \left\{ 
    \begin{bmatrix}
    \Sigma_{11} & \Sigma_{12} \\
    \Sigma_{12} & \Sigma_{22}
    \end{bmatrix}
    \in \mathcal{M}_+ | 1 - \gamma \leq \Sigma_{11}, \Sigma_{22} \leq 1 + \gamma
    \right\},
\end{align*}
where $\mathcal{M}_+$ denote the set of positive semi-definite matrices. 
Then $\overline{\sigma}$ is $(1 + o(\eps))$-Lipschitz on $\mathcal{M}_+^\eps$ with respect to $\infty$-norm.
\end{lemma}

Our analysis follows from the proof of Theorem 14 in \citep{daniely2016toward}. 
\begin{proof}[Proof of \Cref{thm: concentration_g}]
We prove the first inequality first. 
Define the quantity $B_d = \sum_{i = 1}^d (1 + o(\eps))^i$.

We begin our proof by saying the $h$-th layer of a neural network is well-initialized if $\forall i \in [d_{h+1}]$, we have
\begin{align*}
    \left| \left( \mathbf{g}^{(h)}(\mathbf{x}^{(1)}) \odot \mathbf{m}^{(h+1)}_i \right)^\top \left( \mathbf{g}^{(h)}(\mathbf{x}^{(2)}) \odot \mathbf{m}^{(h+1)}_i \right) - {\Sigma}^{(h)}(\mathbf{x}^{(1)}, \mathbf{x}^{(2)}) \right| \leq \eps \frac{B_h}{B_L} .
\end{align*}

We prove the result by induction. 
Since we don't prune the input layer, the result trivially holds for $h = 0$.
Assume all the layers first $h-1$ layers are well-initialized. 

Now, conditioned on $\mathbf{g}^{(h-1)}(\mathbf{x}^{(1)}), \mathbf{g}^{(h-1)}(\mathbf{x}^{(2)}), \mathbf{m}^{(h)}$, we have
\begin{align*}
    & \E_{\mathbf{W}^{(h)}, \mathbf{m}^{(h+1)}} \left[ \left( \mathbf{g}^{(h)}(\mathbf{x}^{(1)}) \odot \mathbf{m}^{(h+1)}_i \right)^\top \left( \mathbf{g}^{(h)}(\mathbf{x}^{(2)}) \odot \mathbf{m}^{(h+1)}_i \right) \right] \\
    &= \E_{\mathbf{W}^{(h)}} \left[ \left( \mathbf{g}^{(h)}(\mathbf{x}^{(1)}) \right)^\top \mathbf{g}^{(h)}(\mathbf{x}^{(2)}) \right] \\
    &= \frac{c_\sigma}{d_h} \sum_{i=1}^{d_h} \E_{\mathbf{W}^{(h)}} \left[ \sigma \left( \inprod{\mathbf{W}^{(h)}_i, \mathbf{m}_i^{(h)} \odot \mathbf{g}^{(h-1)}(\mathbf{x}^{(1)})} \right) \sigma \left( \inprod{\mathbf{W}^{(h)}_i, \mathbf{m}_i^{(h)} \odot \mathbf{g}^{(h-1)}(\mathbf{x}^{(2)})} \right) \right].
\end{align*}
where $\mathbf{W}_i^{(h)}$ denotes the $i$-th row of $\mathbf{W}^{(h)}$.
Define 
\begin{align*}
    \widehat{\Sigma}^{(h)}_i(\mathbf{x}^{(1)},\mathbf{x}^{(2)}) &= \left( \mathbf{g}^{(h)}(\mathbf{x}^{(1)}) \odot \mathbf{m}^{(h+1)}_i \right)^\top \left( \mathbf{g}^{(h)}(\mathbf{x}^{(2)}) \odot \mathbf{m}^{(h+1)}_i \right), \\
    \widehat{\boldsymbol{\Lambda}}^{(h)}_i (\mathbf{x}^{(1)}, \mathbf{x}^{(2)}) &= 
    \begin{bmatrix}
    \widehat{\Sigma}^{(h)}_i(\mathbf{x}^{(1)},\mathbf{x}^{(1)}) & \widehat{\Sigma}^{(h)}_i(\mathbf{x}^{(1)},\mathbf{x}^{(2)}) \\
    \widehat{\Sigma}^{(h)}_i(\mathbf{x}^{(2)},\mathbf{x}^{(1)}) & \widehat{\Sigma}^{(h)}_i(\mathbf{x}^{(2)},\mathbf{x}^{(2)})
    \end{bmatrix}.
\end{align*}
Notice that for a given $j$, conditioned on $\mathbf{m}^{(h)}$, $\mathbf{g}^{(h-1)}(\mathbf{x}^{(1)})$ and $\mathbf{g}^{(h-1)}(\mathbf{x}^{(2)})$, and consider the randomness in $\mathbf{W}_j$, $\sigma \left( \inprod{\mathbf{W}^{(h)}_j, \mathbf{m}_j^{(h)} \odot \mathbf{g}^{(h-1)}(\mathbf{x}^{(1)})} \right) \sigma \left( \inprod{\mathbf{W}^{(h)}_j, \mathbf{m}_j^{(h)} \odot \mathbf{g}^{(h-1)}(\mathbf{x}^{(2)})} \right)$ is subgamma with parameters $(O(1), O(1))$ and
\begin{align*}
    & \E\left[ \sigma \left( \inprod{\mathbf{W}^{(h)}_j, \mathbf{m}_j^{(h)} \odot \mathbf{g}^{(h-1)}(\mathbf{x}^{(1)})} \right) \sigma \left( \inprod{\mathbf{W}^{(h)}_j, \mathbf{m}_j^{(h)} \odot \mathbf{g}^{(h-1)}(\mathbf{x}^{(2)})} \right) \right] \\
    &\leq \sqrt{\E\left[ \left(\sigma \left( \inprod{\mathbf{W}^{(h)}_j, \mathbf{m}_j^{(h)} \odot \mathbf{g}^{(h-1)}(\mathbf{x}^{(1)})} \right) \right)^2 \right] \E\left[ \left( \sigma \left( \inprod{\mathbf{W}^{(h)}_j, \mathbf{m}_j^{(h)} \odot \mathbf{g}^{(h-1)}(\mathbf{x}^{(2)})} \right) \right)^2 \right]} \\
    &\leq \sqrt{\E\left[ \left( \inprod{\mathbf{W}^{(h)}_j, \mathbf{m}_j^{(h)} \odot \mathbf{g}^{(h-1)}(\mathbf{x}^{(1)})} \right)^2 \right] \E\left[ \left( \inprod{\mathbf{W}^{(h)}_j, \mathbf{m}_j^{(h)} \odot \mathbf{g}^{(h-1)}(\mathbf{x}^{(2)})} \right)^2 \right]} \\
    &= \norm{\mathbf{m}_j^{(h)} \odot \mathbf{g}^{(h-1)}(\mathbf{x}^{(1)})}_2 \norm{\mathbf{m}_j^{(h)} \odot \mathbf{g}^{(h-1)}(\mathbf{x}^{(2)})}_2
    \leq 4,
\end{align*}
where the first inequality is by Cauchy-Schwarz inequality. 

By \Cref{thm: concentration_subgamma_bernoulli}, we have
\begin{align*}
    \Pr \left[ \left| \widehat{\Sigma}_i^{(h)}(\mathbf{x}^{(1)}, \mathbf{x}^{(2)}) - \E_{\mathbf{W}^{(h)}, \mathbf{m}^{(h+1)}} \widehat{\Sigma}_i^{(h)}(\mathbf{x}^{(1)}, \mathbf{x}^{(2)}) \right| > \eps \right] \leq 4 \exp \left\{ - \Omega(\alpha d_h \eps^2) \right\},
\end{align*}
for some constant $c_2$ such that $\eps < c_2$.

Taking a union bound over $i \in [d_{h+1}]$, we have if $d_h \geq \Omega(\frac{1}{\alpha} \frac{B_L^2 \log \frac{8 d_{h+1} L}{\delta}}{\eps^2})$, then with probability $1 - \frac{\delta}{L}$ for all $i \in [d_{h+1}]$,
\begin{align*}
    & \left| \left( \mathbf{g}^{(h)}(\mathbf{x}^{(1)}) \odot \mathbf{m}^{(h+1)}_i \right)^\top \left( \mathbf{g}^{(h)}(\mathbf{x}^{(2)}) \odot \mathbf{m}^{(h+1)}_i \right) - \frac{c_\sigma}{d_h} \sum_{j=1}^{d_h} \E_{(u,v) \sim \mathcal{N}(\mathbf{0}, \widehat{\boldsymbol{\Lambda}}^{(h-1)}_j(\mathbf{x}^{(1)}, \mathbf{x}^{(2)}))} \left[ \sigma(u) \sigma(v) \right] \right| \leq \eps/B_L.
\end{align*}
Now apply triangle inequality
\begin{align*}
    & \left| \left( \mathbf{g}^{(h)}(\mathbf{x}^{(1)}) \odot \mathbf{m}^{(h+1)}_i \right)^\top \left( \mathbf{g}^{(h)}(\mathbf{x}^{(2)}) \odot \mathbf{m}^{(h+1)}_i \right) - {\Sigma}^{(h)}(\mathbf{x}^{(1)}, \mathbf{x}^{(2)}) \right| \\
    & \leq \left| \left( \mathbf{g}^{(h)}(\mathbf{x}^{(1)}) \odot \mathbf{m}^{(h+1)}_i \right)^\top \left( \mathbf{g}^{(h)}(\mathbf{x}^{(2)}) \odot \mathbf{m}^{(h+1)}_i \right) - \frac{c_\sigma}{d_h} \sum_{j=1}^{d_h} \E_{(u,v) \sim \mathcal{N}(\mathbf{0}, \widehat{\boldsymbol{\Lambda}}^{(h-1)}_j(\mathbf{x}^{(1)}, \mathbf{x}^{(2)}))} \left[ \sigma(u) \sigma(v) \right] \right| \\
    & \quad + \left| \frac{c_\sigma}{d_h} \sum_{j=1}^{d_h} \E_{(u,v) \sim \mathcal{N}(\mathbf{0}, \widehat{\boldsymbol{\Lambda}}^{(h-1)}_j(\mathbf{x}^{(1)}, \mathbf{x}^{(2)}))} \left[ \sigma(u) \sigma(v) \right] - \Sigma^{(h)}(\mathbf{x}^{(1)}, \mathbf{x}^{(2)}) \right| \\
    &\leq \eps / B_L + \frac{1}{d_h} \sum_{i=1}^{d_h} \left| c_\sigma  \E_{(u,v) \sim \mathcal{N}(\mathbf{0}, \widehat{\boldsymbol{\Lambda}}^{(h-1)}_j(\mathbf{x}^{(1)}, \mathbf{x}^{(2)}))} \left[ \sigma(u) \sigma(v) \right] - \Sigma^{(h)}(\mathbf{x}^{(1)}, \mathbf{x}^{(2)}) \right| \\
    &\leq \eps / B_L + \frac{1}{d_h} \sum_{i=1}^{d_h} (1 + o(\eps)) \eps \frac{B_{h-1}}{B_L} = \eps \frac{B_h}{B_L},
\end{align*}
where the last inequality applies by the fact that $\overline{\sigma}$ is $(1+o(\eps))$-Lipschitz on $\mathcal{M}_+^\gamma$ with respect to the $\infty$-norm in \Cref{lemma: daniely} and the induction hypothesis that the first $h-1$ layers are well-initialized. 

Finally we expand $B_d = \sum_{i=1}^d (1 + o(\eps))^i$ and take $\eps = \min(c_2, \frac{1}{L})$, we have
\begin{align*}
    B_L = \sum_{i=1}^L (1 + o(\eps))^i \leq \sum_{i=1}^L e^{o(\eps) L} = O(L).
\end{align*}

The proof for
\begin{align*}
    \left| \left( \mathbf{g}^{(h)}(\mathbf{x}^{(1)}) \right)^\top  \mathbf{g}^{(h)}(\mathbf{x}^{(2)}) - {\Sigma}^{(h)}(\mathbf{x}^{(1)}, \mathbf{x}^{(2)}) \right| \leq \eps,
\end{align*}
largely follows the same steps as above since
\begin{align*}
    \E_{\mathbf{W}^{(h)}} \left(\mathbf{g}^{(h)}(\mathbf{x}^{(1)}) \right)^\top  \mathbf{g}^{(h)}(\mathbf{x}^{(2)}) = \E_{\mathbf{W}^{(h)}, \mathbf{m}^{(h+1)}} \widehat{\Sigma}_i^{(h)}(\mathbf{x}^{(1)}, \mathbf{x}^{(2)}).
\end{align*}
Now applying the concentration of sub-Gamma random variables we have
\begin{align*}
    \Pr\left[ \left| \left( \mathbf{g}^{(h)}(\mathbf{x}^{(1)}) \right)^\top  \mathbf{g}^{(h)}(\mathbf{x}^{(2)}) - \E_{\mathbf{W}^{(h)}} \left(\mathbf{g}^{(h)}(\mathbf{x}^{(1)}) \right)^\top  \mathbf{g}^{(h)}(\mathbf{x}^{(2)}) \right| \geq \eps \right] \leq 2\exp\{-\eps^2 d_h\}
\end{align*}
for sufficiently small $\eps$, which requires $d_h \geq \Omega(\frac{1}{\eps^2} \log \frac{6L}{\delta})$ by taking a union bound over $L$. 
\end{proof}

\begin{lemma}\label{lemma: gaussian_concentration}
Assume the event $\overline{\mathcal{A}}(\mathbf{x,x'}, \eps)$ holds for $\eps < 1$. 
Then, with probability at least $1 - \delta$ over the randomness of $\mathbf{w}^{(L+1)}$
\begin{align*}
    |f^{(L+1)}(\mathbf{x})| \leq \sqrt{2\log \frac{2}{\delta}}.
\end{align*}
\end{lemma}
\begin{proof}
By definition, we have $f^{(L+1)}(\mathbf{x}) = \inprod{\mathbf{w}^{(L+1)} \odot \mathbf{m}^{(L+1)}, \mathbf{g}^{(h)}(\mathbf{x})} $. 
By our assumption, $\norm{\mathbf{g}^{(h)}(\mathbf{x}) \odot \mathbf{m}^{(h+1)}}_2^2 \leq 2$. 
Thus, by apply standard Gaussian tail bound, with probability at least $1 - \delta$,
\begin{align*}
    |f^{(L+1)}(\mathbf{x})| \leq \sqrt{2 \log \frac{2}{\delta}}.
\end{align*}
\end{proof}



\subsection{Proof of \texorpdfstring{\Cref{lemma: D}}{Single Layer Activation Derivatives}: Analyzing the Activation Gradient of a Single Layer}\label{sec: proof_D}
To prove \Cref{lemma: D}, we first introduce a previous result. 
\begin{lemma}[Lemma E.8. \citep{arora2019exact}]\label{lemma: arora_e8}
Define 
\begin{align*}
    t_{\dot{\sigma}}(\boldsymbol{\Sigma}) = c_\sigma \E_{(u,v) \sim \mathcal{N}(\mathbf{0}, \boldsymbol{\Sigma}')} [\dot{\sigma}(u) \dot{\sigma}(v)] \quad \textnormal{with} \quad \boldsymbol{\Sigma'} = 
    \begin{bmatrix}
    1 & \frac{\Sigma_{12}}{\sqrt{\Sigma_{11} \Sigma_{22}}} \\
    \frac{\Sigma_{12}}{\sqrt{\Sigma_{11} \Sigma_{22}}} & 1 
    \end{bmatrix}.
\end{align*}
Then
\begin{align*}
    \norm{\mathbf{G}^{(h)}(\mathbf{x,x'}) - \boldsymbol{\Lambda}^{(h)}(\mathbf{x,x'})}_\infty \leq \frac{\eps^2}{2} \Rightarrow \left| t_{\dot{\sigma}} \left( \mathbf{G}^{(h)}(\mathbf{x,x'}) \right) - t_{\dot{\sigma}} \left( \boldsymbol{\Lambda}^{(h)} (\mathbf{x,x'}) \right) \right| \leq \eps.
\end{align*}
\end{lemma}
\begin{proof}[Proof of \Cref{lemma: D}]
Conditioned on $\widehat{\boldsymbol{\Lambda}}^{(h)}_i,\ \forall i \in [d_h]$ and consider the randomness of $\mathbf{W}^{(h)}, \mathbf{m}^{(h+1)}$, we have 
\begin{align*}
    & \E_{\mathbf{W}^{(h)}, \mathbf{m}^{(h+1)}} \left[ 2 \frac{\tr(\mathbf{M}^{(h+1)}_i \mathbf{D}^{(h)}(\mathbf{x,x'}) \mathbf{M}^{(h+1)}_i) }{d_h} \right] \\
    &= \E_{\mathbf{W}^{(h)}} \left[ 2 \frac{\tr( \mathbf{D}^{(h)}(\mathbf{x,x'}) ) }{d_h} \right] \\
    &=  \frac{1}{d_h} \sum_{i=1}^{d_h} \E_{\mathbf{W}^{(h)}} \left[ \dot{\sigma}\left(\inprod{\mathbf{W}^{(h)}_i, \mathbf{m}^{(h)}_i \odot \mathbf{g}^{(h-1)}(\mathbf{x})}\right) \dot{\sigma}\left(\inprod{\mathbf{W}^{(h)}_i, \mathbf{m}^{(h)}_i \odot \mathbf{g}^{(h-1)}(\mathbf{x}')} \right) \right] \\
    &= \frac{1}{d_h} \sum_{i=1}^{d_h} t_{\dot{\sigma}} \left( \widehat{\boldsymbol{\Lambda}}_i^{(h)} \right).
\end{align*}
Now, by triangle inequality and our assumption on $\widehat{\boldsymbol{\Lambda}}_i,\ \forall i \in [d_h]$, apply Lemma \ref{lemma: arora_e8}
\begin{align*}
    \left| t_{\dot{\sigma}}\left( \boldsymbol{\Lambda}^{(h)}(\mathbf{x,x'}) \right) - \frac{1}{d_h} \sum_{i=1}^{d_h} t_{\dot{\sigma}} \left( \widehat{\boldsymbol{\Lambda}}_i^{(h)} \right) \right| &\leq \frac{1}{d_h} \sum_{i=1}^{d_h} \left| t_{\dot{\sigma}} \left(\boldsymbol{\Lambda}^{(h)}(\mathbf{x,x'}) \right) - t_{\dot{\sigma}} \left( \widehat{\boldsymbol{\Lambda}}_i^{(h)} \right) \right| \leq \eps_1.
\end{align*}
Finally, since $\dot{\sigma}(\mathbf{f}_j^{(h)}(\mathbf{x})) \dot{\sigma}(\mathbf{f}_j^{(h)}(\mathbf{x}'))$ is a 0-1 random variable, it is sub-Gaussian with variance proxy $\frac{1}{4}$.
By \Cref{thm: concentration_subgaussian_bernoulli}, for a given $i$ and $t>0$,
\begin{align*}
    \Pr\left[ \left| 2 \frac{\tr(\mathbf{M}_i^{(h+1)} \mathbf{D}^{(h)}(\mathbf{x,x'}) \mathbf{M}_i^{(h+1)})}{d_h} - \frac{1}{d_h} \sum_{i=1}^{d_h} t_{\dot{\sigma}} \left( \widehat{\boldsymbol{\Lambda}}_i^{(h)} \right) \right| > t \right] \leq 4 \exp\left\{ - \Omega(\alpha d_h t^2) \right\}.
\end{align*}
Finally, by taking a union bound over $h \in [L]$, $i \in [d_{h+1}]$, if $d_h \geq \Omega(\frac{1}{\alpha}\frac{1}{\eps_4^2} \log \frac{4L d_{h+1}}{\delta})$ with probability $1 - \delta$ over the randomness of $\mathbf{W}^{(h)}, \mathbf{m}^{(h+1)}$, we have $\forall h \in [L], i \in [d_{h+1}]$,
\begin{align*}
    \left| 2 \frac{\tr(\mathbf{M}_i^{(h+1)} \mathbf{D}^{(h)}(\mathbf{x,x'}) \mathbf{M}_i^{(h+1)})}{d_h} - \frac{1}{d_h} \sum_{i=1}^{d_h} t_{\dot{\sigma}} \left( \widehat{\boldsymbol{\Lambda}}_i^{(h)} \right) \right| < \eps_4.
\end{align*}
By triangle inequality we have
\begin{align*}
    \left| 2 \frac{\tr(\mathbf{M}_i^{(h+1)} \mathbf{D}^{(h)}(\mathbf{x,x'}) \mathbf{M}_i^{(h+1)})}{d_h} - t_{\dot{\sigma}}\left( \boldsymbol{\Lambda}^{(h)}(\mathbf{x,x'}) \right) \right| < \eps_1 + \eps_4.
\end{align*}
\end{proof}

\subsection{Proof of \texorpdfstring{\Cref{lemma: event_B}}{Backward Gradient}: The Fresh Gaussian Copy Trick}\label{sec: proof_event_B}
\begin{proof}[Proof of \Cref{lemma: event_B}]
The goal is to show that
\begin{align*}
    \left| \left(\mathbf{b}^{(h+1)} (\mathbf{x}^{(1)}) \right)^\top (\mathbf{W}^{(h+1)} \odot \mathbf{m}^{(h+1)}) \mathbf{D}^{(h)}(\mathbf{x}^{(1)}) \mathbf{D}^{(h)}(\mathbf{x}^{(2)}) (\mathbf{W}^{(h+1)} \odot \mathbf{m}^{(h+1)})^\top \mathbf{b}^{(h+1)} (\mathbf{x}^{(2)}) - \prod_{h' = h}^L \dot{\Sigma}^{(h')}(\mathbf{x}^{(1)}, \mathbf{x}^{(2)}) \right|
\end{align*}
is small. We can write
\begin{align*}
    & \left(\mathbf{b}^{(h+1)} (\mathbf{x}^{(1)}) \right)^\top (\mathbf{W}^{(h+1)} \odot \mathbf{m}^{(h+1)}) \mathbf{D}^{(h)}(\mathbf{x}^{(1)}) \mathbf{D}^{(h)}(\mathbf{x}^{(2)}) (\mathbf{W}^{(h+1)} \odot \mathbf{m}^{(h+1)})^\top \mathbf{b}^{(h+1)} (\mathbf{x}^{(2)}) \\
    &= \frac{c_\sigma}{d_h} \sum_{i,j} \mathbf{b}_i^{(h+1)} (\mathbf{x}^{(1)}) \mathbf{b}_j^{(h+1)} (\mathbf{x}^{(2)}) \left( {\mathbf{w}}^{(h+1)}_i \right) \mathbf{M}_i^{(h+1)} \mathbf{D}^{(h)}(\mathbf{x}^{(1)}) \mathbf{D}^{(h)}(\mathbf{x}^{(2)})  \mathbf{M}_j^{(h+1)} {\mathbf{w}}^{(h+1)}_j
\end{align*}
We first show that this term is close to 
\begin{align}\label{eq: backward_intermediate_goal}
    \frac{2}{d_h} \sum_i \mathbf{b}_i^{(h+1)} (\mathbf{x}^{(1)}) \mathbf{b}_i^{(h+1)} (\mathbf{x}^{(2)}) \tr( \mathbf{M}_i^{(h+1)} \mathbf{D} \mathbf{M}_i^{(h+1)})
\end{align}
We do this by the following. 

Let $\mathbf{G}^{(h)}_i = [(\mathbf{g}^{(h)}(\mathbf{x}) \odot \mathbf{m}^{(h+1)}_i) \quad (\mathbf{g}^{(h)}(\mathbf{x}') \odot \mathbf{m}^{(h+1)}_i)]$ and $\mathbf{G}^{(h)} = [\mathbf{G}^{(h)}_1 \mathbf{G}^{(h)}_2 \ldots \mathbf{G}^{(h)}_{d_{h+1}}]$ and $\mathbf{F}^{(h+1)} = (\mathbf{W}^{(h+1)} \odot \mathbf{m}^{(h+1)}) \mathbf{G}^{(h)}$.
We further simplify our notation to let $\mathbf{G}_i = \mathbf{G}_i^{(h)}$ since there is no ambiguity on layers. 
Notice that conditioned on $\mathbf{F}^{(h+1)}, \mathbf{m}^{(h+1)}, \mathbf{G}^{(h)}$, notice that
\[
\left( \mathbf{b}^{(h+1)}(\mathbf{x}) \right)^\top
\begin{bmatrix}
((\mathbf{w}_1^{(h+1)} )^\top \Pi_{\mathbf{G}_1}^\bot) \odot \mathbf{m}_1^{(h+1)} \\
((\mathbf{w}_2^{(h+1)})^\top \Pi_{\mathbf{G}_2}^\bot ) \odot \mathbf{m}_2^{(h+1)} \\
\vdots \\
((\mathbf{w}_{d_{h+1}}^{(h+1)} )^\top \Pi_{\mathbf{G}_{d_{h+1}}}^\bot ) \odot \mathbf{m}^{(h+1)}_{d_{h+1}} \\
\end{bmatrix}
\in \R^{d_h}
\]
has multivariate Gaussian distribution and by Lemma \ref{lemma: fresh_gaussian} it has the same distribution as
\[
\left( \mathbf{b}^{(h+1)}(\mathbf{x}) \right)^\top
\begin{bmatrix}
((\mathbf{w}_1^{(h+1)} )^\top \Pi_{\mathbf{G}_1}^\bot) \odot \mathbf{m}_1^{(h+1)} \\
((\mathbf{w}_2^{(h+1)})^\top \Pi_{\mathbf{G}_2}^\bot ) \odot \mathbf{m}_2^{(h+1)} \\
\vdots \\
((\mathbf{w}_{d_{h+1}}^{(h+1)} )^\top \Pi_{\mathbf{G}_{d_{h+1}}}^\bot ) \odot \mathbf{m}^{(h+1)}_{d_{h+1}} \\
\end{bmatrix}
= \sum_{i=1}^{d_{h+1}} \mathbf{b}_i^{(h+1)}(\mathbf{x}) ((\tilde{\mathbf{w}}^{(h+1)}_i )^\top \Pi_{\mathbf{G}_i}^\bot ) \odot \mathbf{m}^{(h+1)}_i,
\]
where $\tilde{\mathbf{w}}_i^{(h+1)}$ is a fresh copy of i.i.d. Gaussian.
First of all, let $\mathbf{M}_i^{(h+1)} = \textnormal{diag}(\mathbf{m}_i^{(h+1)})$, and we have 
\begin{align}\label{eq: dependent_independent_decomposition}
    & \frac{c_\sigma}{d_h} \sum_{i,j} \mathbf{b}_i^{(h+1)} (\mathbf{x}^{(1)}) \mathbf{b}_j^{(h+1)} (\mathbf{x}^{(2)}) \left( {\mathbf{w}}^{(h+1)}_i \right) (\Pi_{\mathbf{G}_i} + \Pi_{\mathbf{G}_i}^\bot) \mathbf{M}_i^{(h+1)} \mathbf{D}^{(h)}(\mathbf{x}^{(1)}) \mathbf{D}^{(h)}(\mathbf{x}^{(2)})  \mathbf{M}_j^{(h+1)} (\Pi_{\mathbf{G}_j} + \Pi_{\mathbf{G}_j}^\bot) {\mathbf{w}}^{(h+1)}_j \nonumber\\
    &= \frac{c_\sigma}{d_h} \sum_{i,j} \mathbf{b}_i^{(h+1)} (\mathbf{x}^{(1)}) \mathbf{b}_j^{(h+1)} (\mathbf{x}^{(2)}) \left( {\mathbf{w}}^{(h+1)}_i \right)^\top \Pi_{\mathbf{G}_i}^\bot \mathbf{M}_i^{(h+1)} \mathbf{D}^{(h)}(\mathbf{x}^{(1)}) \mathbf{D}^{(h)}(\mathbf{x}^{(2)}) \mathbf{M}_j^{(h+1)} \Pi_{\mathbf{G}_j}^\bot  {\mathbf{w}}^{(h+1)}_j \nonumber\\
    &\quad + \frac{c_\sigma}{d_h} \sum_{i,j} \mathbf{b}_i^{(h+1)} (\mathbf{x}^{(1)}) \mathbf{b}_j^{(h+1)} (\mathbf{x}^{(2)}) \left( {\mathbf{w}}^{(h+1)}_i \right)^\top \Pi_{\mathbf{G}_i}^\bot \mathbf{M}_i^{(h+1)} \mathbf{D}^{(h)}(\mathbf{x}^{(1)}) \mathbf{D}^{(h)}(\mathbf{x}^{(2)}) \mathbf{M}_j^{(h+1)} \Pi_{\mathbf{G}_j} \mathbf{w}^{(h+1)}_j \nonumber\\
    &\quad + \frac{c_\sigma}{d_h} \sum_{i,j} \mathbf{b}_i^{(h+1)} (\mathbf{x}^{(1)}) \mathbf{b}_j^{(h+1)} (\mathbf{x}^{(2)}) \left( {\mathbf{w}}^{(h+1)}_i \right)^\top \Pi_{\mathbf{G}_i} \mathbf{M}_i^{(h+1)} \mathbf{D}^{(h)}(\mathbf{x}^{(1)}) \mathbf{D}^{(h)}(\mathbf{x}^{(2)}) \mathbf{M}_j^{(h+1)} \Pi_{\mathbf{G}_j}^\bot {\mathbf{w}}^{(h+1)}_j \nonumber\\
    &\quad + \frac{c_\sigma}{d_h} \sum_{i,j} \mathbf{b}_i^{(h+1)} (\mathbf{x}^{(1)}) \mathbf{b}_j^{(h+1)} (\mathbf{x}^{(2)}) \left( {\mathbf{w}}^{(h+1)}_i \right)^\top \Pi_{\mathbf{G}_i} \mathbf{M}_i^{(h+1)} \mathbf{D}^{(h)}(\mathbf{x}^{(1)}) \mathbf{D}^{(h)}(\mathbf{x}^{(2)}) \mathbf{M}_j^{(h+1)} \Pi_{\mathbf{G}_j} {\mathbf{w}}^{(h+1)}_j .
\end{align}
We now show that the main contribution from the above term is from the part that involves $\Pi^\bot_{\mathbf{G}_i}$ and is close to the term in \Cref{eq: backward_intermediate_goal}, and the part with $\Pi_{\mathbf{G}_i}$ is small.  
This is done by \Cref{prop: independent} and \Cref{prop: dependent}.
The rest of proof is by \Cref{prop: b_diff}. 
\end{proof}

\begin{proposition}\label{prop: b_diff}
If $\overline{\mathcal{A}}^L (\eps_1^2 /2) \bigcap \overline{\mathcal{B}}^{h+1}(\eps_2) \bigcap \overline{\mathcal{C}}(\eps_3) \bigcap \overline{\mathcal{D}}^h(\eps_4)$, then we have
\begin{align*}
    \left| \frac{2}{d_h} \sum_i \mathbf{b}_i^{(h+1)} (\mathbf{x}^{(1)}) \mathbf{b}_i^{(h+1)} (\mathbf{x}^{(2)}) \tr( \mathbf{M}_i^{(h+1)} \mathbf{D} \mathbf{M}_i^{(h+1)}) - \prod_{h' = h}^L \dot{\Sigma}^{(h')}(\mathbf{x}^{(1)}, \mathbf{x}^{(2)}) \right| \leq \eps_2 + 2 \eps_4.
\end{align*}
\end{proposition}
\begin{proof}
\begin{align*}
    & \left| \frac{2}{d_h} \sum_i \mathbf{b}_i^{(h+1)} (\mathbf{x}^{(1)}) \mathbf{b}_i^{(h+1)} (\mathbf{x}^{(2)}) \tr( \mathbf{M}_i^{(h+1)} \mathbf{D} \mathbf{M}_i^{(h+1)}) - \prod_{h' = h}^L \dot{\Sigma}^{(h')}(\mathbf{x}^{(1)}, \mathbf{x}^{(2)}) \right| \\
    & \leq \left| \sum_i \mathbf{b}_i^{(h+1)} (\mathbf{x}^{(1)}) \mathbf{b}_i^{(h+1)} (\mathbf{x}^{(2)}) \left(\frac{2}{d_h} \tr(\mathbf{M}_i^{(h+1)} \mathbf{D} \mathbf{M}_i^{(h+1)}) - \dot{\Sigma}^{(h)} (\mathbf{x}^{(1)}, \mathbf{x}^{(2)}) \right) \right| \\
    & + \left|\dot{\Sigma}^{(h)}(\mathbf{x}^{(1)} \mathbf{x}^{(2)}) \right| \left|\inprod{\mathbf{b}^{(h+1)} (\mathbf{x}^{(1)}) \mathbf{b}^{(h+1)}(\mathbf{x}^{(2)})} - \prod_{h' = h+1}^{L} \dot{\Sigma}^{(h')}(\mathbf{x}^{(1)}, \mathbf{x}^{(2)}) \right| \\
    & \leq \norm{\mathbf{b}^{(h+1)} (\mathbf{x}^{(1)})}_2 \norm{\mathbf{b}^{(h+1)}(\mathbf{x}^{(2)})}_2 \eps_4 + \eps_2 \\
    &= 2\eps_4 + \eps_2.
\end{align*}
\end{proof}

Before we prove \Cref{prop: independent} and \Cref{prop: dependent}, we first prove a convenient result. 
\begin{proposition}\label{prop: M_commutes_proj}
$\mathbf{M}_i^{(h+1)}$ commutes with $\Pi_{\mathbf{G}_i}$ (and thus $\Pi_{\mathbf{G}_i}^\bot$).
\end{proposition}
\begin{proof}
We can decompose $\Pi_{\mathbf{G}_i} = \Pi_{\mathbf{M}_i \mathbf{g}_1} + \Pi_{\mathbf{G}_i / \mathbf{M}_i \mathbf{g}_1}$.
Observe that $\Pi_{\mathbf{G}_i / \mathbf{M}_i \mathbf{g}_1}$ is projecting a vector into the space spanned by $\mathbf{M}_i \mathbf{g}_2 - \inprod{\mathbf{M}_i \mathbf{g}_1, \mathbf{M}_i \mathbf{g}_2} \mathbf{M}_i \mathbf{g}_1 = \mathbf{M}_i (\mathbf{g}_2 - \inprod{\mathbf{M}_i \mathbf{g}_1, \mathbf{M}_i \mathbf{g}_2} \mathbf{g}_1)$. 
Thus, we can first prove $\mathbf{M}_i$ commutes with $\Pi_{\mathbf{M}_i \mathbf{g}_1}$ and the same result follows for $\Pi_{\mathbf{G}_i / \mathbf{M}_i \mathbf{g}_1}$.
Notice that $\mathbf{M}_i \Pi_{\mathbf{M}_i \mathbf{g}_1} = \mathbf{M}_i \frac{\mathbf{M}_i \mathbf{g}_1 (\mathbf{M}_i \mathbf{g}_1)^\top}{\norm{\mathbf{M}_i \mathbf{g}_1}_2^2} = \frac{1}{\sqrt{\alpha}} \frac{\mathbf{M}_i \mathbf{g}_1 (\mathbf{M}_i \mathbf{g}_1)^\top}{\norm{\mathbf{M}_i \mathbf{g}_1}_2^2} = \frac{\mathbf{M}_i \mathbf{g}_1 (\mathbf{M}_i \mathbf{g}_1)^\top}{\norm{\mathbf{M}_i \mathbf{g}_1}_2^2} \mathbf{M}_i = \Pi_{\mathbf{M}_i \mathbf{g}_1} \mathbf{M}_i$.
\end{proof}

\subsubsection{Bounding the Independent Part}\label{sec: independent}
\begin{proposition}[Formal Version of \Cref{prop: main_text_independent}]\label{prop: independent}
Conditioned on the event in \Cref{lemma: num_remained_weights} occurs. 
With probability at least $1 - \delta_2$, if $\overline{\mathcal{A}}^L(\eps_1^2 / 2) \bigcap \overline{\mathcal{B}}^{h+1}(\eps_2) \bigcap \overline{\mathcal{C}}(\eps_3) \bigcap \overline{\mathcal{D}}^h(\eps_4)$, then for any $(\mathbf{x}^{(1)}, \mathbf{x}^{(2)}) \in \{ (\mathbf{x,x}), (\mathbf{x,x'}), (\mathbf{x',x'}) \}$, we have
\begin{align*}
    & \Bigg| \frac{c_\sigma}{d_h} \sum_{i,j} \mathbf{b}_i^{(h+1)} (\mathbf{x}^{(1)}) \mathbf{b}_j^{(h+1)} (\mathbf{x}^{(2)}) \left( \tilde{\mathbf{w}}^{(h+1)}_i \right)^\top \Pi_{\mathbf{G}_i}^\bot \mathbf{M}_i^{(h+1)} \mathbf{D} \mathbf{M}_j^{(h+1)} \Pi_{\mathbf{G}_j}^\bot \tilde{\mathbf{w}}_j^{(h+1)} \\
    & \quad - \frac{2}{d_h} \sum_i \mathbf{b}_i^{(h+1)} (\mathbf{x}^{(1)}) \mathbf{b}_i^{(h+1)} (\mathbf{x}^{(2)}) \tr( \mathbf{M}_i^{(h+1)} \mathbf{D} \mathbf{M}_i^{(h+1)}) \Bigg| \leq 3 \sqrt{\frac{8 \log \frac{6}{\delta_2}}{\alpha d_h}},
\end{align*}
which implies for any $\mathbf{x}^{(1)} \in \{\mathbf{x,x'}\}$, 
\begin{align*}
    & \norm{\sqrt{\frac{c_\sigma}{d_h}} \sum_i \mathbf{b}_i^{(h+1)}(\mathbf{x}^{(1)}) \left( \tilde{\mathbf{w}}_i^{(h+1)} \right)^\top \Pi_{\mathbf{G}_i}^\bot \mathbf{M}_i^{(h+1)} \mathbf{D}}_2 \leq \sqrt{4 + 3 \sqrt{\frac{8 \log \frac{6}{\delta_2}}{\alpha d_h}}} \leq 6,
\end{align*}
if $d_h \geq \frac{8}{\alpha} \log \frac{6}{\delta_2}$. 
\end{proposition}
\begin{proof}
First, we compute the difference between the projected version of the inner product and normal inner product in expectation: 
First we have
\begin{align*}
    & \E_{\tilde{\mathbf{W}}^{(h+1)}} \left( \frac{c_\sigma}{d_h} \sum_{i,j} \mathbf{b}_i^{(h+1)} (\mathbf{x}^{(1)}) \mathbf{b}_j^{(h+1)} (\mathbf{x}^{(2)}) \left( \tilde{\mathbf{w}}^{(h+1)}_i \right)^\top \mathbf{M}_i^{(h+1)} \mathbf{D} \mathbf{M}_j^{(h+1)} \tilde{\mathbf{w}}_j^{(h+1)} \right) \\
    &= \frac{c_\sigma}{d_h} \sum_i \mathbf{b}_i^{(h+1)} (\mathbf{x}^{(1)}) \mathbf{b}_i^{(h+1)} (\mathbf{x}^{(2)}) \tr( \mathbf{M}_i^{(h+1)} \mathbf{D} \mathbf{M}_i^{(h+1)}). 
\end{align*}
Then, 
\begin{align*}
    & \E_{\tilde{\mathbf{W}}^{(h+1)}} \Bigg( \frac{c_\sigma}{d_h} \sum_{i,j} \mathbf{b}_i^{(h+1)} (\mathbf{x}^{(1)}) \mathbf{b}_j^{(h+1)} (\mathbf{x}^{(2)}) \left( \tilde{\mathbf{w}}^{(h+1)}_i \right)^\top \Pi_{\mathbf{G}_i}^\bot \mathbf{M}_i^{(h+1)} \mathbf{D} \mathbf{M}_j^{(h+1)}  \Pi_{\mathbf{G}_j}^\bot \tilde{\mathbf{w}}_j^{(h+1)} \\
    &\quad - \frac{c_\sigma}{d_h} \sum_{i,j} \mathbf{b}_i^{(h+1)} (\mathbf{x}^{(1)}) \mathbf{b}_j^{(h+1)} (\mathbf{x}^{(2)}) \left( \tilde{\mathbf{w}}^{(h+1)}_i \right)^\top \mathbf{M}_i^{(h+1)} \mathbf{D} \mathbf{M}_j^{(h+1)} \tilde{\mathbf{w}}_j^{(h+1)} \Bigg)\\
    &= \frac{c_\sigma}{d_h} \sum_i \mathbf{b}_i^{(h+1)} (\mathbf{x}^{(1)}) \mathbf{b}_i^{(h+1)} (\mathbf{x}^{(2)}) \tr( \Pi_{\mathbf{G}_i}^\bot \mathbf{M}_i^{(h+1)} \mathbf{D} \mathbf{M}_i^{(h+1)} \Pi_{\mathbf{G}_i}^\bot - \mathbf{M}_i^{(h+1)} \mathbf{D} \mathbf{M}_i^{(h+1)}) \\
    &= \frac{c_\sigma}{d_h} \sum_i \mathbf{b}_i^{(h+1)} (\mathbf{x}^{(1)}) \mathbf{b}_i^{(h+1)} (\mathbf{x}^{(2)}) \tr( (\Pi_{\mathbf{G}_i}^\bot - I) \mathbf{M}_i^{(h+1)} \mathbf{D} \mathbf{M}_i^{(h+1)}) \\
    &= \frac{c_\sigma}{d_h} \sum_i \mathbf{b}_i^{(h+1)} (\mathbf{x}^{(1)}) \mathbf{b}_i^{(h+1)} (\mathbf{x}^{(2)}) \tr( \Pi_{\mathbf{G}_i} \mathbf{M}_i^{(h+1)} \mathbf{D} \mathbf{M}_i^{(h+1)}) ,
\end{align*}
where the third last equality is true because we can interchange between $\mathbf{M}_i^{(h+1)}$ and $\Pi_{\mathbf{G}_i}^\bot$. 
And the second last equality is because $\tr( \Pi_{\mathbf{G}_i}^\bot \mathbf{M}_i^{(h+1)} \mathbf{D} \mathbf{M}_i^{(h+1)} \Pi_{\mathbf{G}_i}^\bot - \mathbf{M}_i^{(h+1)} \mathbf{D} \mathbf{M}_i^{(h+1)}) = \tr(\Pi_{\mathbf{G}_i}^\bot \mathbf{M}_i^{(h+1)} \mathbf{D} \mathbf{M}_i^{(h+1)} \Pi_{\mathbf{G}_i}^\bot) - \tr(\mathbf{M}_i^{(h+1)} \mathbf{D} \mathbf{M}_i^{(h+1)}) = \tr(\Pi_{\mathbf{G}_i}^\bot \mathbf{M}_i^{(h+1)} \mathbf{D} \mathbf{M}_i^{(h+1)} ) - \tr(\mathbf{M}_i^{(h+1)} \mathbf{D} \mathbf{M}_i^{(h+1)}) = \tr( (\Pi_{\mathbf{G}_i}^\bot - I) \mathbf{M}_i^{(h+1)} \mathbf{D} \mathbf{M}_i^{(h+1)}) = \tr(\Pi_{\mathbf{G}_i} \mathbf{M}_i^{(h+1)} \mathbf{D} \mathbf{M}_i^{(h+1)})$.
Since $\textnormal{rank}(\Pi_{\mathbf{G}_i}) \leq 2$ and $\norm{\mathbf{M}_i^{(h+1)} \mathbf{D} \mathbf{M}_i^{(h+1)}}_2 \leq \frac{1}{\alpha}$, we have 
\begin{align*}
    0 \leq \tr(\Pi_{\mathbf{G}_i} \mathbf{M}_i^{(h+1)} \mathbf{D} \mathbf{M}_i^{(h+1)}) \leq \frac{2}{\alpha}.
\end{align*}
Now notice that
\begin{align*}
    & \sum_i \mathbf{b}_i^{(h+1)} (\mathbf{x}^{(1)}) \mathbf{b}_i^{(h+1)} (\mathbf{x}^{(2)}) \tr( \Pi_{\mathbf{G}_i} \mathbf{M}_i^{(h+1)} \mathbf{D} \mathbf{M}_i^{(h+1)}) = \mathbf{b}^{(h+1)} (\mathbf{x}^{(1)})^\top \mathbf{T} \mathbf{b}^{(h+1)} (\mathbf{x}^{(2)}),
\end{align*}
where 
\begin{align*}
    \mathbf{T} = \begin{bmatrix}
    \tr( \Pi_{\mathbf{G}_1} \mathbf{M}_1^{(h+1)} \mathbf{D} \mathbf{M}_1^{(h+1)}) & 0 & \ldots & 0 \\
    0 & \tr( \Pi_{\mathbf{G}_2} \mathbf{M}_2^{(h+1)} \mathbf{D} \mathbf{M}_2^{(h+1)}) & \ldots & 0 \\
    \vdots & \vdots & & \vdots \\
    0 & 0 & \ldots & \tr( \Pi_{\mathbf{G}_{d_{h+1}}} \mathbf{M}_{d_{h+1}}^{(h+1)} \mathbf{D} \mathbf{M}_{d_{h+1}}^{(h+1)})
    \end{bmatrix}.
\end{align*}
Notice that $\norm{\mathbf{T}}_2 \leq \frac{2}{\alpha}$ and thus, $|\mathbf{b}^{(h+1)}(\mathbf{x}^{(1)})^\top \mathbf{T} \mathbf{b}^{(h+1)}(\mathbf{x}^{(2)})| \leq \frac{2}{\alpha} \norm{\mathbf{b}^{(h+1)}(\mathbf{x}^{(1)})}_2 \norm{\mathbf{b}^{(h+1)}(\mathbf{x}^{(2)})}_2$.
Therefore, we have
\begin{align}\label{eq: projected_bias}
    & \E_{\tilde{\mathbf{W}}^{(h+1)}} \Bigg( \frac{c_\sigma}{d_h} \sum_{i,j} \mathbf{b}_i^{(h+1)} (\mathbf{x}^{(1)}) \mathbf{b}_j^{(h+1)} (\mathbf{x}^{(2)}) \left( \tilde{\mathbf{w}}^{(h+1)}_i \right)^\top \Pi_{\mathbf{G}_i}^\bot \mathbf{M}_i^{(h+1)} \mathbf{D} \mathbf{M}_j^{(h+1)} \Pi_{\mathbf{G}_j}^\bot  \tilde{\mathbf{w}}_j^{(h+1)} \nonumber \\
    &\quad - \frac{c_\sigma}{d_h} \sum_{i,j} \mathbf{b}_i^{(h+1)} (\mathbf{x}^{(1)}) \mathbf{b}_j^{(h+1)} (\mathbf{x}^{(2)}) \left( \tilde{\mathbf{w}}^{(h+1)}_i \right)^\top \mathbf{M}_i^{(h+1)} \mathbf{D} \mathbf{M}_j^{(h+1)} \tilde{\mathbf{w}}_j^{(h+1)} \Bigg) \nonumber \\
    &\leq \frac{c_\sigma}{d_h} \frac{2}{\alpha} \norm{\mathbf{b}^{(h+1)}(\mathbf{x}^{(1)})}_2 \norm{\mathbf{b}^{(h+1)}(\mathbf{x}^{(2)})}_2 \leq \frac{c_\sigma}{d_h} \frac{8}{\alpha}.
\end{align}
Next, we analyze concentration of
\begin{align*}
    \frac{c_\sigma}{d_h} \sum_{i,j} \mathbf{b}_i^{(h+1)} (\mathbf{x}^{(1)}) \mathbf{b}_j^{(h+1)} (\mathbf{x}^{(2)}) \left( \tilde{\mathbf{w}}^{(h+1)}_i \right)^\top \Pi_{\mathbf{G}_i}^\bot \mathbf{M}_i^{(h+1)} \mathbf{D} \mathbf{M}_j^{(h+1)} \Pi_{\mathbf{G}_j}^\bot \tilde{\mathbf{w}}_j^{(h+1)}.
\end{align*}
Since the following new random vector has multivariate Gaussian distribution, we can write
\[
\begin{bmatrix}
\sum_{i=1}^{d_{h+1}} \mathbf{b}_i^{(h+1)}(\mathbf{x}^{(1)}) ((\tilde{\mathbf{w}}^{(h+1)}_i )^\top \Pi_{\mathbf{G}_i}^\bot) \odot \mathbf{m}^{(h+1)}_i & \sum_{i=1}^{d_{h+1}} \mathbf{b}_i^{(h+1)}(\mathbf{x}^{(2)}) ((\tilde{\mathbf{w}}^{(h+1)}_i)^\top \Pi_{\mathbf{G}_i}^\bot) \odot \mathbf{m}^{(h+1)}_i 
\end{bmatrix}^\top
\stackrel{D}{=} \mathbf{M} \boldsymbol{\xi},
\]
where $\boldsymbol{\xi} \sim \mathcal{N}(\mathbf{0}, \mathbf{I}_{2d_h})$, and $\mathbf{M} \in \R^{2d_h \times 2d_h}$ and its covariance matrix is given by a blocked symmetric matrix
\[
\mathbf{C} = 
\begin{bmatrix}
\mathbf{C}(\mathbf{x}^{(1)}, \mathbf{x}^{(1)}) & \mathbf{C}(\mathbf{x}^{(1)}, \mathbf{x}^{(2)}) \\
\mathbf{C}(\mathbf{x}^{(1)}, \mathbf{x}^{(2)}) & \mathbf{C}(\mathbf{x}^{(2)}, \mathbf{x}^{(2)}) 
\end{bmatrix}
= \mathbf{MM}^\top,
\]
where each block is given by
\begin{align*}
    & \mathbf{C}(\mathbf{x}^{(p)}, \mathbf{x}^{(q)}) \\
    &= \E_{\tilde{\mathbf{W}}^{(h+1)}} \left( \sum_{i=1}^{d_{h+1}} \mathbf{b}_i^{(h+1)} (\mathbf{x}^{(p)}) ((\tilde{\mathbf{w}}_i^{(h+1)} )^\top \Pi_{\mathbf{G}_i}^\bot ) \odot \mathbf{m}_i^{(h+1)} \right)^\top \left( \sum_{j=1}^{d_{h+1}} \mathbf{b}_j^{(h+1)} (\mathbf{x}^{(q)}) ((\tilde{\mathbf{w}}_j^{(h+1)} )^\top  \Pi_{\mathbf{G}_j}^\bot) \odot \mathbf{m}_j^{(h+1)} \right) \\
    &= \E_{\tilde{\mathbf{W}}^{(h+1)}} \left( \sum_{i=1}^{d_{h+1}} \mathbf{b}_i^{(h+1)} (\mathbf{x}^{(p)}) (\Pi_{\mathbf{G}_i}^\bot \tilde{\mathbf{w}}_i^{(h+1)}) \odot \mathbf{m}_i^{(h+1)} \right) \left( \sum_{j=1}^{d_{h+1}} \mathbf{b}_j^{(h+1)} (\mathbf{x}^{(q)}) ((\tilde{\mathbf{w}}_j^{(h+1)} )^\top \Pi_{\mathbf{G}_j}^\bot) \odot \mathbf{m}_j^{(h+1)} \right) \\
    &= \E_{\tilde{\mathbf{W}}^{(h+1)}} \left( \sum_{i=1}^{d_{h+1}} \sum_{j=1}^{d_{h+1}} \mathbf{b}_i^{(h+1)}(\mathbf{x}^{(p)}) \mathbf{b}_j^{(h+1)}(\mathbf{x}^{(q)}) \Pi_{\mathbf{G}_i}^\bot (\tilde{\mathbf{w}}_i^{(h+1)} \odot \mathbf{m}_i^{(h+1)}) (\tilde{\mathbf{w}}_j^{(h+1)} \odot \mathbf{m}_j^{(h+1)})^\top \Pi_{\mathbf{G}_j}^\bot \right) \\
    &= \sum_{i=1}^{d_{h+1}} \mathbf{b}_i^{(h+1)} (\mathbf{x}^{(p)}) \mathbf{b}_i^{(h+1)} (\mathbf{x}^{(q)}) \Pi_{\mathbf{G}_i}^\bot \left( \E_{\tilde{\mathbf{W}}^{(h+1)}} (\tilde{\mathbf{w}}_i^{(h+1)} \odot \mathbf{m}_i^{(h+1)}) (\tilde{\mathbf{w}}_i^{(h+1)} \odot \mathbf{m}_i^{(h+1)})^\top \right) \Pi_{\mathbf{G}_i}^\bot  \\
    &= \sum_{i=1}^{d_{h+1}} \mathbf{b}_i^{(h+1)} (\mathbf{x}^{(p)}) \mathbf{b}_i^{(h+1)} (\mathbf{x}^{(q)}) \Pi_{\mathbf{G}_i}^\bot \textnormal{diag}\left(\left(\mathbf{m}_i^{(h+1)} \right)^2 \right) \Pi_{\mathbf{G}_i}^\bot ,
\end{align*}
where the third equality is from \Cref{prop: M_commutes_proj} and the square on a vector in the last equality is applied element-wise.
Therefore, we can write
\begin{align*}
\mathbf{C} &= 
\begin{bmatrix}
\mathbf{C}(\mathbf{x}^{(1)}, \mathbf{x}^{(1)}) & \mathbf{C}(\mathbf{x}^{(1)}, \mathbf{x}^{(2)}) \\
\mathbf{C}(\mathbf{x}^{(1)}, \mathbf{x}^{(2)}) & \mathbf{C}(\mathbf{x}^{(2)}, \mathbf{x}^{(2)}) 
\end{bmatrix}
\\
&= \sum_{i=1}^{d_{h+1}} 
\begin{bmatrix}
\mathbf{b}_i^{(h+1)} (\mathbf{x}^{(1)}) \mathbf{b}_i^{(h+1)} (\mathbf{x}^{(1)}) & \mathbf{b}_i^{(h+1)} (\mathbf{x}^{(1)}) \mathbf{b}_i^{(h+1)} (\mathbf{x}^{(2)}) \\
\mathbf{b}_i^{(h+1)} (\mathbf{x}^{(1)}) \mathbf{b}_i^{(h+1)} (\mathbf{x}^{(2)}) & \mathbf{b}_i^{(h+1)} (\mathbf{x}^{(2)}) \mathbf{b}_i^{(h+1)} (\mathbf{x}^{(2)})
\end{bmatrix}
\otimes \Pi_{\mathbf{G}_i}^\bot \textnormal{diag}\left(\left(\mathbf{m}_i^{(h+1)}\right)^2 \right) \Pi_{\mathbf{G}_i}^\bot .
\end{align*}

\textbf{Bounding the Operator Norm of the Covariance Matrix $\mathbf{C}$}

Next, we want to show that 
\begin{align}\label{eq: cov_diff}
    \sum_{i=1}^{d_{h+1}} 
\begin{bmatrix}
\mathbf{b}_i^{(h+1)} (\mathbf{x}^{(1)}) \mathbf{b}_i^{(h+1)} (\mathbf{x}^{(1)}) & \mathbf{b}_i^{(h+1)} (\mathbf{x}^{(1)}) \mathbf{b}_i^{(h+1)} (\mathbf{x}^{(2)}) \\
\mathbf{b}_i^{(h+1)} (\mathbf{x}^{(1)}) \mathbf{b}_i^{(h+1)} (\mathbf{x}^{(2)}) & \mathbf{b}_i^{(h+1)} (\mathbf{x}^{(2)}) \mathbf{b}_i^{(h+1)} (\mathbf{x}^{(2)})
\end{bmatrix}
\otimes \left( \frac{1}{\alpha} \mathbf{I} - \Pi_{\mathbf{G}_i}^\bot \textnormal{diag}\left(\left(\mathbf{m}_i^{(h+1)}\right)^2 \right) \Pi_{\mathbf{G}_i}^\bot \right) \succeq \mathbf{0}.
\end{align}
Given this, since Kronecker product preserves two norm we have that
\begin{align*}
    \norm{\mathbf{C}}_2 &\leq \frac{1}{\alpha} \norm{\sum_{i=1}^{d_{h+1}} \begin{bmatrix}
\mathbf{b}_i^{(h+1)} (\mathbf{x}^{(1)}) \mathbf{b}_i^{(h+1)} (\mathbf{x}^{(1)}) & \mathbf{b}_i^{(h+1)} (\mathbf{x}^{(1)}) \mathbf{b}_i^{(h+1)} (\mathbf{x}^{(2)}) \\
\mathbf{b}_i^{(h+1)} (\mathbf{x}^{(1)}) \mathbf{b}_i^{(h+1)} (\mathbf{x}^{(2)}) & \mathbf{b}_i^{(h+1)} (\mathbf{x}^{(2)}) \mathbf{b}_i^{(h+1)} (\mathbf{x}^{(2)})
\end{bmatrix}
}_2 \\
&= \frac{1}{\alpha} \norm{
\begin{bmatrix}
\inprod{\mathbf{b}^{(h+1)}(\mathbf{x}^{(1)}), \mathbf{b}^{(h+1)}(\mathbf{x}^{(1)})} & \inprod{\mathbf{b}^{(h+1)}(\mathbf{x}^{(1)}), \mathbf{b}^{(h+1)}(\mathbf{x}^{(2)})} \\
\inprod{\mathbf{b}^{(h+1)}(\mathbf{x}^{(1)}), \mathbf{b}^{(h+1)}(\mathbf{x}^{(2)})} & 
\inprod{\mathbf{b}^{(h+1)}(\mathbf{x}^{(2)}), \mathbf{b}^{(h+1)}(\mathbf{x}^{(2)})}
\end{bmatrix}
}_2 \\
&\leq \frac{1}{\alpha} \sqrt{2} \left(\inprod{\mathbf{b}^{(h+1)}(\mathbf{x}^{(1)}), \mathbf{b}^{(h+1)}(\mathbf{x}^{(1)})} + \inprod{\mathbf{b}^{(h+1)}(\mathbf{x}^{(1)}), \mathbf{b}^{(h+1)}(\mathbf{x}^{(2)})} \right),
\end{align*}
where the last inequality is by applying $\norm{\mathbf{A}}_2 \leq \sqrt{m} \norm{\mathbf{A}}_\infty$.

We prove the matrix in Equation \eqref{eq: cov_diff} is positive semi-definite by constructing a multivariate Gaussian distribution such that its covariance matrix is exactly the matrix and exploring the fact that the covariance matrix of two independent Gaussian distribution is the sum of the two covariance matrix.
First, notice that 
\begin{align*}
    \frac{1}{\alpha} \mathbf{I} - \Pi_{\mathbf{G}_i}^\bot \textnormal{diag}\left( \left( \mathbf{m}_i^{(h+1)} \right)^2 \right) \Pi_{\mathbf{G}_i}^\bot &= \frac{1}{\alpha} \left(\Pi_{\mathbf{G}_i}^\bot + \Pi_{\mathbf{G}_i} \right) - \Pi_{\mathbf{G}_i}^\bot \textnormal{diag}\left( \left( \mathbf{m}_i^{(h+1)} \right)^2 \right) \Pi_{\mathbf{G}_i}^\bot \\
    &= \Pi_{\mathbf{G}_i}^\bot \left(\frac{1}{\alpha} \mathbf{I} - \textnormal{diag}\left(\left(\mathbf{m}_i^{(h+1)} \right)^2 \right) \Pi_{\mathbf{G}_i}^\bot \right) + \frac{1}{\alpha} \Pi_{\mathbf{G}_i} \\
    &= \Pi_{\mathbf{G}_i}^\bot \left(\frac{1}{\alpha} \left( \Pi_{\mathbf{G}_i}^\bot + \Pi_{\mathbf{G}_i} \right) - \textnormal{diag}\left(\left(\mathbf{m}_i^{(h+1)} \right)^2 \right) \Pi_{\mathbf{G}_i}^\bot \right) + \frac{1}{\alpha} \Pi_{\mathbf{G}_i} \\
    &= \Pi_{\mathbf{G}_i}^\bot \left(\frac{1}{\alpha} \Pi_{\mathbf{G}_i} + \left( \frac{1}{\alpha}\mathbf{I}  - \textnormal{diag}\left(\left(\mathbf{m}_i^{(h+1)} \right)^2 \right) \right) \Pi_{\mathbf{G}_i}^\bot \right) + \frac{1}{\alpha} \Pi_{\mathbf{G}_i} \\
    &= \Pi_{\mathbf{G}_i}^\bot  \left( \frac{1}{\alpha}\mathbf{I}  - \textnormal{diag}\left(\left(\mathbf{m}_i^{(h+1)} \right)^2 \right) \right) \Pi_{\mathbf{G}_i}^\bot + \frac{1}{\alpha} \Pi_{\mathbf{G}_i} .
\end{align*}
The final Gaussian is constructed by the sum of the following two groups of Gaussian: let $\mathbf{W}_1, \mathbf{W}_2$ be two independent standard Gaussian matrices,
\begin{align*}
&
\begin{bmatrix}
\sum_{i=1}^{d_{h+1}} \mathbf{b}_i(\mathbf{x}^{(1)}) ({\mathbf{w}}^{(h+1)}_{1,i})^\top \left( \frac{1}{\sqrt{\alpha}}\mathbf{I}  - \textnormal{diag}\left(\mathbf{m}_i^{(h+1)} \right)  \right) \Pi_{\mathbf{G}_i}^\bot & \sum_{i=1}^{d_{h+1}} \mathbf{b}_i(\mathbf{x}^{(2)}) ({\mathbf{w}}^{(h+1)}_{1,i})^\top \left( \frac{1}{\sqrt{\alpha}}\mathbf{I}  - \textnormal{diag}\left(\mathbf{m}_i^{(h+1)} \right)  \right) \Pi_{\mathbf{G}_i}^\bot
\end{bmatrix} ,
\\
& \begin{bmatrix}
\sum_{i=1}^{d_{h+1}} \mathbf{b}_i(\mathbf{x}^{(1)}) ({\mathbf{w}}^{(h+1)}_{2,i})^\top \frac{1}{\sqrt{\alpha}} \Pi_{\mathbf{G}_i}
& 
\sum_{i=1}^{d_{h+1}} \mathbf{b}_i(\mathbf{x}^{(2)}) ({\mathbf{w}}^{(h+1)}_{2,i})^\top \frac{1}{\sqrt{\alpha}} \Pi_{\mathbf{G}_i}
\end{bmatrix},
\end{align*}
where $\mathbf{w}_{i,j}$ denote the $j$-th row of $\mathbf{W}_i$.

Now conditioned on $\{\mathbf{b}^{(h+1)}(\mathbf{x}^{(1)}), \mathbf{b}^{(h+1)}(\mathbf{x}^{(2)}), \mathbf{g}^{(h)}(\mathbf{x}^{(1)}), \mathbf{g}^{(h)}(\mathbf{x}^{(2)})\}$, we have
\begin{align*}
    & \left( \sum_{i=1}^{d_{h+1}} \mathbf{b}_i^{(h+1)} (\mathbf{x}^{(1)}) ({\mathbf{w}}^{(h+1)}_i \odot \mathbf{m}^{(h+1)}_i)^\top \Pi_{\mathbf{G}_i}^\bot \right) \mathbf{D} \left( \sum_{i=1}^{d_{h+1}} \mathbf{b}_i^{(h+1)} (\mathbf{x}^{(2)}) ({\mathbf{w}}^{(h+1)}_i \odot \mathbf{m}^{(h+1)}_i)^\top \Pi_{\mathbf{G}_i}^\bot \right) \\
    &\stackrel{D}{=} \left( \sum_{i=1}^{d_{h+1}} \mathbf{b}_i^{(h+1)} (\mathbf{x}^{(1)}) (\tilde{\mathbf{w}}^{(h+1)}_i \odot \mathbf{m}^{(h+1)}_i)^\top \Pi_{\mathbf{G}_i}^\bot \right) \mathbf{D} \left( \sum_{i=1}^{d_{h+1}} \mathbf{b}_i^{(h+1)} (\mathbf{x}^{(2)}) (\tilde{\mathbf{w}}^{(h+1)}_i \odot \mathbf{m}^{(h+1)}_i)^\top \Pi_{\mathbf{G}_i}^\bot \right) \\
    &\stackrel{D}{=} ([\mathbf{I}_{d_h} \quad \mathbf{0}] \mathbf{M}\boldsymbol{\xi})^\top \mathbf{D} ([\mathbf{0} \quad \mathbf{I}_{d_h}] \mathbf{M}\boldsymbol{\xi}) \\
    &\stackrel{D}{=} \frac{1}{2} \boldsymbol{\xi}^\top \mathbf{M}^\top \begin{bmatrix}
    \mathbf{0} & \mathbf{D} \\
    \mathbf{D} & \mathbf{0}
    \end{bmatrix}
    \mathbf{M} \boldsymbol{\xi}.
\end{align*}
Now, let 
\begin{align*}
    \mathbf{A} = \frac{1}{2} \mathbf{M}^\top 
    \begin{bmatrix}
    \mathbf{0} & \mathbf{D} \\
    \mathbf{D} & \mathbf{0}
    \end{bmatrix}
    \mathbf{M},
\end{align*}
and we have
\begin{align*}
    \norm{\mathbf{A}}_2 &\leq \frac{1}{2} \norm{\mathbf{M}}_2^2 \norm{\mathbf{D}}_2 \\
    &= \frac{1}{2} \norm{\mathbf{MM^\top}}_2 \norm{\mathbf{D}}_2 \\
    &= \frac{1}{2} \norm{\mathbf{C}}_2 \\
    &\leq \frac{1}{2 \alpha} \sqrt{2} \left(\inprod{\mathbf{b}^{(h+1)} (\mathbf{x}^{(1)}), \mathbf{b}^{(h+1)} (\mathbf{x}^{(1)})} + \inprod{\mathbf{b}^{(h+1)} (\mathbf{x}^{(1)}), \mathbf{b}^{(h+1)} (\mathbf{x}^{(2)})} \right) \\
    &\leq \frac{2 \sqrt{2}}{\alpha}.
\end{align*}

\textbf{Bounding the Trace of the Covariance Matrix $\mathbf{C}$}

Naively apply 2-norm-Frobenius-norm bound for matrices will give us 
\begin{align*}
    \norm{\mathbf{A}}_F \leq \sqrt{2d_h} \norm{\mathbf{A}}_2 \leq \frac{4 \sqrt{d_h}}{\alpha}.
\end{align*}
We prove a better bound.
Observe that
\begin{align*}
    \frac{1}{d_h} \norm{\mathbf{A}}_F = \frac{1}{d_h} \norm{\frac{1}{2} \mathbf{M}^\top 
    \begin{bmatrix}
    \mathbf{0} & \mathbf{D} \\
    \mathbf{D} & \mathbf{0}
    \end{bmatrix}
    \mathbf{M}
    }_F \leq \frac{1}{2 d_h} \norm{\mathbf{M}}_2 \norm{\mathbf{M}}_F \norm{\mathbf{D}}_2 = \frac{1}{2 d_h \sqrt{\alpha}} \norm{\mathbf{M}}_F = \frac{1}{2 d_h \sqrt{\alpha}} \sqrt{\tr(\mathbf{MM}^\top)}.
\end{align*}
Using the similar idea from bounding the 2-norm of $\mathbf{C} = \mathbf{M} \mathbf{M}^\top$, we want to show that 
\begin{align}
    & \sum_{i=1}^{d_{h+1}} 
\begin{bmatrix}
\mathbf{b}_i^{(h+1)} (\mathbf{x}^{(1)}) \mathbf{b}_i^{(h+1)} (\mathbf{x}^{(1)}) & \mathbf{b}_i^{(h+1)} (\mathbf{x}^{(1)}) \mathbf{b}_i^{(h+1)} (\mathbf{x}^{(2)}) \\
\mathbf{b}_i^{(h+1)} (\mathbf{x}^{(1)}) \mathbf{b}_i^{(h+1)} (\mathbf{x}^{(2)}) & \mathbf{b}_i^{(h+1)} (\mathbf{x}^{(2)}) \mathbf{b}_i^{(h+1)} (\mathbf{x}^{(2)})
\end{bmatrix}
\otimes \left( \left( \mathbf{M}_i^{(h+1)} \right)^2 - \Pi_{\mathbf{G}_i}^\bot \left( \mathbf{M}_i^{(h+1)} \right)^2  \Pi_{\mathbf{G}_i}^\bot \right) \nonumber \\
&= \sum_{i=1}^{d_{h+1}} 
\begin{bmatrix}
\mathbf{b}_i^{(h+1)} (\mathbf{x}^{(1)}) \mathbf{b}_i^{(h+1)} (\mathbf{x}^{(1)}) & \mathbf{b}_i^{(h+1)} (\mathbf{x}^{(1)}) \mathbf{b}_i^{(h+1)} (\mathbf{x}^{(2)}) \\
\mathbf{b}_i^{(h+1)} (\mathbf{x}^{(1)}) \mathbf{b}_i^{(h+1)} (\mathbf{x}^{(2)}) & \mathbf{b}_i^{(h+1)} (\mathbf{x}^{(2)}) \mathbf{b}_i^{(h+1)} (\mathbf{x}^{(2)})
\end{bmatrix}
\otimes \Pi_{\mathbf{G}_i} \left( \mathbf{M}_i^{(h+1)} \right)^2 \Pi_{\mathbf{G}_i} \succeq \mathbf{0}. \label{eq: fro_A_step}
\end{align}
If this equation is true, then we have
\begin{align*}
    \frac{1}{d_h} \tr(\mathbf{MM}^\top) &\leq \frac{1}{d_h} \tr \left( \sum_{i=1}^{d_{h+1}} 
\begin{bmatrix}
\mathbf{b}_i^{(h+1)} (\mathbf{x}^{(1)}) \mathbf{b}_i^{(h+1)} (\mathbf{x}^{(1)}) & \mathbf{b}_i^{(h+1)} (\mathbf{x}^{(1)}) \mathbf{b}_i^{(h+1)} (\mathbf{x}^{(2)}) \\
\mathbf{b}_i^{(h+1)} (\mathbf{x}^{(1)}) \mathbf{b}_i^{(h+1)} (\mathbf{x}^{(2)}) & \mathbf{b}_i^{(h+1)} (\mathbf{x}^{(2)}) \mathbf{b}_i^{(h+1)} (\mathbf{x}^{(2)})
\end{bmatrix}
\otimes \mathbf{M}_i^2 \right) \\
&= \frac{1}{d_h} \sum_{i=1}^{d_{h+1}} \left[ \left( \mathbf{b}_i^{(h+1)} (\mathbf{x}^{(1)}) \right)^2 + \left( \mathbf{b}_i^{(h+1)} (\mathbf{x}^{(2)}) \right)^2 \right] \tr\left( \left( \mathbf{M}_i^{(h+1)} \right)^2 \right) \\
&\leq \frac{1}{d_h} \sum_{i=1}^{d_{h+1}} \left[ \left( \mathbf{b}_i^{(h+1)} (\mathbf{x}^{(1)}) \right)^2 + \left( \mathbf{b}_i^{(h+1)} (\mathbf{x}^{(2)}) \right)^2 \right] \max_i \tr\left( \left( \mathbf{M}_i^{(h+1)} \right)^2 \right) \\
&= \max_i \frac{1}{d_h} \tr\left( \left( \mathbf{M}_i^{(h+1)} \right)^2 \right) \left( \norm{\mathbf{b}^{(h+1)} (\mathbf{x}^{(1)})}_2^2 + \norm{\mathbf{b}^{(h+1)} (\mathbf{x}^{(2)})}_2^2 \right) \\
&\leq 4 \max_i \frac{1}{d_h} \tr\left( \left( \mathbf{M}_i^{(h+1)} \right)^2 \right).
\end{align*}
By \Cref{lemma: num_remained_weights}, with probability $\geq 1 - \delta$, $\max_i \frac{1}{d_h} \tr \left(\left( \mathbf{M}_i^{(h+1)} \right)^2 \right) \leq 1+1 = 2$. 
Thus, we have
\begin{align*}
    \frac{1}{d_h} \norm{\mathbf{A}}_F \leq \sqrt{\frac{2}{d_h {\alpha}}}.
\end{align*}
To prove Equation \eqref{eq: fro_A_step}, since $\mathbf{M}_i$ commutes with $\Pi_{\mathbf{G}_i}^\bot$, we have
\begin{align*}
    \mathbf{M}_i^2 - \Pi_{\mathbf{G}_i}^\bot \mathbf{M}_i^2 \Pi_{\mathbf{G}_i}^\bot = \mathbf{M}_i^2 - \mathbf{M}_i^2 \Pi_{\mathbf{G}_i}^\bot = \mathbf{M}_i^2 \Pi_{\mathbf{G}_i} = \Pi_{\mathbf{G}_i} \mathbf{M}_i^2 \Pi_{\mathbf{G}_i}.
\end{align*}
The Gaussian vector given by 
\begin{align*}
    \begin{bmatrix}
\sum_{i=1}^{d_{h+1}} \mathbf{b}_i(\mathbf{x}^{(1)}) ({\mathbf{w}}^{(h+1)}_{2,i})^\top \mathbf{M}_i \Pi_{\mathbf{G}_i}
& 
\sum_{i=1}^{d_{h+1}} \mathbf{b}_i(\mathbf{x}^{(2)}) ({\mathbf{w}}^{(h+1)}_{2,i})^\top \mathbf{M}_i \Pi_{\mathbf{G}_i}
\end{bmatrix}
\end{align*}
has the covariance matrix. 

Now apply Gaussian chaos concentration bound (Lemma \ref{lemma: gaussian_chaos}), we have with probability $1 - \frac{\delta_2}{6}$,
\begin{align}\label{eq: gaussian_chaos_concentration}
    \frac{1}{d_h} |\boldsymbol{\xi}^\top \mathbf{A} \boldsymbol{\xi} - \E[\boldsymbol{\xi}^\top \mathbf{A} \boldsymbol{\xi}]| &\leq \frac{1}{d_h} \left( 2 \norm{\mathbf{A}}_F \sqrt{\log \frac{6}{\delta_2}} + 2 \norm{\mathbf{A}}_2 \log \frac{6}{\delta_2} \right) \nonumber \\
    &\leq \sqrt{\frac{8 \log \frac{6}{\delta_2}}{\alpha d_h}} + 4\sqrt{2} \frac{\log \frac{6}{\delta_2}}{\alpha d_h}.
\end{align}
Finally, combining Equation \ref{eq: projected_bias} and Equation \ref{eq: gaussian_chaos_concentration}, we have
\begin{align*}
    & \Bigg| \frac{c_\sigma}{d_h} \sum_{i,j} \mathbf{b}_i^{(h+1)} (\mathbf{x}^{(1)}) \mathbf{b}_j^{(h+1)} (\mathbf{x}^{(2)}) \left( \tilde{\mathbf{w}}^{(h+1)}_i \right)^\top \mathbf{M}_i^{(h+1)} \Pi_{\mathbf{G}_i}^\bot \mathbf{D}  \Pi_{\mathbf{G}_j}^\bot \mathbf{M}_j^{(h+1)} \tilde{\mathbf{w}}_j^{(h+1)} \\
    & \quad - \frac{2}{d_h} \sum_i \mathbf{b}_i^{(h+1)} (\mathbf{x}^{(1)}) \mathbf{b}_i^{(h+1)} (\mathbf{x}^{(2)}) \tr( \mathbf{M}_i^{(h+1)} \mathbf{D} \mathbf{M}_i^{(h+1)}) \Bigg|\\
    &\leq \frac{2}{d_h} |\boldsymbol{\xi}^\top \mathbf{A} \boldsymbol{\xi} - \E[\boldsymbol{\xi}^\top \mathbf{A} \boldsymbol{\xi}]| + \left| \frac{2}{d_h} \E\left[\boldsymbol{\xi}^\top \mathbf{A} \boldsymbol{\xi} \right] - \frac{2}{d_h} \sum_i \mathbf{b}_i^{(h+1)} (\mathbf{x}^{(1)}) \mathbf{b}_i^{(h+1)} (\mathbf{x}^{(2)}) \tr( \mathbf{M}_i^{(h+1)} \mathbf{D} \mathbf{M}_i^{(h+1)}) \right| \\
    &\leq \frac{c_\sigma}{d_h} \frac{8}{\alpha} + \sqrt{\frac{8 \log \frac{6}{\delta_2}}{\alpha d_h}} + 4\sqrt{2} \frac{\log \frac{6}{\delta_2}}{\alpha d_h} \leq 3 \sqrt{\frac{8 \log \frac{6}{\delta_2}}{\alpha d_h}}.
\end{align*}
where we choose $d_h \geq \frac{8}{\alpha} \log \frac{6}{\delta_2}$.
Then take a union bound over $(\mathbf{x,x}), (\mathbf{x,x'}), (\mathbf{x',x'})$. 
Finally, taking $\mathbf{x}^{(1)} = \mathbf{x}^{(2)}$, we have 
\begin{align*}
    & \norm{\sqrt{\frac{c_\sigma}{d_h}} \sum_i \mathbf{b}_i^{(h+1)}(\mathbf{x}^{(1)}) \left( \tilde{\mathbf{w}}_i^{(h+1)} \right)^\top \mathbf{M}_i^{(h+1)} \Pi_{\mathbf{G}_i}^\bot \mathbf{D}}_2 \\
    & \leq \sqrt{\left| {\frac{c_\sigma}{d_h}} \sum_{i,j} \mathbf{b}_i^{(h+1)}(\mathbf{x}^{(1)}) \left( \tilde{\mathbf{w}}_i^{(h+1)} \right)^\top \mathbf{M}_i^{(h+1)} \Pi_{\mathbf{G}_i}^\bot \mathbf{D} \Pi_{\mathbf{G}_j}^\bot \mathbf{M}_j^{(h+1)} \tilde{\mathbf{w}}_j^{(h+1)} \mathbf{b}_j^{(h+1)} (\mathbf{x}^{(1)}) \right|} \\
    &\leq \sqrt{\frac{2}{d_h} \sum_i \mathbf{b}_i^{(h+1)} (\mathbf{x}^{(1)}) \mathbf{b}_i^{(h+1)} (\mathbf{x}^{(2)}) \tr( \mathbf{M}_i^{(h+1)} \mathbf{D} \mathbf{M}_i^{(h+1)}) + 3 \sqrt{\frac{8 \log \frac{6}{\delta_2}}{\alpha d_h}}} \\
    &\leq \sqrt{4 + 3 \sqrt{\frac{8 \log \frac{6}{\delta_2}}{\alpha d_h}}}
    \leq 6.
\end{align*}
\end{proof}

\subsubsection{Proof of \texorpdfstring{\Cref{lemma: event_C}}{Bounding Pseudo Networks' Output}: Bounding Pseudo Networks' Output}\label{sec: proof_event_C}
This is the most involving part of the proof. 
To facilitate the proof, we first introduce a special property of the standard Gaussian vector. 
\begin{proposition}\label{prop: gaussian_indicator}
For any given nonzero vectors $\mathbf{x,y}$, the distribution of $(\mathbf{w}^\top \mathbf{x})^2 \mathbb{I}(\mathbf{w}^\top \mathbf{y} > 0)$ is the same as $(\mathbf{w}^\top \mathbf{x})^2 \mathbb{I}(\mathbf{w}^\top \mathbf{x} > 0)$ where $\mathbf{w} \sim \mathcal{N}(\mathbf{0, I})$.
\end{proposition}
\begin{proof}
Define random variables $z_1 = \left(\mathbf{w}^\top \mathbf{x} \right)^2 \mathbb{I}(\mathbf{w}^\top \mathbf{y}>0)$ and $z_2 = \left(\mathbf{w}^\top \mathbf{x} \right)^2 \mathbb{I}(\mathbf{w}^\top \mathbf{x}>0)$.
Let $F_1, F_2$ be the cumulative distribution function of $z_1, z_2$. 
It is easy to see that both $z_1$ and $z_2$ has probability $1/2$ of being zero and thus we consider the probability that $z_1$ and $z_2$ are not identically zero.
Then for $z > 0$, 
\begin{align*}
    \Pr[0 < z_1 \leq z] &= \int_{\{ \mathbf{w}: \mathbf{w}^\top \mathbf{y}>0, |\mathbf{w^\top x}| \leq \sqrt{z}\}} \frac{1}{(2\pi)^{k/2}} e^{-\frac{1}{2} \norm{\mathbf{w}}_2^2} \ d\mathbf{w} \\
    &= \int_{\{ \mathbf{w}: \mathbf{w}^\top \mathbf{x}>0, \mathbf{w}^\top \mathbf{y}>0, |\mathbf{w^\top x}| \leq \sqrt{z}\} \cup \{ \mathbf{w}: \mathbf{w}^\top \mathbf{x} \leq 0, \mathbf{w}^\top \mathbf{y}>0, |\mathbf{w^\top x}| \leq \sqrt{z}\}} \frac{1}{(2\pi)^{k/2}} e^{-\frac{1}{2} \norm{\mathbf{w}}_2^2} \ d\mathbf{w} \\ 
    &= \int_{\{ \mathbf{w}: \mathbf{w}^\top \mathbf{x}>0, \mathbf{w}^\top \mathbf{y}>0, |\mathbf{w^\top x}| \leq \sqrt{z}\}} \frac{1}{(2\pi)^{k/2}} e^{-\frac{1}{2} \norm{\mathbf{w}}_2^2} \ d\mathbf{w} \\ 
    &\quad + \int_{ \{ \mathbf{w}: \mathbf{w}^\top \mathbf{x} \leq 0, \mathbf{w}^\top \mathbf{y}>0, |\mathbf{w^\top x}| \leq \sqrt{z}\}}  \frac{1}{(2\pi)^{k/2}} e^{-\frac{1}{2} \norm{\mathbf{w}}_2^2} \ d\mathbf{w} \\
    &= \int_{\{ \mathbf{w}: \mathbf{w}^\top \mathbf{x}>0, \mathbf{w}^\top \mathbf{y}>0, |\mathbf{w^\top x}| \leq \sqrt{z}\}} \frac{1}{(2\pi)^{k/2}} e^{-\frac{1}{2} \norm{\mathbf{w}}_2^2} \ d\mathbf{w} \\ 
    &\quad + \int_{ \{ \mathbf{w}: \mathbf{w}^\top \mathbf{x} > 0, \mathbf{w}^\top \mathbf{y} \leq 0, |\mathbf{w^\top x}| \leq \sqrt{z}\}} \frac{1}{(2\pi)^{k/2}} e^{-\frac{1}{2} \norm{\mathbf{w}}_2^2} \ d\mathbf{w} \\
    &= \int_{\{\mathbf{w}: \mathbf{w}^\top \mathbf{x}>0, |\mathbf{w^\top x}| \leq \sqrt{z}\}} \frac{1}{(2\pi)^{k/2}} e^{-\frac{1}{2} \norm{\mathbf{w}}_2^2} \ d\mathbf{w} \\
    &= \Pr[0 < z_2 \leq z],
\end{align*}
where the third last equality is by spherical symmetry of Gaussian and take $\mathbf{w} := -\mathbf{w}$ over the region. 
\end{proof}

\subsubsection*{Mask-Induced Pseudo-Network}
It turns out that the term in \Cref{eq: dependent_two_parts} is closely related to a network structure which we defined as follows. 
\begin{definition}[Pseudo-network induced by mask]
Define the pseudo-network induced by the $h$-th layer $j$-th column of sparse masks $ \mathbf{m}^{(h)}$ denoted by $\mathbf{m}^{(h)}_{\cdot j}$ for all $h \in \{2, \ldots, L\}$, $j \in [d_{h-1}]$ and $h' \in \{h+1, h+2, \ldots, L\}$ to be
\begin{align*}
    \mathbf{g}^{(h,j,h)}(\mathbf{x}) &= \sqrt{\frac{c_\sigma}{d_{h}}} \mathbf{D}^{(h)}(\mathbf{x}) \textnormal{diag}_i \left(\frac{\mathbf{m}_{ij}^{(h)} \sqrt{\alpha} }{\norm{\mathbf{g}^{(h-1)} \odot \mathbf{m}_i^{(h)}}_2^2} \right) \mathbf{f}^{(h)}(\mathbf{x}), \\
    \mathbf{f}^{(h,j,h')}(\mathbf{x}) &= \left(\mathbf{W}^{(h')} \odot \mathbf{m}^{(h')} \right) \mathbf{g}^{(h, j, h'-1)}(\mathbf{x}), \\
    \mathbf{g}^{(h, j, h')}(\mathbf{x}) &= \sqrt{\frac{c_\sigma}{d_{h'}}} \mathbf{D}^{(h')}(\mathbf{x}) \mathbf{f}^{(h,j,h')}(\mathbf{x}).
\end{align*}
where $f^{(h,j,L+1)}(\mathbf{x})$ is the output of the pseudo-network. 
\end{definition}
We would like to bound $|{f}^{(h+1, j, L+1)}(\mathbf{x})|$ for all $h \in \{2, \ldots, L\}$, $j \in [d_{h-1}]$.
Observe that without the diagonal matrix in $\mathbf{g}^{(h,j,h)}(\mathbf{x})$ we have 
\begin{align*}
    \mathbf{g}^{(h+1)}(\mathbf{x}) = \sqrt{\frac{c_\sigma}{d_{h+1}}} \mathbf{D}^{(h+1)}(\mathbf{x})  \mathbf{f}^{(h+1)}(\mathbf{x}).
\end{align*}
Conditioned on $\mathbf{g}^{(h+1, j, L)}(\mathbf{x})$, $f^{(h+1, j, L+1)}(\mathbf{x})$ has distribution $\mathcal{N}({0}, \norm{\mathbf{g}^{(h+1, j, L)}(\mathbf{x}) \odot \mathbf{m}^{(L+1)}}_2^2)$.
Therefore, the magnitude of $|{f}^{(h+1, j, L+1)}(\mathbf{x})|$ would depend on $\norm{\mathbf{g}^{(h+1, j, L)}(\mathbf{x}) \odot \mathbf{m}^{(L+1)}}_2$.

\begin{definition}\label{def: pseudo_network_event}
Define the event 
\begin{align*}
    & \mathcal{C}_1(\eps) = \left\{ \left|\norm{\mathbf{g}^{(h,j,h')}}_2^2 - \E \norm{\mathbf{g}^{(h,j,h')}}_2^2\right| < \eps, \quad \forall h \in \{2, \ldots, L\}, j \in [d_{h-1}], h' \in \{h+1, h+2, \ldots, L\} \right\}, \\
    & \mathcal{C}_2\left(\mathbf{x}, 2 \sqrt{\log \frac{4 \sum_{h'=1}^{L-1} d_{h'}}{\delta}} \right) \\
    & \quad = \left\{ |f^{(h, j, L+1)}(\mathbf{x})| < 2 \sqrt{\log \frac{4 \sum_{h'=1}^{L-1} d_{h'}}{\delta}}, \ \forall h \in \{2, \ldots, L\}, j \in [d_{h-1}], h' \in \{h+1, \ldots, L\} \right\}, \\
    & \overline{\mathcal{C}}\left( 2 \sqrt{\log \frac{4 \sum_{h'=1}^{L-1} d_{h'}}{\delta}} \right) = \mathcal{C}_1(\eps) \cap \mathcal{C}_2\left(\mathbf{x}, 2 \sqrt{\log \frac{4 \sum_{h'=1}^{L-1} d_{h'}}{\delta}} \right) \cap \mathcal{C}_2\left(\mathbf{x}', 2 \sqrt{\log \frac{4 \sum_{h'=1}^{L-1} d_{h'}}{\delta}} \right).
\end{align*}
\end{definition}
We are going to show that the event $\overline{\mathcal{C}}$ holds with probability $1 - \delta$.

First, we show that 
\begin{lemma}
Assume $\overline{\mathcal{A}}(\eps_1)$ holds for $\eps_1 < 1/2$. 
For all $h \in \{2, \ldots, L\}$, $j \in [d_{h-1}]$, it holds that for all $h' \in \{h+1, h+2, h+3, \ldots, L\}$,
\begin{align*}
    \E_{\mathbf{W}^{(h+1)}, \mathbf{m}^{(h+1)} \ldots, \mathbf{W}^{(h')}, \mathbf{m}^{(h')}} \left[ \left. \norm{\mathbf{g}^{(h, j, h')}(\mathbf{x})}_2^2 \right| \mathbf{g}^{(h,j,h)} \right] \leq 2 \E_{\mathbf{W}^{(h+1)}, \mathbf{m}^{(h+1)}, \ldots, \mathbf{W}^{(h')}, \mathbf{m}^{(h')}} \left[ \left. \norm{\mathbf{g}^{(h')}(\mathbf{x})}_2^2 \right| \mathbf{g}^{(h)}(\mathbf{x}) \right].
\end{align*}
\end{lemma}
\begin{proof}
By \Cref{prop: gaussian_indicator}, for two non-zero vectors $\mathbf{x,y}$ we have 
\begin{align*}
    \E_{\mathbf{w} \sim \mathcal{N}(\mathbf{0, I})} \left[\left(\mathbf{w}^\top \mathbf{x} \right)^2 \mathbb{I}(\mathbf{w}^\top \mathbf{y}>0) \right] = \E_{\mathbf{w} \sim \mathcal{N}(\mathbf{0, I})} \left[\left(\mathbf{w}^\top \mathbf{x} \right)^2 \mathbb{I}(\mathbf{w}^\top \mathbf{x}>0) \right].
\end{align*}
This equation tells us that the direction of $\mathbf{y}$ doesn't matter which implies
\begin{align*}
    \E_{\mathbf{w} \sim \mathcal{N}(\mathbf{0, I})} \left[\left(\mathbf{w}^\top \mathbf{x} \right)^2 \frac{c_\sigma}{d_{h+1}} \dot{\sigma}\left(\mathbf{w}^\top \mathbf{y} \right) \right] = \E_{\mathbf{w} \sim \mathcal{N}(\mathbf{0, I})} \left[\left(\mathbf{w}^\top \mathbf{x} \right)^2 \frac{c_\sigma}{d_{h+1}} \dot{\sigma}\left(\mathbf{w}^\top \mathbf{x} \right) \right] = \frac{ \norm{\mathbf{x}}_2^2}{d_{h+1}}.
\end{align*}
Now, this implies that conditioned on $\mathbf{m}$,
\begin{align}\label{eq: simplified_expectation}
     &\E_{\mathbf{w} \sim \mathcal{N}(\mathbf{0, I})} \left[\left( \left(\mathbf{w \odot m} \right)^\top \mathbf{x} \right)^2 \frac{c_\sigma}{d_{h+1}} \dot{\sigma}\left(\left(\mathbf{w \odot m} \right)^\top \mathbf{y} \right) \right] \nonumber\\
     &= \E_{\mathbf{w} \sim \mathcal{N}(\mathbf{0, I})} \left[\left( \left(\mathbf{w \odot m} \right)^\top \mathbf{x} \right)^2 \frac{c_\sigma}{d_{h+1}} \dot{\sigma} \left(\left(\mathbf{w \odot m} \right)^\top \mathbf{x} \right) \right] = \frac{\norm{\mathbf{x \odot m}}_2^2}{ d_{h+1}}.
\end{align}
Now, we fix $h$ and $j$ and prove the inequality holds for all $h'$. 
By \Cref{eq: simplified_expectation},
\begin{align*}
    &\E_{\mathbf{m}^{(h+1)}, \mathbf{W}^{(h+1)}} \left[ \norm{\mathbf{g}^{(h, j, h+1)}}_2^2 \right] \\
    &=\E_{\mathbf{m}^{(h+1)}} \left[ \sum_{i=1}^{d_{h+1}} \E_{\mathbf{w}_i^{(h+1)}} \left[\left. \left( \left(\mathbf{w}_i^{(h+1)} \odot \mathbf{m}_i^{(h+1)} \right)^\top \mathbf{g}^{(h, j, h)} \right)^2 \frac{c_\sigma}{d_{h+1}} \dot{\sigma}\left(\left(\mathbf{w}_i^{(h+1)} \odot \mathbf{m}^{(h+1)}_i \right)^\top \mathbf{g}^{(h)} \right) \right| \mathbf{m}^{(h+1)} \right] \right] \\
    &= {\norm{\mathbf{g}^{(h, j, h)}}_2^2}.
\end{align*}
Hence, by iterated expectation, we have for all $h' \in \{h+1, h+2, h+3, \ldots, L\}$,
\begin{align*}
    \E_{\mathbf{W}^{(h+1)}, \mathbf{m}^{(h+1)}, \ldots, \mathbf{W}^{(h')}, \mathbf{m}^{(h')}} \left[ \left. \norm{\mathbf{g}^{(h, j, h')}(\mathbf{x})}_2^2 \right| \mathbf{g}^{(h,j,h)} \right] &= \norm{\mathbf{g}^{(h,j,h)}}_2^2 ,\\
    \E_{\mathbf{W}^{(h+1)}, \mathbf{m}^{(h+1)}, \ldots, \mathbf{W}^{(h')}, \mathbf{m}^{(h')}} \left[ \left. \norm{\mathbf{g}^{(h')}(\mathbf{x})}_2^2 \right| \mathbf{g}^{(h)}(\mathbf{x}) \right] &= \norm{\mathbf{g}^{(h)}(\mathbf{x})}_2^2.
\end{align*}
By our assumption $\norm{\mathbf{g}^{(h-1)} \odot \mathbf{m}^{(h)}}_2^2 \geq 1 - \eps_1^2 \geq 1/2$, we have $\norm{\mathbf{g}^{(h,j, h)}(\mathbf{x}) }_2^2 \leq 2 \norm{\mathbf{g}^{(h)}(\mathbf{x})}_2^2$.
This proves the lemma. 
\end{proof}
\begin{corollary}
Assume $\overline{\mathcal{A}}(\eps_1)$ holds for $\eps_1 < 1/2$. 
For all $h \in \{2, \ldots, L\}$, $j \in [d_{h-1}]$, $h' \in \{h+1, h+2, h+3, \ldots, L\}$ and $i \in [d_{h'+1}]$,
\begin{align*}
    &\E_{\mathbf{W}^{(h+1)}, \mathbf{m}^{(h+1)} \ldots, \mathbf{W}^{(h')}, \mathbf{m}^{(h')}, \mathbf{m}^{(h'+1)}} \left[ \left. \norm{\mathbf{g}^{(h, j, h')}(\mathbf{x}) \odot \mathbf{m}^{(h'+1)}_i}_2^2 \right| \mathbf{g}^{(h,j,h)} \right] \\
    &\leq 2 \E_{\mathbf{W}^{(h+1)}, \mathbf{m}^{(h+1)}, \ldots, \mathbf{W}^{(h')}, \mathbf{m}^{(h')}, \mathbf{m}^{(h'+1)}} \left[ \left. \norm{\mathbf{g}^{(h')}(\mathbf{x}) \odot \mathbf{m}^{(h'+1)}_i}_2^2 \right| \mathbf{g}^{(h)}(\mathbf{x}) \right].
\end{align*}
\end{corollary}
\begin{proof}
Use the fact that the mask $\mathbf{m}^{(h'+1)}_i$ is independent and preserve the 2-norm in expectation. 
\end{proof}

\begin{lemma}\label{lemma: induced_pseudo_network_output}
Assume $\overline{\mathcal{A}}(\eps_1)$ holds for $\eps_1 < 1/2$. 
Let $\eps \in (0,1)$. 
If for all $h \in {L}$, it satisfies that $d_h \geq \Omega(\frac{1}{\alpha} \frac{L^2}{\eps^2} \log \frac{2L d_{h+1} \sum_{h'=1}^{h-1} d_h' }{\delta}) = \tilde{\Omega}(\frac{1}{\alpha} \frac{L^2}{\eps^2} )$, then with probability at least $1 - \delta$ over the randomness in the initialization of weights and masks, we have for all $h \in \{2, \ldots, L\}$, $j \in [d_{h-1}]$,
\begin{align*}
    |f^{(h, j, L+1)}(\mathbf{x})|,\ |f^{(h, j, L+1)}(\mathbf{x}')| \leq 2 \sqrt{\log \frac{4 \sum_{h'=1}^{L-1} d_{h'}}{\delta}}.
\end{align*}
In other words, if $d_h \geq \Omega(\frac{1}{\alpha} \frac{L^2}{\eps^2} \log \frac{2L d_{h+1} \sum_{h'=1}^{h-1} d_h' }{\delta_3}) = \tilde{\Omega}(\frac{1}{\alpha} \frac{L^2}{\eps^2} )$, then
\begin{align*}
    \Pr\left[ \overline{\mathcal{A}}( \eps_1) \Rightarrow \overline{\mathcal{C}}\left( 2 \sqrt{\log \frac{4 \sum_{h'=1}^{L-1} d_{h'}}{\delta_3}} \right) \right] \geq 1 - \delta_3
\end{align*}
\end{lemma}
\begin{proof}
\Cref{prop: gaussian_indicator} proved that conditioned on $\mathbf{g}, \tilde{\mathbf{g}}, \mathbf{m}$, the random variable $((\mathbf{w} \odot \mathbf{m})^\top \tilde{\mathbf{g}} \sqrt{\frac{c_\sigma}{d_h}})^2 \dot{\sigma}(\mathbf{(w \odot m)^\top g})$ has the same distribution as $((\mathbf{w} \odot \mathbf{m})^\top \tilde{\mathbf{g}} \sqrt{\frac{c_\sigma}{d_h}})^2 \dot{\sigma}(\mathbf{(w \odot m)^\top \tilde{g}})$, which implies their concentration properties are the same. 
At a given layer $h'$, we want this concentration to holds for all $\norm{\mathbf{g}^{(h, j, h')}(\mathbf{x}) \odot \mathbf{m}^{(h'+1)}}_2^2$ where $2 \leq h \leq h'$ and $h \in [d_{h-1}]$. 
Thus there is in total $\sum_{h=1}^{h'-1} d_h$ events. 
Therefore, by \Cref{thm: concentration_g}, if $d_{h'} \geq \Omega( \frac{1}{\alpha} \frac{L^2 }{\eps^2} \log \frac{8  d_{h'+1} L \sum_{h=1}^{h'-1} d_{h}}{\delta})$, with probability $1 - \delta/2$, for all layer $h'$, for all $h \in \{2, \ldots, L\}$, $j \in [d_{h-1}]$ and $h' \in \{h+1, h+2, \ldots, L\}$, and for both $\mathbf{x,x'}$
\begin{align*}
    \norm{\mathbf{g}^{(h, j, h')}(\mathbf{x}) \odot \mathbf{m}^{(h'+1)}}_2^2 \leq 2 \E\left[\norm{\mathbf{g}^{(h')}(\mathbf{x}) \odot \mathbf{m}^{(h'+1)}}_2^2 \right] + \eps \leq 3.
\end{align*}
By Lemma \ref{lemma: gaussian_concentration}, this implies with probability $1 - \delta/2 $, for all $j \in [d_h]$, 
\begin{align*}
    |f^{(h, j, L+1)}(\mathbf{x})|,\ |f^{(h, j, L+1)}(\mathbf{x}')| \leq 2 \sqrt{\log \frac{4\sum_{h'=1}^{L-1} d_h}{\delta}}.
\end{align*}
\end{proof}

\subsubsection{Bounding the Dependent Part}
\begin{proposition}[Formal Version of \Cref{prop: main_text_dependent}]\label{prop: dependent}
If $d_{h'} \geq \Omega(\frac{1}{\alpha} \frac{L^2}{\eps^2} \log \frac{8 d_{h'+1} L \sum_{h \leq h'} d_{h} }{\delta})$, with probability $1 - \delta_3/2$, the event $\overline{\mathcal{C}}(\sqrt{\log \frac{\sum d_h}{\delta_3}})$ (which we define in the proof) holds and at layer $h'$, for all $j \in [d_h]$, 
\begin{align*}
    \norm{\sum_{i} \mathbf{b}^{(h+1)}_i (\mathbf{x}^{(1)}) \left( {\mathbf{w}}^{(h+1)}_i \right)^\top \Pi_{\mathbf{G}_i} \mathbf{M}_i^{(h+1)}}_2 
    &\leq 2 + 2\sqrt{\frac{1}{\alpha} \log \frac{8}{\delta_2}} + \frac{4}{{\alpha}} \sqrt{\log \frac{4\sum d_h}{\delta_3}}.
\end{align*}
\end{proposition}
\begin{proof}
By triangle inequality, combining the result from Lemma \ref{lemma: second_dependent_term} and Lemma \ref{lemma: first_dependent_term}, we have
\begin{align}\label{eq: dependent_two_parts}
    \norm{\sum_{i} \mathbf{b}^{(h+1)}_i (\mathbf{x}^{(1)}) \left( {\mathbf{w}}^{(h+1)}_i \right)^\top \Pi_{\mathbf{G}_i} \mathbf{M}_i^{(h+1)}}_2 &\leq \norm{\sum_{i} \mathbf{b}^{(h+1)}_i (\mathbf{x}^{(1)}) \left( {\mathbf{w}}^{(h+1)}_i \right)^\top \Pi_{(\mathbf{g}^{(h)}\mathbf{(x)} \odot \mathbf{m}_i^{(h+1)})} \mathbf{M}_i^{(h+1)}}_2 \nonumber \\
    &\quad + \norm{\sum_{i} \mathbf{b}^{(h+1)}_i (\mathbf{x}^{(1)}) \left( {\mathbf{w}}^{(h+1)}_i \right)^\top \Pi_{\mathbf{G}_i/(\mathbf{g}^{(h)}\mathbf{(x)} \odot \mathbf{m}_i^{(h+1)})} \mathbf{M}_i^{(h+1)}}_2 \\
    &\leq 2+2\sqrt{\frac{1}{\alpha} \log \frac{8}{\delta_2}} + \frac{4}{{\alpha}} \sqrt{\log \frac{\sum_h d_h}{\delta_3}}. \nonumber
\end{align}
\end{proof}
Thus, we need to upper bound the two terms in \Cref{eq: dependent_two_parts}.
We first bound the second term which is easier. 
\begin{lemma}\label{lemma: second_dependent_term}
With probability $1 - \delta_2$,
\begin{align*}
    \norm{\sum_{i} \mathbf{b}^{(h+1)}_i (\mathbf{x}^{(1)}) \left( {\mathbf{w}}^{(h+1)}_i \right)^\top \Pi_{\mathbf{G}_i/(\mathbf{g}^{(h)}\mathbf{(x)} \odot \mathbf{m}_i^{(h+1)})} \mathbf{M}_i^{(h+1)}}_2 \leq 2 \left( 1 + \sqrt{\frac{1}{\alpha} \log \frac{8}{\delta_2}} \right).
\end{align*}
\end{lemma}
\begin{proof}
We omit the superscript denoting layers in this proof when there is no confusion. 
Notice that $\mathbf{G}_i/(\mathbf{g}^{(h)}\mathbf{(x) \odot m}_i^{(h+1)})$ is spanned by the vector
\begin{align*}
    \mathbf{u}^{(i)} := \mathbf{g}^{(h)}\mathbf{(x') \odot m}_i^{(h+1)} - \inprod{\mathbf{g}^{(h)}\mathbf{(x) \odot m}_i^{(h+1)},  \mathbf{g}^{(h)}\mathbf{(x') \odot m}_i^{(h+1)}} \mathbf{g}^{(h)} \mathbf{(x) \odot m}_i^{(h+1)}.
\end{align*}
Now conditioned on $\mathbf{g}^{(h)}(\mathbf{x}), \mathbf{b}^{(h+1)}(\mathbf{x})$,
observe that $\sum_i \mathbf{b}_i (\tilde{\mathbf{w}}_i^\top  \mathbf{u}^{(i)}) \mathbf{u}^{(i)} = \sum_i \mathbf{b}_i w_i \mathbf{u}^{(i)}$ where $w_i \stackrel{}{\sim} \mathcal{N}(0,1)$ is Gaussian (independent of $\mathbf{g}^{(h)}(\mathbf{x}), \mathbf{b}^{(h+1)}(\mathbf{x})$).
Let $\mathbf{w} := [w_1, w_2, \ldots, w_{d_{h+1}}]$. 
Its covariance matrix is given by
\begin{align*}
    \E_{\tilde{\mathbf{w}}} \left( \sum_i \mathbf{b}_i {{w}}_i \mathbf{u}^{(i)} \right) \left( \sum_j \mathbf{b}_j {{w}}_j \mathbf{u}^{(j)} \right)^\top &= \E_{\mathbf{w}} \sum_{i,j} \mathbf{b}_i \mathbf{b}_j {w}_i w_j \mathbf{u}^{(i)} \left(\mathbf{u}^{(j)} \right)^\top = \sum_i \mathbf{b}_i^2 \mathbf{u}^{(i)} \left( \mathbf{u}^{(i)} \right)^\top.
\end{align*}
Let the eigenvalue decomposition of this matrix be $\mathbf{UDU^\top}$, then the vector $\sum_i \mathbf{b}_i w_i \mathbf{u}^{(i)}$ has the same distribution as $\mathbf{U D}^{1/2} \tilde{\mathbf{w}}$ where $\tilde{\mathbf{w}} \sim \mathcal{N}(\mathbf{0, I})$.
Thus, 
\begin{align*}
    \E_{\mathbf{w}} \left[ \norm{\sum_i \mathbf{b}_i {w}_i \mathbf{u}^{(i)}}_2^2 \right] = \E_{\mathbf{\tilde{w}}} \left[ \tilde{\mathbf{w}}^\top \mathbf{D}^{1/2} \mathbf{U}^\top \mathbf{U} \mathbf{D}^{1/2} \tilde{\mathbf{w}} \right] = \tr(\mathbf{D}).
\end{align*}
Now, we use the fact that the sum of the eigenvalues of a SPD matrix is its trace and we have
\begin{align*}
    \tr(\mathbf{D}) = \tr\left( \sum_i \mathbf{b}_i^2 \mathbf{u}^{(i)} \left(\mathbf{u}^{(i)} \right)^\top \right) = \sum_j \sum_i \mathbf{b}_i^2 \left(\mathbf{u}^{(i)}_{j} \right)^2 = \sum_i \mathbf{b}_i^2 = \norm{\mathbf{b}}_2^2.
\end{align*}
By Jensen's inequality, we have
\begin{align*}
    \E_{\mathbf{w}} \left[ \norm{\sum_i \mathbf{b}_i {w}_i \mathbf{u}^{(i)}}_2 \right] \leq \sqrt{\E_{\mathbf{w}} \left[ \norm{\sum_i \mathbf{b}_i {w}_i \mathbf{u}^{(i)}}_2^2 \right]} = \norm{\mathbf{b}}_2.
\end{align*}
Further, use the definition of two norm we can write
\begin{align*}
    \norm{\sum_i \mathbf{b}_i {w}_i \mathbf{u}^{(i)}}_2 = \sup_{\norm{\mathbf{x}}_2 = 1} \inprod{\mathbf{x}, \sum_i \mathbf{b}_i {w}_i \mathbf{u}^{(i)}} \stackrel{\mathcal{D}}{=} \sup_{\norm{\mathbf{x}}_2=1} \inprod{\mathbf{x}, \mathbf{UD}^{1/2} \tilde{\mathbf{w}}} = \sup_{\norm{\mathbf{x}}_2=1} \inprod{\mathbf{x} \mathbf{D}^{1/2}, \tilde{\mathbf{w}}}.
\end{align*}
The last quantity is in form of a Gaussian complexity and, by \Cref{lemma: gaussian_complexity_concentration}, has sub-Gaussian concentration with variance proxy $\sigma^2 = \max_i \mathbf{D}_{ii} \leq \tr(\mathbf{D}) = \norm{\mathbf{b}}_2^2$.
Thus, with probability $1 - \delta_2/4$,
\begin{align*}
    \norm{\sum_{i} \mathbf{b}^{(h+1)}_i (\mathbf{x}^{(1)}) \left( {\mathbf{w}}^{(h+1)}_i \right)^\top \Pi_{\mathbf{G}_i/(\mathbf{g}^{(h)}\mathbf{(x)} \odot \mathbf{m}_i^{(h+1)})} \mathbf{M}_i^{(h+1)}}_2 \leq \left( 1 + \sqrt{\frac{2}{\alpha}  \log \frac{8}{\delta_2}} \right) \norm{\mathbf{b}^{(h+1)}}_2 \leq 2 \left( 1 + \sqrt{\frac{1}{\alpha} \log \frac{8}{\delta_2}} \right).
\end{align*}
\end{proof}

Now we bound the first term in \Cref{eq: dependent_two_parts}.

\begin{lemma}\label{lemma: first_dependent_term}
With probability $1 - \delta$,
\begin{align*}
    \norm{\sum_{i} \mathbf{b}^{(h+1)}_i (\mathbf{x}^{(1)}) \left( {\mathbf{w}}^{(h+1)}_i \right)^\top  \Pi_{(\mathbf{g}^{(h)} \odot \mathbf{m}_i^{(h+1)})} \mathbf{M}_i^{(h+1)}}_2 \leq \frac{4}{{\alpha}} \sqrt{\log \frac{4\sum_h d_h}{\delta}}.
\end{align*}
\end{lemma}
\begin{proof}
Since $\Pi_{(\mathbf{g}^{(h)} \odot \mathbf{m}_i^{(h+1)})} = \frac{(\mathbf{g}^{(h)} \odot \mathbf{m}_i^{(h+1)}) (\mathbf{g}^{(h)} \odot \mathbf{m}_i^{(h+1)})^\top}{\norm{(\mathbf{g}^{(h)} \odot \mathbf{m}_i^{(h+1)})}^2_2}$ we have
\begin{align*}
    & \norm{\sum_{i} \mathbf{b}^{(h+1)}_i (\mathbf{x}^{(1)}) \left( {\mathbf{w}}^{(h+1)}_i \right)^\top \Pi_{(\mathbf{g}^{(h)} \odot \mathbf{m}_i^{(h+1)})}  \mathbf{M}_i^{(h+1)}}_2 \\
    &= \norm{\sum_{i} \mathbf{b}^{(h+1)}_i (\mathbf{x}^{(1)}) \left( {\mathbf{w}}^{(h+1)}_i \right)^\top \frac{(\mathbf{g}^{(h)} \odot \mathbf{m}_i^{(h+1)}) (\mathbf{g}^{(h)} \odot \mathbf{m}_i^{(h+1)})^\top}{\norm{(\mathbf{g}^{(h)} \odot \mathbf{m}_i^{(h+1)})}^2_2} \mathbf{M}_i^{(h+1)}}_2 \\
    &= \norm{\frac{1}{\sqrt{\alpha}} \sum_{i} \mathbf{b}^{(h+1)}_i (\mathbf{x}^{(1)}) \left( {\mathbf{w}}^{(h+1)}_i \right)^\top \frac{(\mathbf{g}^{(h)} \odot \mathbf{m}_i^{(h+1)}) (\mathbf{g}^{(h)} \odot \mathbf{m}_i^{(h+1)})^\top}{\norm{(\mathbf{g}^{(h)} \odot \mathbf{m}_i^{(h+1)})}^2_2}}_2.
\end{align*}
Now let's look at the $j$-th coordinate of this vector: 
\begin{align*}
    & \left(\frac{1}{\sqrt{\alpha}} \sum_{i} \mathbf{b}^{(h+1)}_i (\mathbf{x}^{(1)}) \left( {\mathbf{w}}^{(h+1)}_i \right)^\top \frac{(\mathbf{g}^{(h)} \odot \mathbf{m}_i^{(h+1)}) (\mathbf{g}^{(h)} \odot \mathbf{m}_i^{(h+1)})^\top}{\norm{(\mathbf{g}^{(h)} \odot \mathbf{m}_i^{(h+1)})}^2_2} \right)_j \\
    &= \frac{1}{\sqrt{\alpha}} \sum_{i} \mathbf{b}^{(h+1)}_i (\mathbf{x}^{(1)}) \left( {\mathbf{w}}^{(h+1)}_i \right)^\top \frac{(\mathbf{g}^{(h)} \odot \mathbf{m}_i^{(h+1)}) \mathbf{m}_{ij}^{(h+1)} \mathbf{g}_j^{(h)} }{\norm{(\mathbf{g}^{(h)} \odot \mathbf{m}_i^{(h+1)})}^2_2} \\
    &= \frac{1}{{\alpha}} \mathbf{g}_j^{(h)} \left( \mathbf{b}^{(h+1)}(\mathbf{x}^{(1)}) \right)^\top \textnormal{diag}_i \left(\frac{ \mathbf{m}_{ij}^{(h+1)} \sqrt{\alpha}}{\norm{\mathbf{g}^{(h)} \odot \mathbf{m}_i^{(h+1)}}_2^2} \right)
    \begin{bmatrix}
    \left( {\mathbf{w}}^{(h+1)}_1 \odot \mathbf{m}^{(h+1)}_1 \right)^\top (\mathbf{g}^{(h)} \odot \mathbf{m}_1^{(h+1)}) \sqrt{\alpha} \\
    \left( {\mathbf{w}}^{(h+1)}_2 \odot \mathbf{m}^{(h+1)}_2 \right)^\top (\mathbf{g}^{(h)} \odot \mathbf{m}_2^{(h+1)}) \sqrt{\alpha} \\
    \vdots \\
    \left( {\mathbf{w}}^{(h+1)}_{d_{h+1}} \odot \mathbf{m}^{(h+1)}_{d_{h+1}} \right)^\top (\mathbf{g}^{(h)} \odot \mathbf{m}_{d_{h+1}}^{(h+1)}) \sqrt{\alpha}\\
    \end{bmatrix} \\
    &= \frac{1}{{\alpha}} \mathbf{g}_j^{(h)} \left( \mathbf{b}^{(h+1)}(\mathbf{x}^{(1)}) \right)^\top \textnormal{diag}_i\left(\frac{ \mathbf{m}_{ij}^{(h+1)} \sqrt{\alpha}}{\norm{\mathbf{g}^{(h)} \odot \mathbf{m}_i^{(h+1)}}_2^2} \right) \mathbf{f}^{(h+1)}(\mathbf{x}^{(1)}) \\
    &= \frac{1}{{\alpha}} \mathbf{g}_j^{(h)} \left(\mathbf{w}^{(L+1)} \odot \mathbf{m}^{(L+1)} \right)^\top \sqrt{\frac{c_\sigma}{d_L}} \mathbf{D}^{(L)}(\mathbf{x}^{(1)}) \left( \mathbf{W}^{(L)} \odot \mathbf{m}^{(L)} \right) \\
    & \quad \ldots \sqrt{\frac{c_\sigma}{d_{h+1}}} \mathbf{D}^{(h+1)}(\mathbf{x}^{(1)}) \textnormal{diag}_i\left(\frac{ \mathbf{m}_{ij}^{(h+1)} \sqrt{\alpha}}{\norm{\mathbf{g}^{(h)} \odot \mathbf{m}_i^{(h+1)}}_2^2} \right) \mathbf{f}^{(h+1)}(\mathbf{x}^{(1)}) \\
    &= \frac{1}{{\alpha}} \mathbf{g}_j^{(h)} {f}^{(h+1, j, L+1)}(\mathbf{x}^{(1)})
\end{align*}
By \Cref{lemma: induced_pseudo_network_output}, we have
\begin{align*}
    |f^{(h, j, L+1)}| \leq 2 \sqrt{\log \frac{4 \sum_{h'=1}^{L-1} d_{h'}}{\delta}}.
\end{align*}
Finally, by \Cref{thm: concentration_g}, we have $\norm{\mathbf{g}^{(h)}}_2 \leq 2$. 
This implies
\begin{align*}
    \norm{\sum_{i} \mathbf{b}^{(h+1)}_i (\mathbf{x}^{(1)}) \left( {\mathbf{w}}^{(h+1)}_i \right)^\top \mathbf{M}_i^{(h+1)} \Pi_{(\mathbf{g}^{(h)} \odot \mathbf{m}_i^{(h+1)})}}_2 \leq \frac{4}{{\alpha}} \sqrt{\log \frac{4 \sum_{h'=1}^{L-1} d_h}{\delta_3}}
\end{align*}
\end{proof}

\begin{proof}[Continuing Proof of \Cref{lemma: event_B}]
Wrapping things up, from \Cref{eq: dependent_independent_decomposition}, by \Cref{prop: independent} and \Cref{prop: dependent},
\begin{align*}
    & \Bigg| \frac{c_\sigma}{d_h} \sum_{i,j} \mathbf{b}_i^{(h+1)} (\mathbf{x}^{(1)}) \mathbf{b}_j^{(h+1)} (\mathbf{x}^{(2)}) \left( {\mathbf{w}}^{(h+1)}_i \right)  \mathbf{M}_i^{(h+1)} \mathbf{D}^{(h)}(\mathbf{x}^{(1)}) \mathbf{D}^{(h)}(\mathbf{x}^{(2)})  \mathbf{M}_j^{(h+1)} {\mathbf{w}}^{(h+1)}_j \\
    &\quad - \frac{c_\sigma}{d_h} \sum_{i,j} \mathbf{b}_i^{(h+1)} (\mathbf{x}^{(1)}) \mathbf{b}_j^{(h+1)} (\mathbf{x}^{(2)}) \left( {\mathbf{w}}^{(h+1)}_i \right)^\top \mathbf{M}_i^{(h+1)} \Pi_{\mathbf{G}_i}^\bot \mathbf{D}^{(h)}(\mathbf{x}^{(1)}) \mathbf{D}^{(h)}(\mathbf{x}^{(2)}) \Pi_{\mathbf{G}_j}^\bot \mathbf{M}_j^{(h+1)} {\mathbf{w}}^{(h+1)}_j \Bigg|\\
    \leq & \norm{\sqrt{\frac{c_\sigma}{d_h}} \sum_{i} \mathbf{b}_i^{(h+1)} (\mathbf{x}^{(1)}) \left( {\mathbf{w}}^{(h+1)}_i \right)^\top \mathbf{M}_i^{(h+1)} \Pi_{\mathbf{G}_i}^\bot \mathbf{D}^{(h)}(\mathbf{x}^{(1)})} \norm{\sqrt{\frac{c_\sigma}{d_h}} \sum_j \mathbf{b}_j^{(h+1)} (\mathbf{x}^{(2)}) \mathbf{D}^{(h)}(\mathbf{x}^{(2)}) \Pi_{\mathbf{G}_j} \mathbf{M}_j^{(h+1)} \mathbf{w}^{(h+1)}_j} \\
    +& \norm{\sqrt{\frac{c_\sigma}{d_h}} \sum_{i} \mathbf{b}_i^{(h+1)} (\mathbf{x}^{(1)}) \left( {\mathbf{w}}^{(h+1)}_i \right)^\top \mathbf{M}_i^{(h+1)} \Pi_{\mathbf{G}_i} \mathbf{D}^{(h)}(\mathbf{x}^{(1)})} \norm{\sqrt{\frac{c_\sigma}{d_h}} \sum_j  \mathbf{b}_j^{(h+1)} (\mathbf{x}^{(2)}) \mathbf{D}^{(h)}(\mathbf{x}^{(2)}) \Pi_{\mathbf{G}_j}^\bot \mathbf{M}_j^{(h+1)} {\mathbf{w}}^{(h+1)}_j} \\
    +& \norm{\sqrt{\frac{c_\sigma}{d_h}} \sum_{i} \mathbf{b}_i^{(h+1)} (\mathbf{x}^{(1)}) \left( {\mathbf{w}}^{(h+1)}_i \right)^\top \mathbf{M}_i^{(h+1)} \Pi_{\mathbf{G}_i} \mathbf{D}^{(h)}(\mathbf{x}^{(1)})} \norm{\sqrt{\frac{c_\sigma}{d_h}} \sum_j \mathbf{b}_j^{(h+1)} (\mathbf{x}^{(2)}) \mathbf{D}^{(h)}(\mathbf{x}^{(2)}) \Pi_{\mathbf{G}_j} \mathbf{M}_j^{(h+1)} {\mathbf{w}}^{(h+1)}_j }\\
    \leq & 2\left(\frac{12 \sqrt{2}}{\sqrt{d_h}}+ 12\sqrt{\frac{2}{\alpha} \frac{\log \frac{8}{\delta_2}}{d_h}} + \frac{24}{{\alpha}} \frac{\sqrt{2\log \frac{4\sum d_h}{\delta_3}}}{\sqrt{d_h}} \right) + \frac{2}{d_h} \left( 2 + 2\sqrt{\frac{1}{\alpha} \log \frac{8}{\delta_2}} + \frac{4}{{\alpha}} \sqrt{\log \frac{4\sum d_h}{\delta_3}} \right)^2 \\
    \leq & \frac{48\sqrt{2}}{\sqrt{d_h}} +  48 \sqrt{\frac{2}{\alpha} \frac{\log \frac{8}{\delta_2}}{d_h}} + \frac{96}{{\alpha}} \frac{\sqrt{2\log \frac{4\sum d_h}{\delta_3}}}{\sqrt{d_h}}
\end{align*}
\end{proof}

\section{ADDITIONAL EXPERIMENT RESULTS}\label{app: exp}
\subsection{Experimental Setup}\label{sec: exp_setup}
All of our models are trained with SGD and the detailed settings are summarized below. 
\begin{table}[th]
\caption{Summary of architectures, dataset and training hyperparameters}
\label{sample-table}
\vskip 0.15in
\begin{center}
\begin{small}
\begin{sc}
\begin{tabular}{lcccccccr}
\toprule
Model & Data & Epoch & Batch Size & LR & Momentum & LR Decay, Epoch & Weight Decay \\
\midrule
LeNet & MNIST & 40 & 128 & 0.1 & 0 & 0 & 0 \\
{VGG} & {CIFAR-10} & 160 & 128 & 0.1 & 0.9 & 0.1 $\times$ [80, 120] & 0.0001 \\
ResNets & CIFAR-10 & 160 & 128 & 0.1 & 0.9 & 0.1 $\times$ [80, 120] & 0.0001 \\
\bottomrule
\end{tabular}
\end{sc}
\end{small}
\end{center}
\vskip -0.1in
\end{table}

\subsection{Further Experiment Results}\label{sec: further_exp}
\subsubsection{MNIST}
For MNIST dataset, we train a fully-connected neural network with 2-hidden layers of width $2048$. 
The performance is shown in Figure \ref{figure: lenet}.

\begin{figure}[ht]
\vskip 0.2in
\begin{center}
\centerline{\includegraphics[width=0.5\columnwidth]{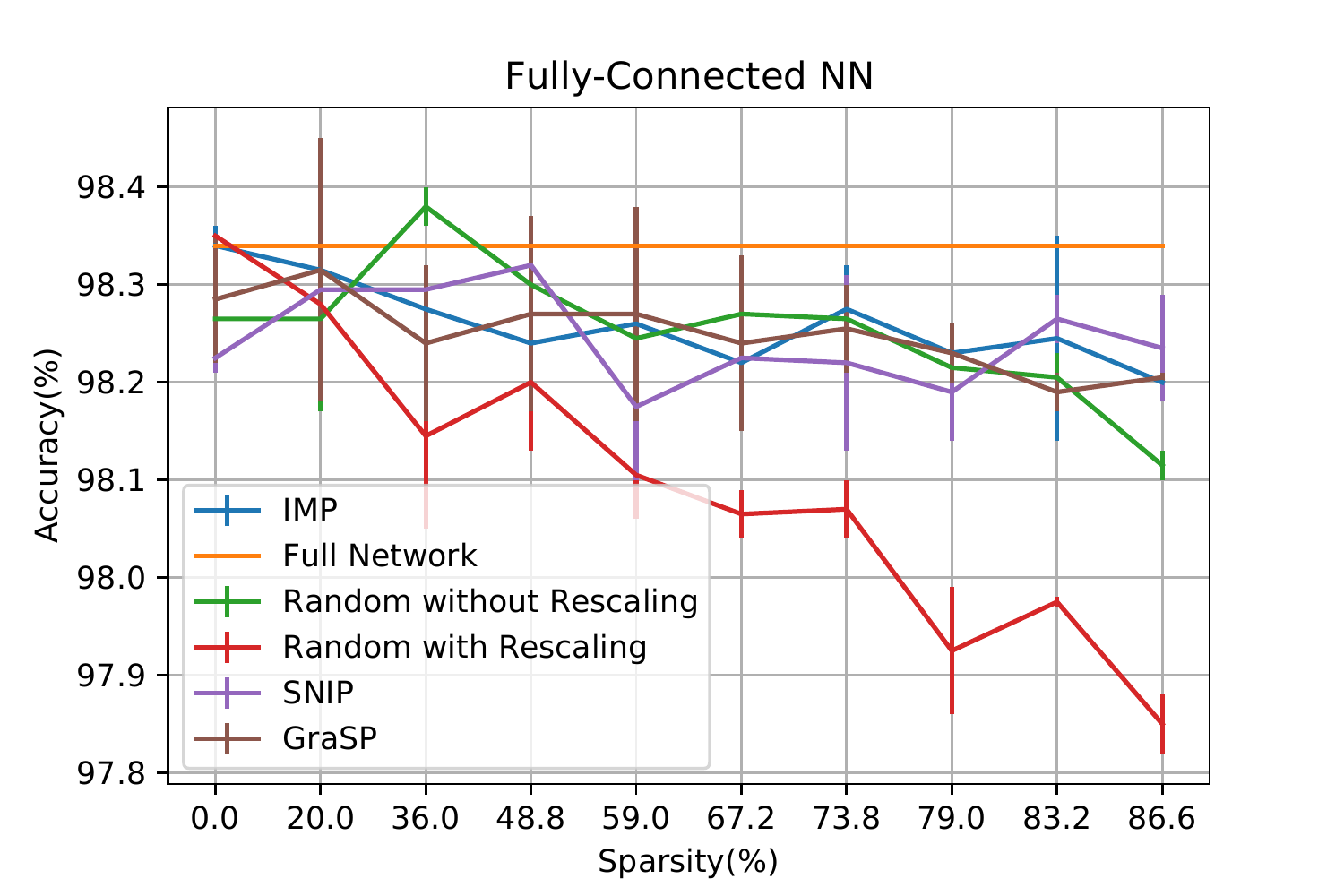}}
\caption{Comparing the performance of random pruning with/without rescaling with IMP, SNIP and GraSP using a fully-connected neural network with 2 hidden layers of width 2048 on MNIST dataset.}
\label{figure: lenet}
\end{center}
\vskip -0.2in
\end{figure}

\subsubsection{CIFAR-10}
\textbf{VGG.}
{We train standard VGG-11 (i.e., VGG-11-64) and VGG-11-128 on CIFAR-10 dataset. 
The results are shown in \Cref{figure: vgg-imp} and \Cref{figure: vgg-random}.}

\begin{figure*}[ht]
  \centering
  \subfloat[]{\includegraphics[width=0.5\columnwidth]{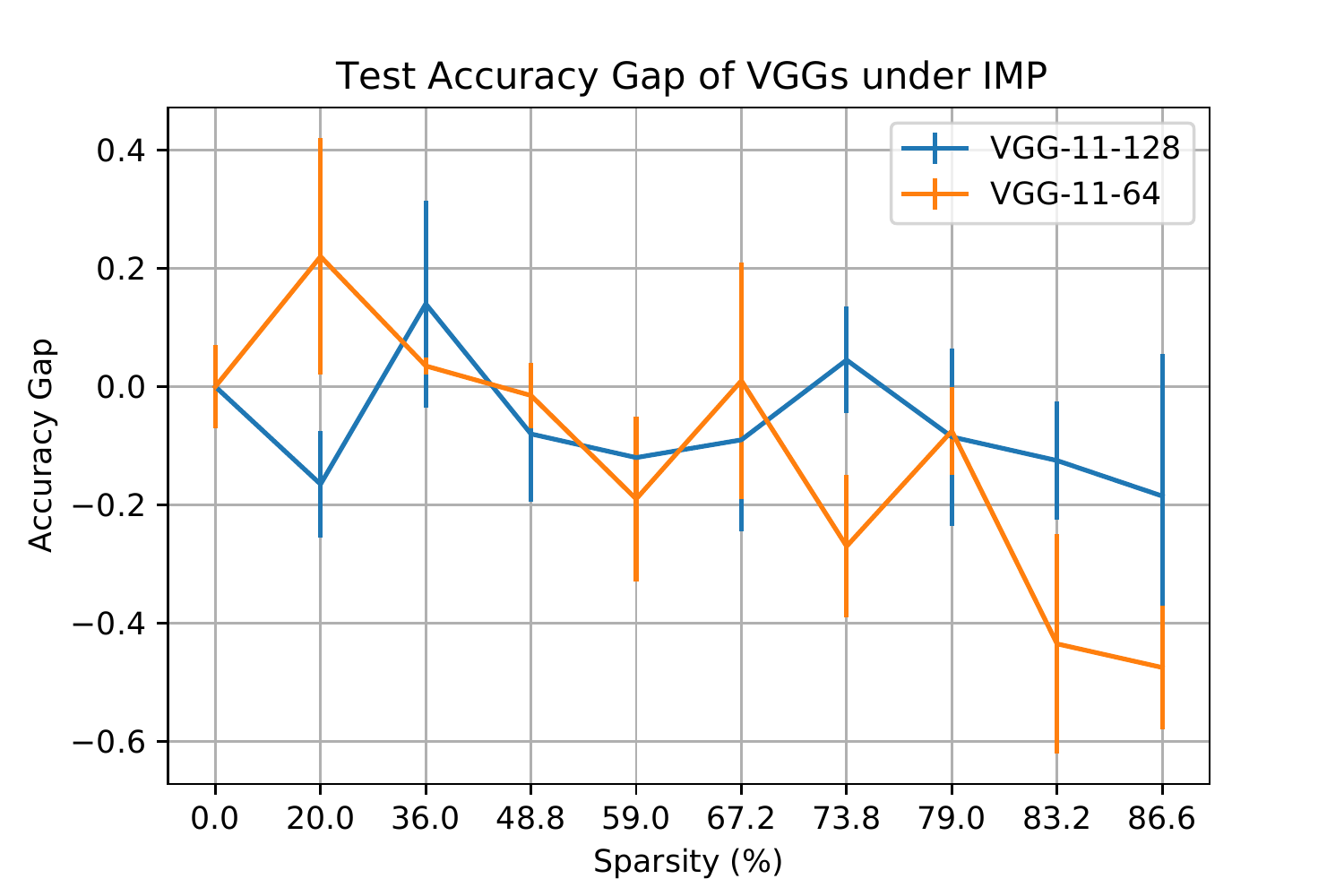} \label{figure: vgg-imp}}
  \subfloat[]{\includegraphics[width=0.5\columnwidth]{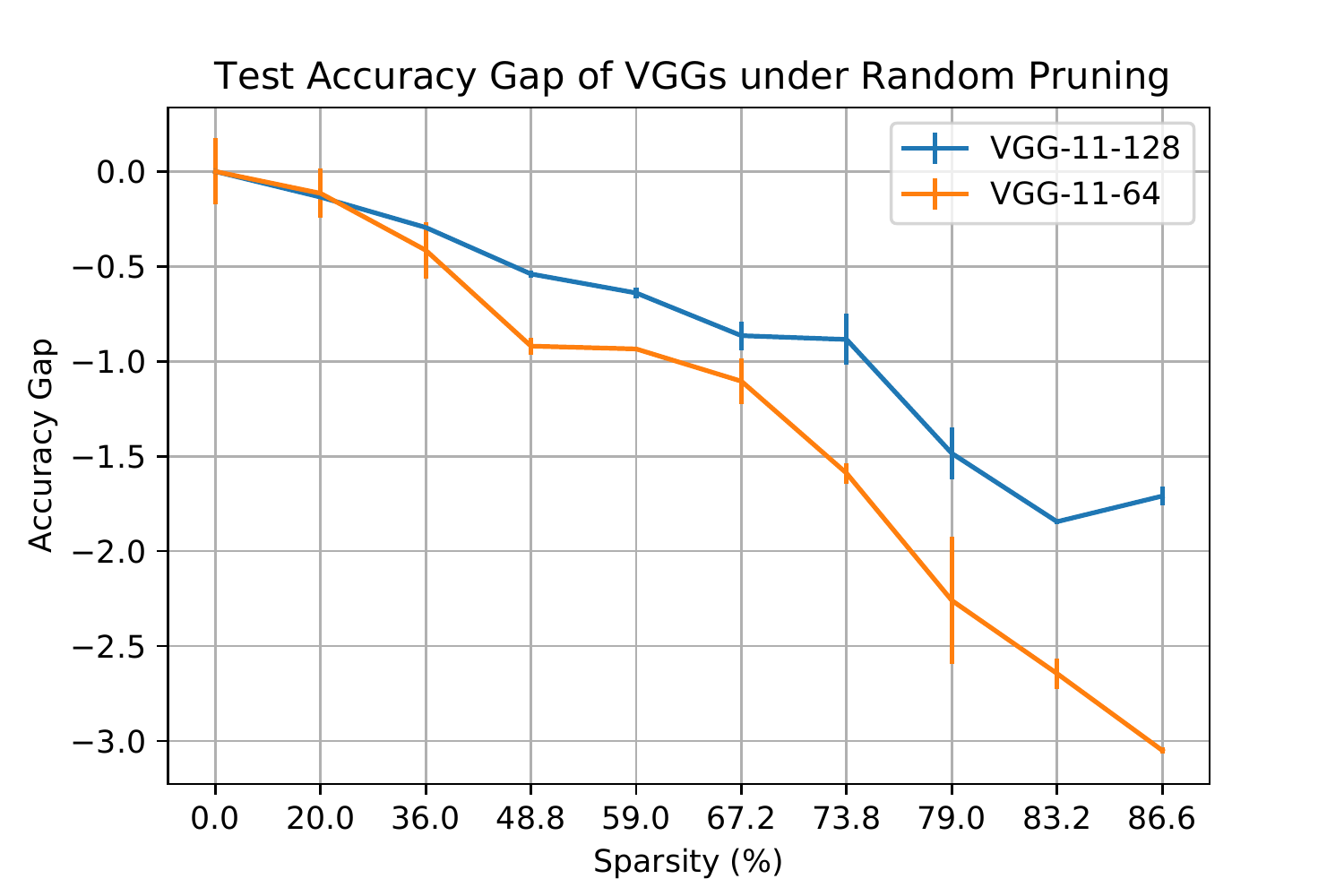} \label{figure: vgg-random}}
  \caption{{The performance of random pruning and IMP using VGG-11 of different width on CIFAR-10 dataset.}}
\end{figure*}

\textbf{ResNet.}
We further train ResNet-20 of width 32, 64 and 128 and compare the performance of random pruning with and without rescaling against IMP. 
The results are shown in Figure \ref{figure: resnet-20-32}, Figure \ref{figure: resnet-20-64} and Figure \ref{figure: resnet-20-128}.

\begin{figure*}[ht]
  \centering
  \subfloat[]{\includegraphics[width=0.5\columnwidth]{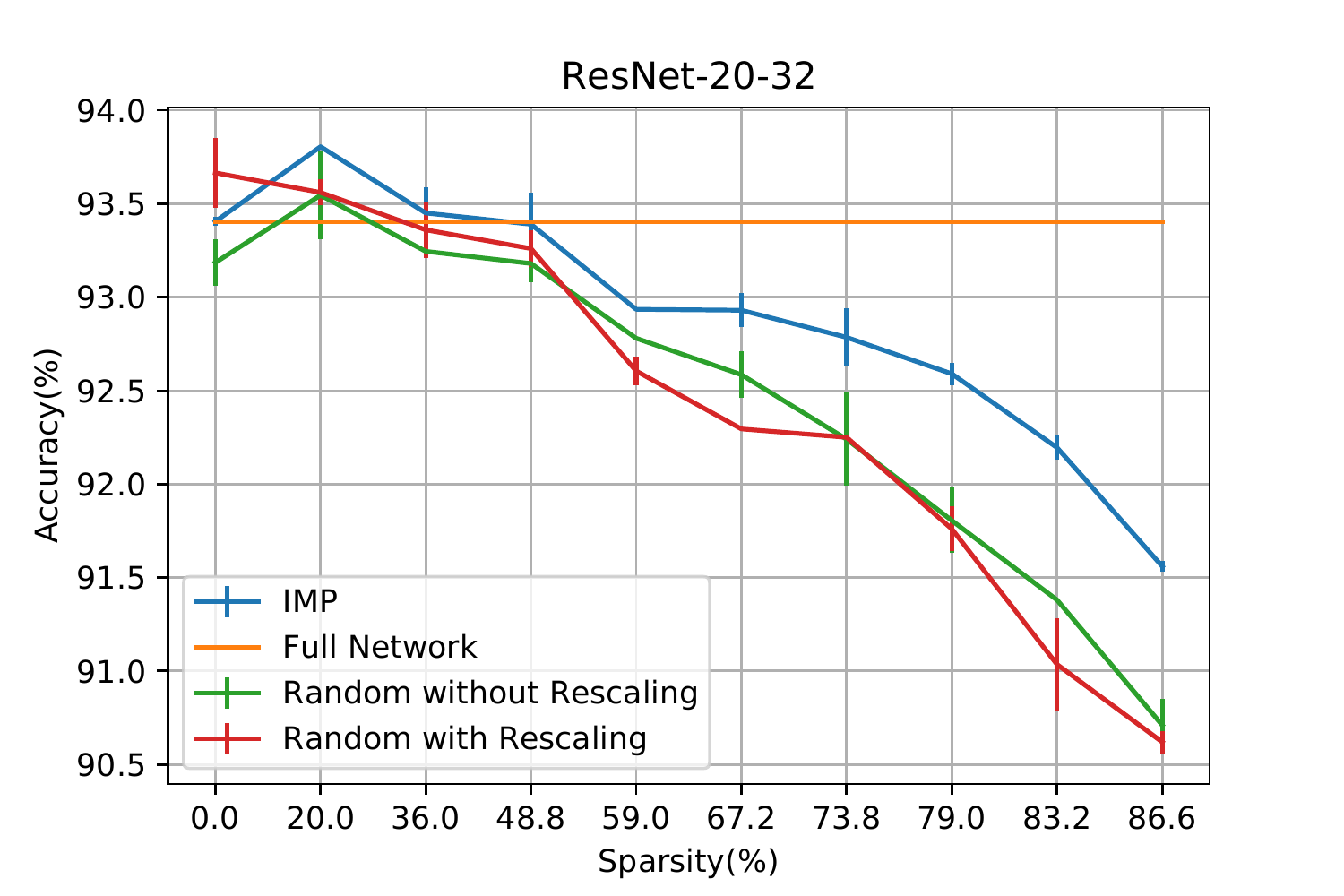} \label{figure: resnet-20-32}}
  \hfill
  \subfloat[]{\includegraphics[width=0.5\columnwidth]{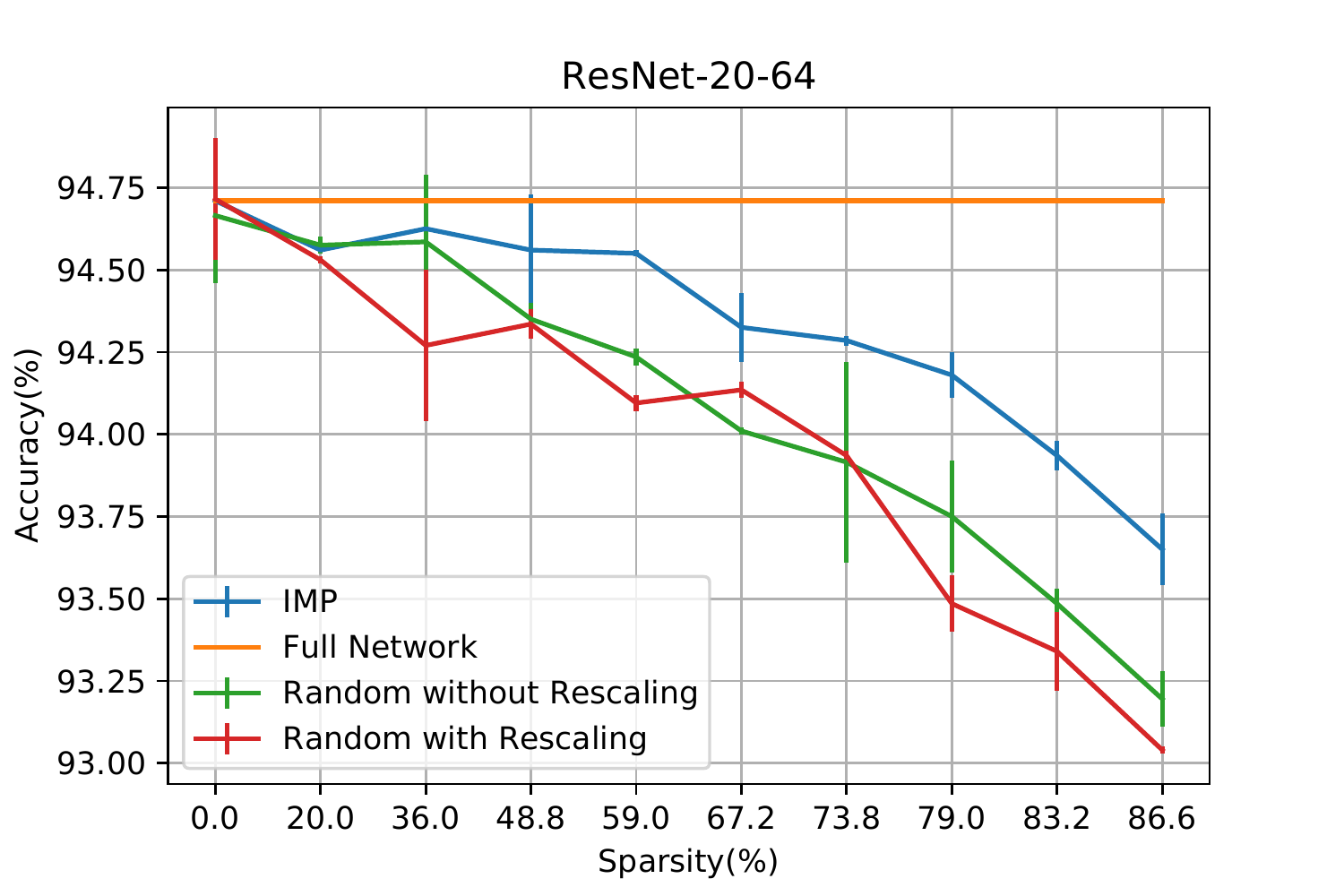} \label{figure: resnet-20-64}}
  \subfloat[]{\includegraphics[width=0.5\columnwidth]{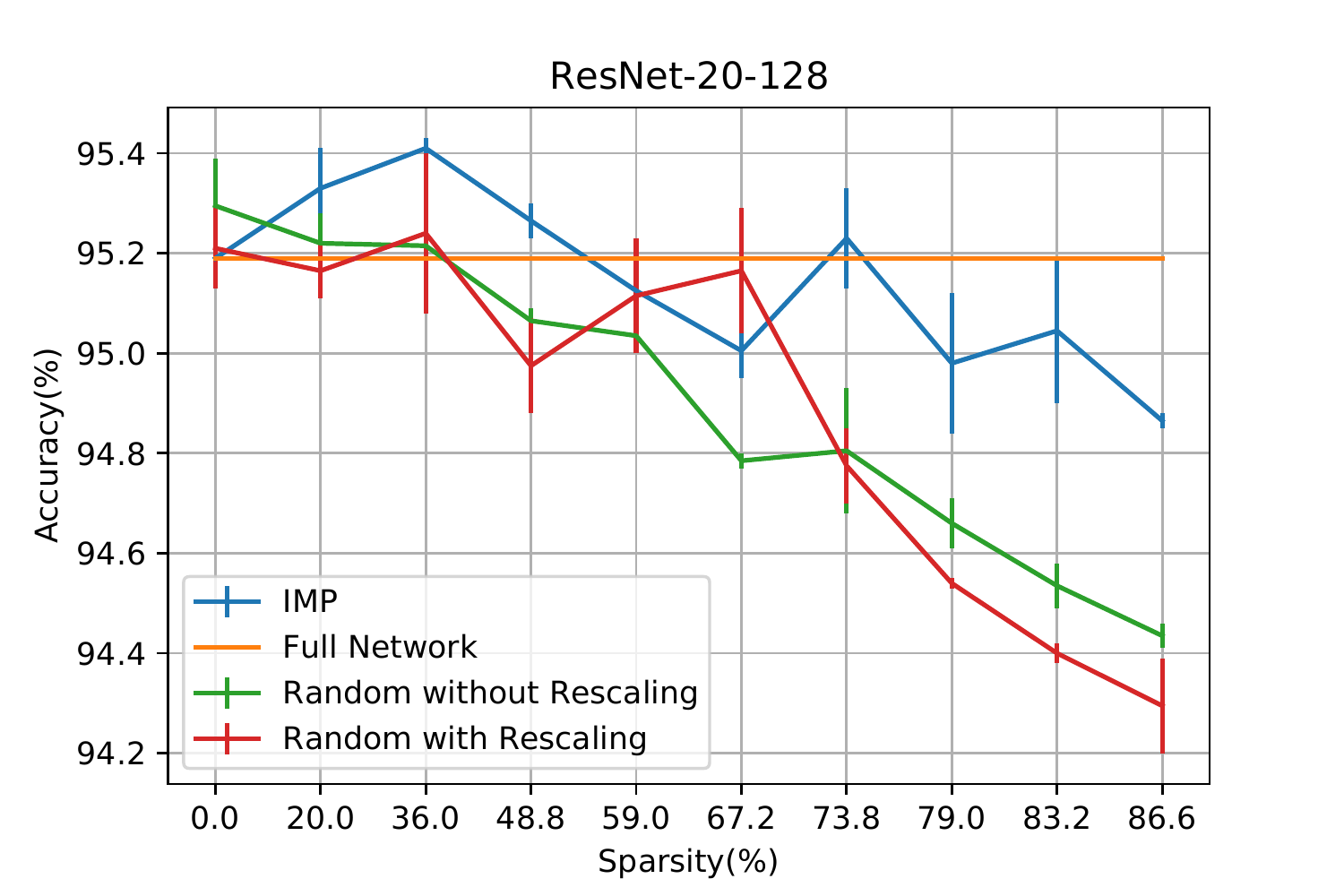} \label{figure: resnet-20-128}}
  \caption{The performance of random pruning with/without rescaling and IMP using ResNet-20 of different width on CIFAR-10 dataset.}
\end{figure*}
We further plot random pruning with rescaling across different width in \Cref{figure: resnet_accuracy_gap_rescale} and pruning by IMP in \Cref{figure: resnet_accuracy_gap_imp}.
The result further shows under the same pruning rate, increasing width can make the pruned model perform on par with the full model. 
\begin{figure*}[ht]
  \centering
  {\includegraphics[width=0.5\columnwidth]{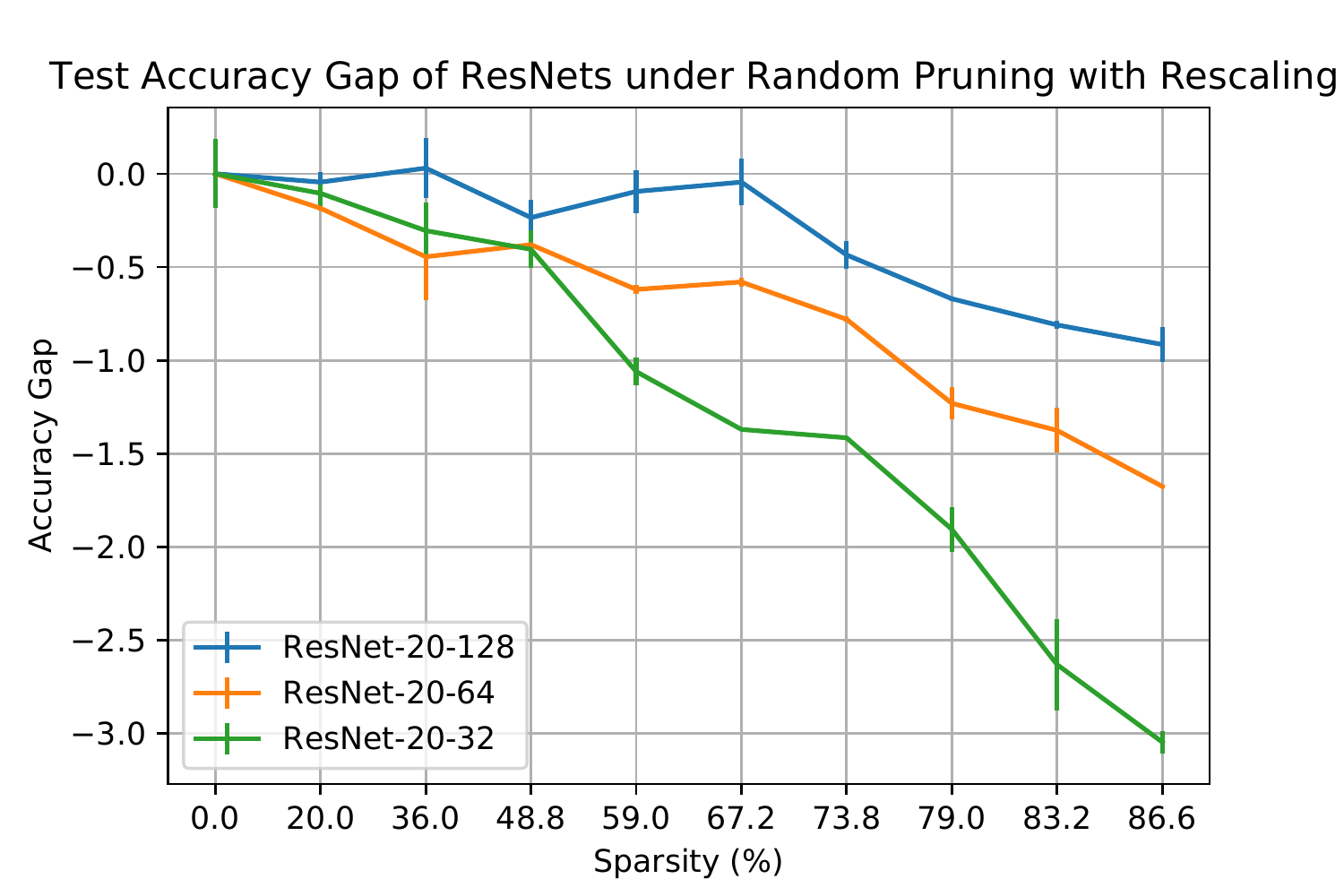} }
  \caption{The test accuracy gap of random pruning with rescaling using ResNet-20 of different width on CIFAR-10 dataset.}
  \label{figure: resnet_accuracy_gap_rescale}
\end{figure*}

\begin{figure*}[ht]
  \centering
  {\includegraphics[width=0.5\columnwidth]{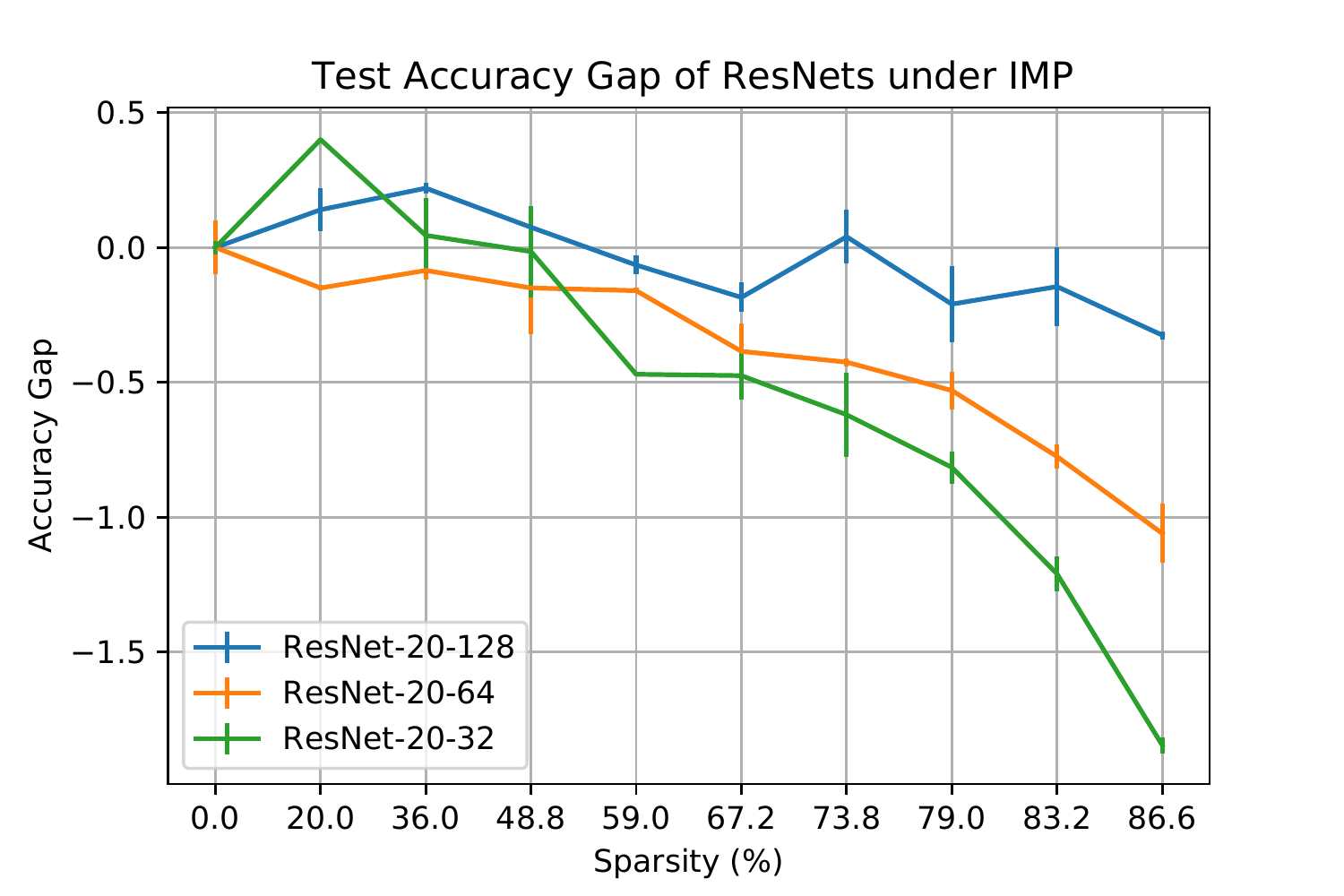} }
  \caption{The test accuracy gap of IMP using ResNet-20 of different width on CIFAR-10 dataset.}
  \label{figure: resnet_accuracy_gap_imp}
\end{figure*}



\end{document}